\documentclass[sn-mathphys-num]{sn-jnl}

\usepackage[utf8]{inputenc}
\usepackage[T1]{fontenc}

\usepackage{graphicx}
\usepackage{amsmath,amssymb,amsfonts}
\usepackage{amsthm}
\usepackage{mathtools}
\usepackage{mathrsfs}

\usepackage{booktabs}
\usepackage{multirow}
\usepackage{float}
\usepackage{stfloats}

\usepackage{subcaption}

\usepackage{algorithm}
\usepackage{algorithmic}

\usepackage{listings}

\usepackage[table]{xcolor}

\usepackage{wrapfig}

\usepackage[capitalize,noabbrev]{cleveref}

\usepackage{geometry}
\usepackage{tabularx}

\usepackage{amsthm}

\newtheorem{lemma}{Lemma}

\theoremstyle{thmstyleone}%
\newtheorem{theorem}{Theorem}%
\theoremstyle{thmstyletwo}%
\theoremstyle{thmstylethree}%
\newtheorem{definition}{Definition}%
\usepackage{titletoc}
\usepackage{makecell}

\definecolor{alggray}{gray}{0.45}
\newcommand{\algc}[1]{\textcolor{alggray}{// #1}} 

\definecolor{bestcol}{RGB}{120,170,245}    
\definecolor{secondcol}{RGB}{170,210,255}  
\definecolor{thirdcol}{RGB}{225,242,255}

\newcommand{\bestcell}[1]{\cellcolor{bestcol}\textbf{#1}}
\newcommand{\second}[1]{\cellcolor{secondcol}#1}
\newcommand{\third}[1]{\cellcolor{thirdcol}#1}
\begin{document}

\title[GESR]{GESR: A Genetic Programming–Based Symbolic Regression Method with Gene Editing}


\author[1,2,3]{\fnm{Yanjie} \sur{Li}}\email{liyanjie@semi.ac.cn}
\author[1]{\fnm{Liping} \sur{Zhang}}\email{lipingzhang@semi.ac.cn}
\author*[1]{\fnm{Min} \sur{Wu}}\email{wumin@semi.ac.cn}
\author*[1,2,3]{\fnm{Weijun} \sur{Li}}\email{wjli@semi.ac.cn}
\author[1]{\fnm{Lina} \sur{Yu}}\email{yulina@semi.ac.cn}
\author[1]{\fnm{Jingyi} \sur{Liu}}\email{liujingyi@semi.ac.cn}
\author[1]{\fnm{Yusong} \sur{Deng}}\email{dys@semi.ac.cn}
\author[1]{\fnm{Mingzhu} \sur{Wan}}\email{wanmingzhu@semi.ac.cn}
\author[1,3]{\fnm{Xin} \sur{Ning}}\email{ningxin@semi.ac.cn}

\affil[1]{\orgdiv{AnnLab}, \orgname{Institute of Semiconductors, Chinese Academy of Sciences}, \orgaddress{\city{Beijing}, \country{China}}}
\affil[2]{\orgdiv{Zhongguancun Academy}, \orgname{Zhongguancun Academy}, \orgaddress{\city{Beijing}, \country{China}}}
\affil[3]{\orgdiv{School of Electronic, Electrical and Communication Engineering}, \orgname{University of Chinese Academy of Sciences}, \orgaddress{\city{Beijing}, \country{China}}}


\abstract{
Mathematical formulas serve as a language through which humans communicate with nature. Discovering mathematical laws from scientific data to describe natural phenomena has been a long-standing pursuit of humanity for centuries. In the field of artificial intelligence, this challenge is known as the symbolic regression problem. Among existing symbolic regression approaches, Genetic Programming (GP) based on evolutionary algorithms remains one of the most classical and widely adopted methods. GP simulates the evolutionary process across generations through genetic mutation and crossover. However, mutations and crossovers in GP are entirely random. While this randomness effectively mimics natural evolution, it inevitably produces both beneficial and detrimental variations. If there existed a metaphorical “God” capable of foreseeing which genetic mutations or crossovers would yield superior outcomes and performing targeted gene editing accordingly, the efficiency of evolution could be substantially improved. Motivated by this idea, we propose in this paper a symbolic regression approach based on gene editing, termed \textit{GESR}. In GESR, we trained two "hands of God" (two BERT models). Among them, the first leverages the BERT's masked language modeling capability to guide the mutation of genes (expression symbols). The other BERT model guides the crossover of individual genes by predicting the crossover point. Experimental results demonstrate that GESR significantly improves computational efficiency compared with traditional GP algorithms and achieves strong overall performance across multiple symbolic regression tasks.
}


\maketitle

\section{Introduction}
\label{sec:introduction}

Mathematical formulas constitute a fundamental language for humans to understand and describe natural laws. Automatically discovering mathematical principles from observational data in order to characterize natural phenomena has long been a central objective in scientific research. In the field of artificial intelligence, this task is known as the \textit{Symbolic Regression} problem, whose core goal is to automatically derive mathematical expressions from data that achieve both high predictive accuracy and strong interpretability.

Among existing symbolic regression approaches, methods based on Genetic Programming (GP) have become one of the most classical and widely adopted paradigms due to their expressive power and flexible search capability. GP simulates biological evolution by iteratively evolving expression structures through genetic operators such as crossover and mutation, gradually approximating the target function. However, mutation operations in traditional GP typically rely on purely random strategies. Although random mutation and crossover helps maintain population diversity and global search capability, its uncontrollability often results in a large number of ineffective or even detrimental variations, thereby reducing search efficiency, slowing convergence, and increasing computational cost. 

Recently, NGGP~\citep{dso} uses LSTM to generate a good initial population for GP, but does not effectively guide the mutation and crossover of GP. PGGP\citep{han2025transformer} calculates the ideal output of a subtree by backward calculation in the evolution process, and then predicts a new subtree structure by NeSymRes, and replaces the original subtree. Then the crossover operation is performed. Although this method introduces some guidance, it is very time-consuming and easy to produce very complex expressions.

From the perspective of biological evolution, if there existed a metaphorical “God” capable of predicting which genetic mutations are more likely to improve fitness and performing targeted gene editing accordingly, the efficiency of evolution could be substantially enhanced. Although gene editing may raise ethical concerns in real-world biological contexts, it can be safely leveraged as a conceptual tool to improve search efficiency in symbolic regression. Inspired by this idea, we propose a novel symbolic regression approach termed Gene Editing Symbolic Regression (GESR).

In GESR, we train a pair of “Hands of God” — two multimodal BERT models — to perform gene editing. First, the first BERT model is used to guide mutation. Specifically, an expression binary tree is linearized into a symbolic sequence via preorder traversal. We then exploit BERT’s masked language modeling capability to predict the appropriate replacement for a masked symbol, thereby enabling guided mutation. Second, we employ a second BERT model to guide crossover. Concretely, two expression sequences are concatenated using a special [SEQ] token to form a single symbolic sequence, which is then fed into BERT together with the data [X,y]. The model predicts the root node of the subtree to be crossed over in each expression, thus enabling guided crossover. This mechanism preserves the evolutionary search advantages of GP while significantly reducing ineffective mutations  and crossovers, thereby improving search stability and efficiency.

Extensive experimental results demonstrate that, across multiple symbolic regression benchmark tasks, GESR significantly improves computational efficiency and convergence speed compared with traditional GP methods, while achieving competitive or superior performance in terms of predictive accuracy and expression interpretability. These findings indicate that integrating deep language models for intelligent gene editing provides a novel and efficient research direction for symbolic regression.

Our main contributions are summarized as follows:

\begin{itemize}
\item 1, We propose a gene editing-based symbolic regression framework, which enables targeted guidance of genetic mutations and crossovers during evolution, and the search efficiency is improved by 33\%.
\item 2, We trained a pair of “Hands of God” — two multimodal BERT models — to perform gene editing and accelerate the GP evolutionary process. They can also serve as a plug-and-play module for GP-based symbolic regression algorithms.
\item 3, We designed the mutation function and based on it and GP designed a method for constructing high-quality gene editing training data.
\item 4, We introduce the concept of “gene editing” into the field of evolutionary computation.

\end{itemize}

\section{Related Work}

\textbf{Genetic Algorithm-Based Methods.}
Genetic Algorithms (GA)~\citep{ga,ga2} are classical optimization techniques inspired by biological evolution. Genetic Programming (GP)~\citep{gp1,gp2,gp3} extends GA to symbolic regression by representing mathematical expressions as tree structures and evolving them through crossover and mutation. Due to its flexible representation, GP remains one of the most widely used paradigms in symbolic regression. However, standard GP often suffers from inefficient random mutation, slow convergence, and expression bloat. PySR~\citep{cranmer2023interpretable} improves GP by integrating constant optimization, evolutionary search, and Pareto-based complexity control, enabling the discovery of compact and interpretable symbolic expressions.

\textbf{Reinforcement Learning-Based Methods.}
Deep Symbolic Regression (DSR)~\citep{dsr} formulates symbolic regression as a reinforcement learning problem, where a recurrent neural network (RNN) acts as a policy to sequentially generate expression tokens, optimized via policy gradient methods. NGGP~\citep{dso} further integrates GP with DSR by using DSR-generated expressions to initialize GP populations and feeding evolved expressions back to refine the policy, forming a closed-loop learning process.  uDSR\cite{landajuela2022unified} is a deep symbolic regression framework that unifies neural-guided search, reinforcement learning optimization, and symbolic expression generation to automatically discover high-precision and interpretable mathematical formulas from data. Other approaches incorporate Monte Carlo Tree Search (MCTS) and reinforcement learning to improve symbolic search efficiency, including SPL~\citep{spl}, SR-GPT~\citep{li2025discovering}, DySymNet~\cite{li2023neural}, and RSRM~\cite{xu2023rsrm}, which combine neural guidance with structured exploration strategies. RSRM~\cite{ruan2025discovering} Combining GP and MCTS, this paper proposes a symbolic regression method for parallel inference of genes on GPU.

\textbf{Neural Network-Based Methods.}
The AI Feynman family~\citep{aif1,aif2} follows a decomposition-based strategy: it first fits data using neural networks, then exploits structural priors such as symmetry, separability, and invariance to decompose complex formulas into simpler symbolic submodules. The EQL framework~\citep{eql1,eql2} replaces standard neural activations with symbolic operators (e.g., $+, -, \sin$), prunes redundant connections, and extracts explicit analytic expressions. DGP~\cite{dgp} explores inferring symbolic operators through normalized neural weights, though it struggles with variable selection in high-dimensional settings. MetaSymNet~\cite{li2025metasymnet} proposes a signed neural architecture with adaptive structure and activation functions for symbolic discovery.

\textbf{Transformer-Based Methods.}
NeSymReS~\citep{nesy} casts symbolic regression as a sequence-to-sequence translation problem using Transformers pre-trained on synthetic data. Subsequent works integrate Transformers with MCTS or policy learning to guide symbolic search, including TPSR~\citep{TPSR}, DGSR-MCTS~\citep{dgsr_mcts}, and FormuleGPT~\citep{li2024generative}. SNIP~\citep{meidani2023snip} and MMSR~\citep{li2025mmsr} further explore multimodal learning for symbolic expression generation. ChatSR~\citep{li2024mllm} introduces large multimodal language models into symbolic regression. PGGP~\citep{han2025transformer} uses a pretrained Transformer to provide high-quality initialization and semantically guided mutation for genetic programming in symbolic regression, thereby combining the expression prior learned by the neural model with GP’s task-specific search.

\begin{figure*}[t]
\centering
\setlength{\belowcaptionskip}{-0.0cm} 
\includegraphics[width=130mm]{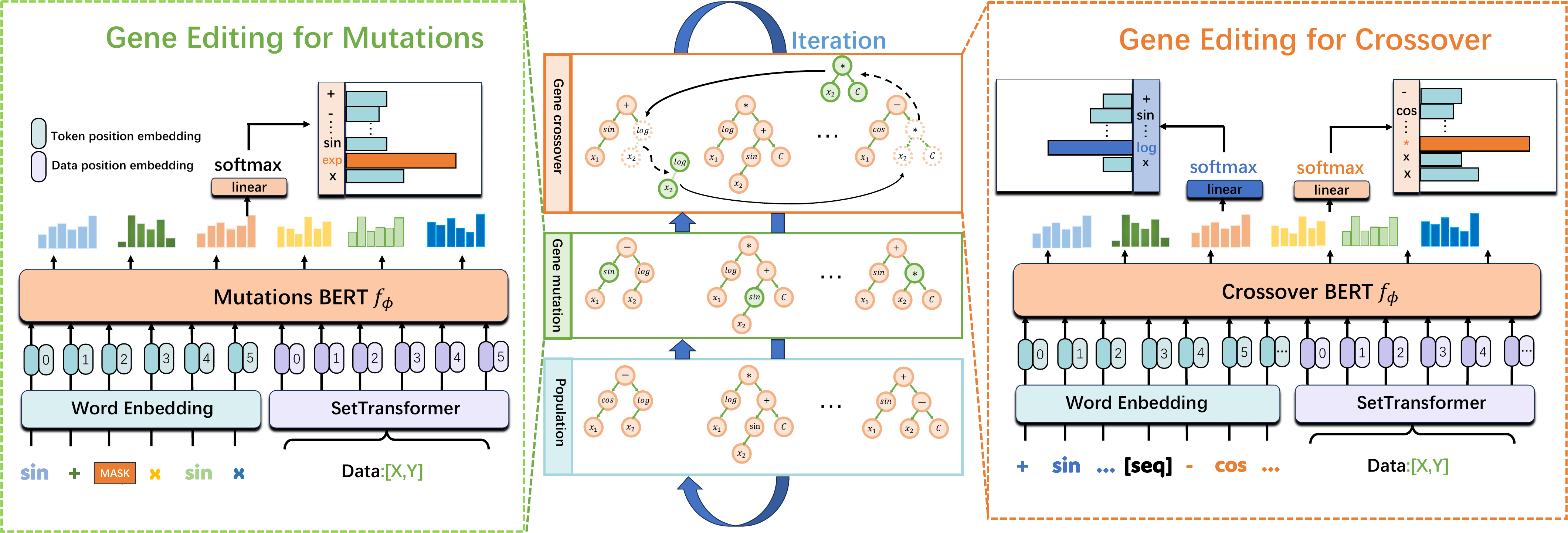}
\caption{Flowchart of the GESR. The left plot is the process of BERT guiding the mutation; The diagram of the evolution process is drawn on the right.
}
\label{fig1}
\end{figure*}
\section{Method}
\subsection{Overall Framework}

We propose a GESR symbolic regression framework that introduces multimodal BERT models to guide the random mutation and crossover process in traditional GP-based symbolic regression. By enabling controllable guided gene editing, the proposed method significantly improves search efficiency and regression performance. The overall framework is illustrated in Fig.~\ref{fig1}.

Within this framework, each symbolic expression is first represented as a binary expression tree and then linearized into a preorder traversal symbolic sequence for sequence-based modeling. First, to simulate mutation, a subset of operator or operand tokens in the sequence is randomly masked and jointly fed, together with the regression dataset $(X, Y)$, into a multimodal BERT. The data modality is encoded into permutation-invariant embeddings using a Set Transformer and is fused with symbolic token embeddings within the Transformer layers, enabling the model to perform data-conditioned gene editing based on the underlying data distribution.

The BERT model outputs a conditional probability distribution over the masked position and selects the symbol with the highest probability as the mutated gene. Subsequently, multiple crossover operations are performed among individuals in the mutated population. For each crossover, a candidate pair of individuals together with the data $[X, y]$ is fed into the BERT model, which probabilistically predicts the root nodes of the subtrees to be exchanged between the two individuals, after which the crossover operation is executed. Next, each expression in the population undergoes constant optimization via BFGS and is evaluated using regression goodness-of-fit metrics, such as $R^2$. Finally, the top $N$ individuals with the highest fitness scores are retained to form the next generation. This evolutionary process is iterated until convergence.

By replacing conventional random mutation and crossover with a data-driven strategy, the proposed framework preserves genetic diversity while steering the search process toward regions of higher fitness in the expression space, thereby achieving more efficient symbolic regression.

\subsection{Symbolic Editor: Multimodal BERT for Mutation}
To enable data-conditioned gene editing of symbolic expressions, we adopt a multimodal BERT architecture as the symbolic sequence decoding and editing module. This module takes as joint input a masked symbolic expression sequence and the embedding vectors produced by the data encoder, and learns to predict optimal symbol replacements conditioned on the regression dataset $(X, Y)$, thereby guiding expression evolution toward structures with improved fitting accuracy.

Specifically, each symbolic expression is first linearized from a binary expression tree into a preorder traversal sequence:
\[
s = (t_1, t_2, \dots, t_L),
\]
where $t_i$ denotes an operator, variable, or constant token. To simulate mutation in genetic programming, one or more tokens in the sequence are randomly selected and replaced with a special marker \texttt{[MASK]}, producing a masked sequence $\tilde{s}$.
Tokens in the masked sequence are mapped into a continuous representation space through a trainable embedding layer, with positional embeddings added to preserve sequential order. Meanwhile, data embeddings generated by the Set Transformer encoder are projected to the same dimensionality as symbolic embeddings and concatenated with token embeddings before being fed into a multi-layer Transformer encoder. This design enables the model to establish cross-modal attention links between symbolic structure and data distribution.
The BERT encoder consists of $L$ stacked self-attention layers, each containing Multi-Head Self-Attention and a Feedforward Network to capture global dependencies within the symbolic sequence. For each masked position, the model outputs a conditional probability distribution over the symbol vocabulary $\mathcal{V}$:

\begin{equation}
p(t_i \mid \tilde{s}, X, Y) = \mathrm{Softmax}(W h_i),
\end{equation}

where $h_i$ denotes the hidden representation at position $i$, and $W$ is a trainable linear projection matrix.

By maximizing the log-likelihood of the ground-truth token under this conditional 
distribution, BERT learns to recover plausible and high-fitness symbol replacements under data constraints, thereby enabling targeted gene editing. This mechanism allows the model to preferentially generate expression mutation with greater potential regression performance (e.g., higher $R^2$), rather than relying on unbiased random perturbations.

Key architectural configurations and training hyperparameters of the BERT model are summarized in Table~\ref{bert-tab}.

\subsection{Gene Editor: Multimodal BERT for Crossover}

Similar to the BERT model for mutation guidance, we further design a multimodal BERT for \textbf{crossover operations} to direct crossover. Unlike the setting in the previous subsection, where the model predicts the symbol at a single masked position, this model takes as input \textbf{a pair of candidate expressions} together with the corresponding dataset $(X,Y)$. As illustrated in Fig.~\ref{fig:crosser}, its objective is to predict the root nodes of the subtrees that are most suitable for crossover in the two parent expressions, thereby guiding the crossover process in genetic programming.

Specifically, let the two parent expressions be linearized into preorder traversal sequences
\[
s^{(a)}=(t^{(a)}_1,t^{(a)}_2,\dots,t^{(a)}_{L_a}),\qquad
s^{(b)}=(t^{(b)}_1,t^{(b)}_2,\dots,t^{(b)}_{L_b}).
\]
The two sequences are then concatenated into a joint sequence using a special separator token:
\[
s^{(ab)}=\bigl(s^{(a)},\texttt{[SEP]},s^{(b)}\bigr),
\]
which is fed into the multimodal BERT together with the sample representation produced by the data encoder. The remaining embedding construction and multimodal fusion mechanism are the same as those described in the previous subsection and are omitted here for brevity.

After Bert encoding, the model outputs, for each position in the two expressions, a conditional probability distribution indicating its suitability as a crossover point. Let positions $i$ and $j$ correspond to the root nodes of candidate subtrees in expressions $s^{(a)}$ and $s^{(b)}$, respectively. The corresponding distributions are defined as
\[
p(i\mid s^{(ab)},X,Y)=\mathrm{Softmax}(W_a h),
\]
\[
p(j\mid s^{(ab)},X,Y)=\mathrm{Softmax}(W_b h),
\]
where $h$ denote the hidden representation for input, and $W_a$ and $W_b$ are trainable linear projection matrices. Based on these distributions, the model selects the subtree root nodes with the highest probabilities in the two parent expressions and exchanges the corresponding subtrees to generate new offspring.

This mechanism replaces the random selection of crossover points in conventional genetic programming with a \textbf{data-driven crossover strategy}. In other words, the model learns not only whether two expressions can be crossed, but more importantly which structural fragments are worth exchanging under the current regression task. As a result, crossover is no longer an unbiased random perturbation, but is explicitly constrained by both the regression data distribution and the semantic structure of the expressions, making it more likely to produce offspring with better fitting performance.

In the actual evolutionary process, the mutated population undergoes multiple crossover operations. For each candidate pair, the crossover-guiding BERT predicts the root nodes of the subtrees to be exchanged based on the joint symbolic sequence and the data $(X,Y)$, and then performs the directed crossover operation. Subsequently, each newly generated expression is further optimized with BFGS for constant refinement and evaluated according to regression goodness-of-fit metrics such as $R^2$. Finally, the top $N$ individuals with the highest fitness scores are retained to form the next generation, and the evolutionary process is repeated until convergence.

\subsection{Data Encoder: SetTransformer}
Data information plays a crucial role in guiding the decoder. To respect the permutation invariance of the data---that is, the fact that the representation should not depend on the ordering of input samples---we adopt the SetTransformer as our data encoder, following~\cite{pmlr-v97-lee19d}. Our encoder takes as input a set of data points $\mathcal{D} = \{X, y\}$, where $X \in \mathbb{R}^{n \times d}$ denotes the input features and $y \in \mathbb{R}^n$ denotes the corresponding targets.

The data points are first passed through a trainable affine layer, which projects them into a latent space $h_n \in \mathbb{R}^{d_h}$. The resulting representations are then processed by a stack of Induced Set Attention Blocks (ISABs)~\cite{pmlr-v97-lee19d}, each of which employs cross-attention. In the first cross-attention layer, a set of learnable inducing points is used as queries, while the input data serves as keys and values. The outputs of this layer are then used as keys and values in a subsequent cross-attention layer, with the original data representations acting as queries.

After these attention layers, we apply dropout to mitigate overfitting. Finally, we standardize the output size via a last cross-attention layer, which uses another set of learnable vectors as queries, ensuring that the encoder output has a fixed dimensionality independent of the number of input samples. The hyperparameter settings of the SetTransformer are summarized in Table~\ref{set-tab}.

\subsection{Training data pair collection for BERT}

\subsubsection{Training data pair collection for Mutation}

To train BERT to learn symbol mutation strategies that improve regression fitting performance, we design a fitness-improvement-based data generation pipeline to construct high-quality editing samples. We sampled the binary operators have [+, -, *, /], unary operators have [sin, cos, exp, log, sqrt], variable operators have $[x_1,..., x_{10}]$, and there are constant placeholders [C]. We sample expressions with a maximum length limit of 60. The overall procedure is summarized as follows and fig.\ref{fig2b}:
\begin{wrapfigure}{r}{0.49\textwidth}
  \centering
  \vspace{-18pt}

  \subfloat[]{
    \includegraphics[width=0.95\linewidth]{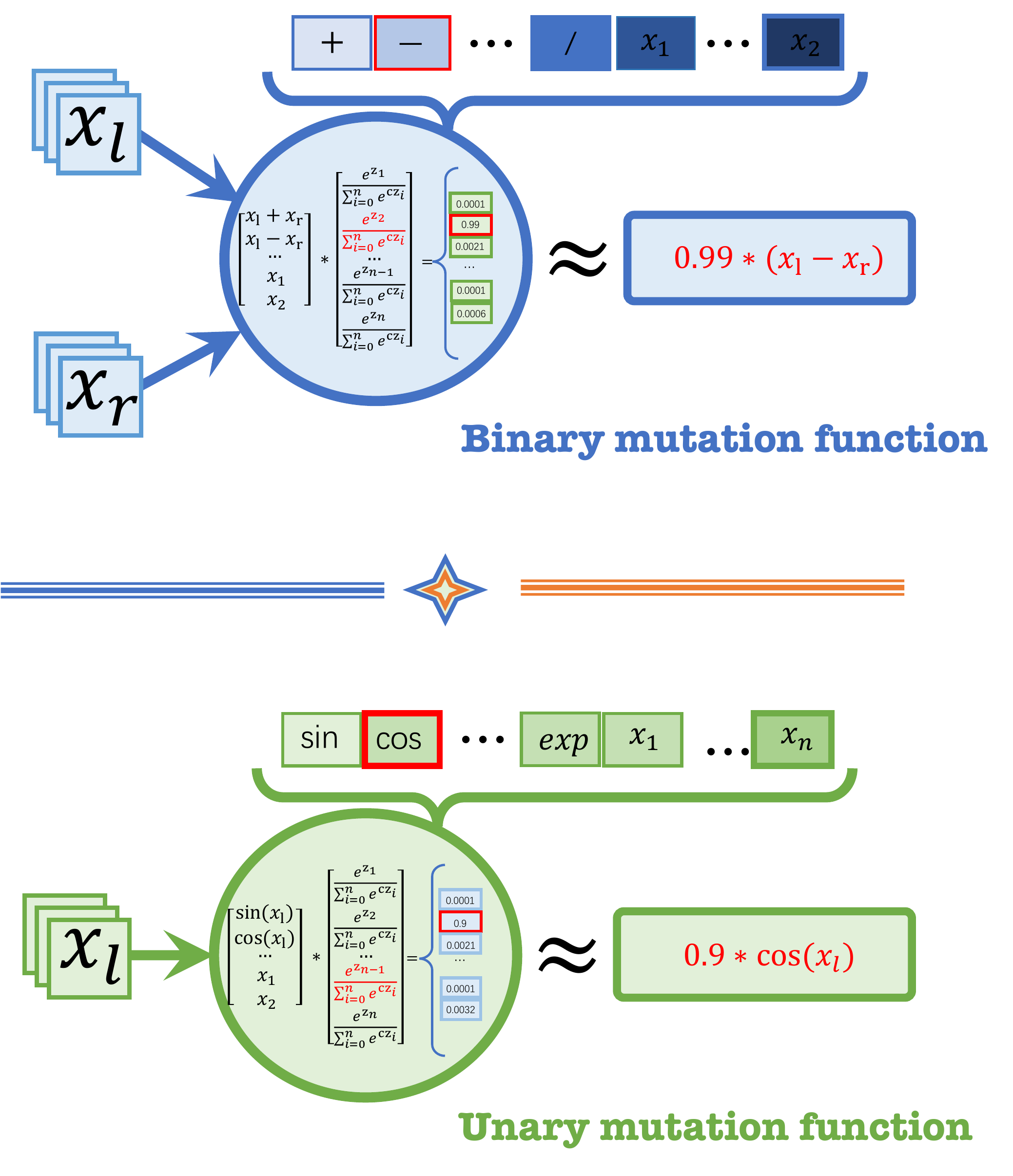}
    \label{fig2a}
  }

  \vspace{-4pt}

  \subfloat[]{
    \includegraphics[width=0.95\linewidth]{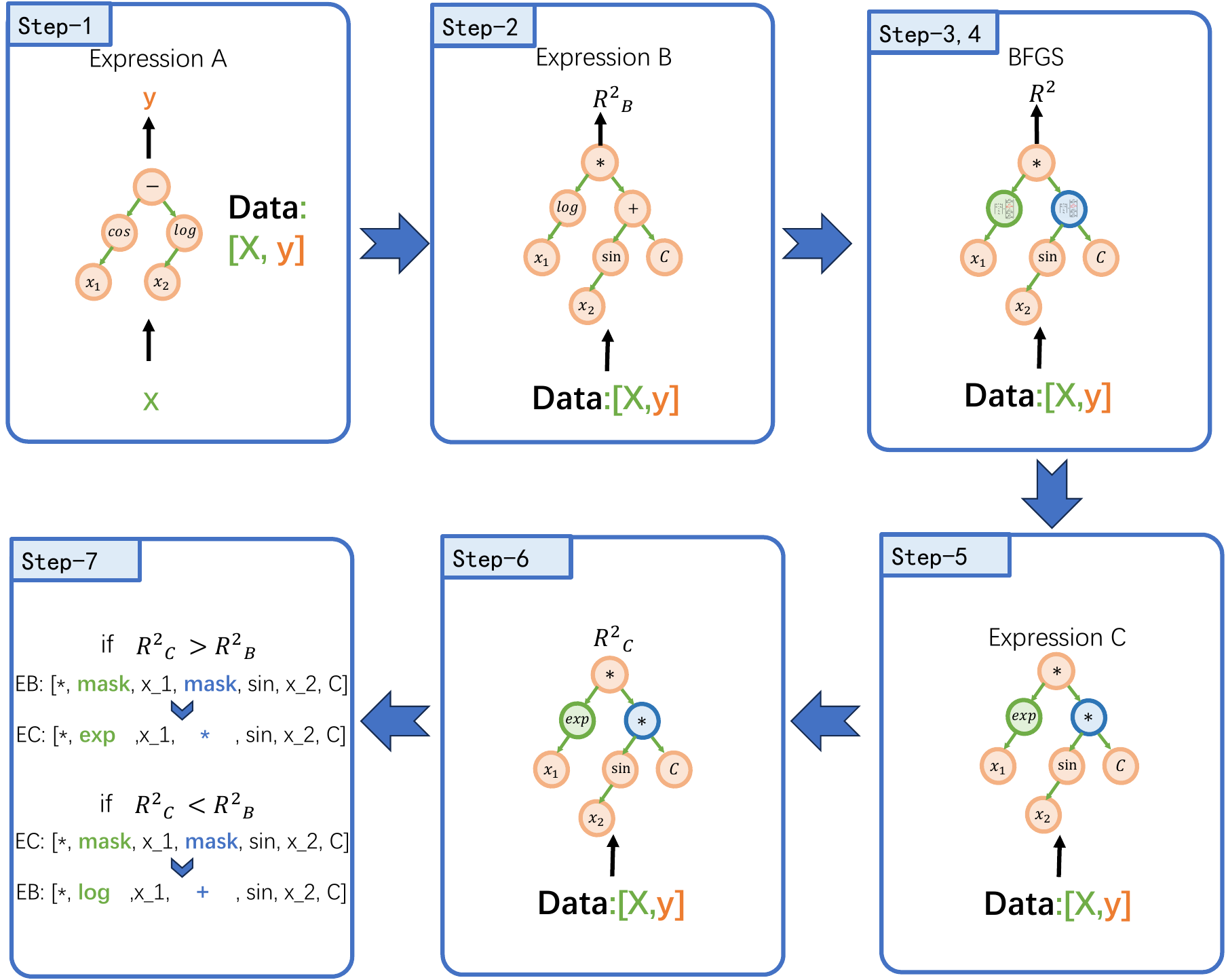}
    \label{fig2b}
  }

  \caption{
  Data collection flow for training BERT. Panel (a) shows the details of the unary mutation function and the binary mutation function. Panel (b) shows the detailed process of the whole data collection.
  }
  \label{fig2}
\end{wrapfigure}
\begin{enumerate}
    \item Randomly generate a target expression $A$, and sample a supervised dataset $(X, y)$ based on this expression;
    \item Randomly initialize a binary expression tree $B$, and compute its regression goodness-of-fit $R^2(B)$ on $(X, y)$;
    \item Select several nodes in expression $B$, and replace the corresponding operators or variables using a parameterized mutation function, constructing a continuously relaxed intermediate expression;
    \item Optimize the continuous parameters in the mutation function, as well as expression constants, using BFGS to obtain optimal mutation parameters;
    \item Discretize the optimized continuous mutation parameters into concrete symbols, thereby generating a mutated expression $C$;
    \item Compute the goodness-of-fit $R^2(C)$ of the mutated expression on $(X, y)$ and compare it with $R^2(B)$;
    \item If $R^2(C) > R^2(B)$, use $B$ as the input expression, mask the modified nodes, and assign the corresponding symbols in $C$ as supervision targets; otherwise, swap the roles of input and target, using $C$ as input and the corresponding symbols in $B$ as prediction targets.
\end{enumerate}

This construction strategy ensures that each training sample satisfies the constraint that the target expression achieves better fitting performance than the input expression, thereby guiding BERT to learn a regression-performance-oriented symbolic editing policy.

The mutation functions are divided into \textbf{unary mutation functions} and \textbf{binary mutation functions}, as shown in the figure. \ref{fig2a}. Both are formulated as weighted combinations of candidate operator outputs and variable outputs, where the weights are controlled by Softmax-normalized selection parameters. This design maintains differentiability of the mutation process in continuous space and enables efficient optimization.

In addition, in order to alleviate the problem of data leakage, we perform similarity retrieval and matching on each data point in the test data set after the generation of the training data. If there is the same preorder traversal, we will directly eliminate it from the training data.

\begin{figure*}[t]
\centering
\setlength{\belowcaptionskip}{-0.0cm} 
\includegraphics[width=130mm]{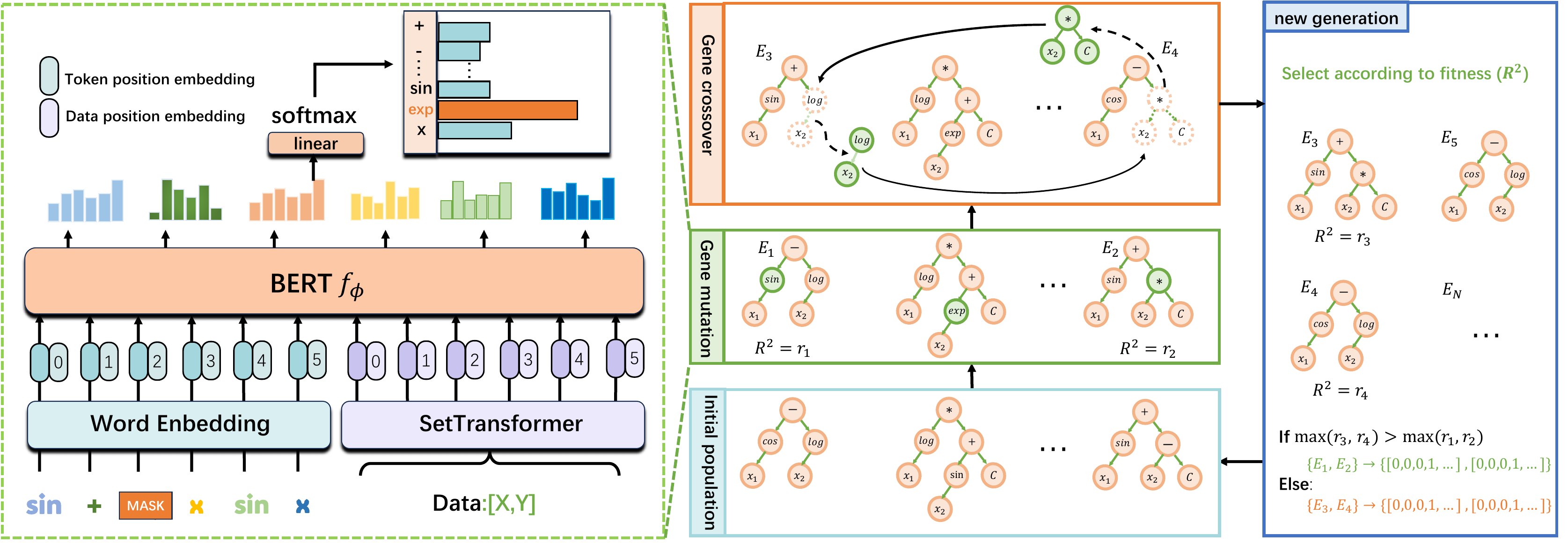}
\caption{Cross-guides Bert's training data collection flow.
}
\label{fig:crosser}
\end{figure*}
\subsubsection{Training data pair collection for Crossover}

The training data for the multimodal BERT used for crossover guidance are collected online during the \textbf{GP evolutionary process in which the mutation-guiding BERT has already been incorporated}. Specifically, for a given dataset $(X,Y)$, the population is first initialized by standard GP, and the mutation-guiding BERT trained in the first stage is then used to perform guided mutation on individuals, producing an intermediate population. Subsequently, a pair of candidate individuals, denoted by $E_1$ and $E_2$, is selected from this intermediate population, and their subtrees are exchanged according to the standard GP crossover mechanism to generate two offspring, $E_3$ and $E_4$.

Both the parent and offspring individuals are then subjected to constant optimization using BFGS, and their fitness values are evaluated according to regression goodness-of-fit metrics such as $R^2$, denoted by
\[
r_1,\;r_2,\;r_3,\;r_4.
\]
If
\[
\max(r_3,r_4)>\max(r_1,r_2),
\]
the crossover is regarded as effective, and the corresponding event is retained as a training sample.

To construct the sample, the two parent expressions $(E_1,E_2)$ are linearized into preorder traversal sequences,
\[
s_1^{(a)}=(t^{(a)}_1,\dots,t^{(a)}_{L_a}), \qquad
s_2^{(b)}=(t^{(b)}_1,\dots,t^{(b)}_{L_b}),
\]
and the positions of the two subtree root nodes actually involved in the crossover are encoded as supervision labels,
\[
y^{(a)}\in\{0,1\}^{L_a}, \qquad y^{(b)}\in\{0,1\}^{L_b},
\]
where only the crossover positions are assigned the value $1$. Accordingly, a training sample can be written as
\[
\bigl(s_1^{(a)},\,s_2^{(b)},\,[X,Y]\;;\;y^{(a)},\,y^{(b)}\bigr).
\]

Conversely, if
\[
\max(r_3,r_4)<\max(r_1,r_2),
\]
The construction is reversed: $E_3$ and $E_4$ are treated as the parent expressions to form the training sample, while the supervision labels remain unchanged. In this way, the crossover-guiding BERT learns crossover positions that can lead to fitness improvement under data constraints, rather than arbitrary crossover events, thereby better supporting directed crossover in the actual search process.

\section{Results}

In this paper, we conduct a comprehensive evaluation of GESR against four state-of-the-art baselines across 13 datasets, including SRBench. Our evaluation focuses on the following key aspects:

\ref{S4.1} How strong is the fitting performance of GESR?  

\ref{S4.2} How complex are the expressions discovered by GESR?  

\ref{S4.3} How efficient is GESR in inference?  

\ref{S4.4} How much does the introduction of BERT affect the performance of GESR?

\ref{S4.5} How does GESR perform in real physics applications?

\subsection{Fit Ability Test}
\label{S4.1}
In the comparative experiments, we strictly control experimental variables to ensure that all conditions other than algorithm type remain identical, thereby guaranteeing fairness and reliability. Specifically:  
(1) During the testing phase, all algorithms sample input $x_i$ within the same range for each target expression to ensure consistent test data distributions. The detailed sampling ranges are provided in Appendix \cref{a-tab1,a-tab2,a-tab3,a-tab5,a-tab6}.  
(2) For constraints introduced in GESR, we enable the corresponding constraints in baseline methods as well to prevent performance bias arising from constraint discrepancies, ensuring fair comparisons under identical rules.  
(3) The population size of GP-based methods is strictly kept consistent across all algorithms.

We adopt the coefficient of determination ($R^2$) \citep{r21,r22} as the primary metric for evaluating regression fitting performance. The computation of $R^2$ is defined as:
\begin{equation}
\label{e5}
\mathcal{R}^2 = 1 - \frac{\sum_{i=1}^{\mathcal{N}} (y_i - \hat{y}_i)^2}{\sum_{i=1}^{\mathcal{N}} (y_i - \overline{y})^2},
\end{equation}
where $\hat{y}_i$ denotes the predicted value, $\overline{y}$ represents the mean of the ground-truth labels $y$, and $\mathcal{N}$ is the number of test samples. We compare GESR with several widely used symbolic regression methods, including PySR, NGGP, GPlearn, and SNIP. The experimental results are summarized in Table~\ref{tab1}.

From the results, we observe that GESR achieves superior or comparable fitting performance to current state-of-the-art (SOTA) methods across multiple benchmark expressions and datasets, demonstrating its effectiveness in searching complex symbolic structures and performing high-precision regression.

\begin{table*}[t]
\renewcommand{\arraystretch}{1.1}
\centering
\resizebox{13.4cm}{!}{
\def\arraystretch{1.1}
    \small
        \bgroup
        \setlength{\tabcolsep}{0.4em}
        \begin{tabular}{c|l|cccccccccc}
\toprule
\multirow{2}{*}{\bf Group} & \multicolumn{1}{c|}{\multirow{2}{*}{\bf Dataset}} &
\multicolumn{2}{c}{\bf GESR}  & \multicolumn{2}{c}{\bf PYSR} &
\multicolumn{2}{c}{\bf NGGP} & \multicolumn{2}{c}{\bf SNIP } &
\multicolumn{2}{c}{\bf GPlearn} \\
\cmidrule(lr){3-12}
& & $R^2 \uparrow$ & Nodes $\downarrow$
& $R^2 \uparrow$ & Nodes $\downarrow$
& $R^2 \uparrow$ & Nodes $\downarrow$
& $R^2 \uparrow$ & Nodes $\downarrow$
& $R^2 \uparrow$ & Nodes $\downarrow$ \\
\cmidrule(lr){2-2}
\cmidrule(lr){3-4}
\cmidrule(lr){5-6}
\cmidrule(lr){7-8}
\cmidrule(lr){9-10}
\cmidrule(lr){11-12}

\multirow{10}{*}{\rotatebox{90}{Standards}}
& Nguyen
& \bestcell{$\mathbf{0.9999}_{\pm0.002}$} & \bestcell{13.6}
& \bestcell{$\mathbf{0.9999}_{\pm0.001}$} & \third{17.3}
& \bestcell{$\mathbf{0.9999}_{\pm0.004}$} & 20.3
& \bestcell{$\mathbf{0.9999}_{\pm0.003}$} & \second{16.6}
& \second{$0.9973_{\pm0.004}$} & 37.4 \\

& Keijzer
& \bestcell{$\mathbf{0.9998}_{\pm0.005}$} & \bestcell{14.7}
& \second{$0.9982_{\pm0.004}$} & 21.8
& $0.9924_{\pm0.005}$ & \second{18.4}
& \third{$0.9963_{\pm0.005}$} & \third{19.4}
& $0.9904_{\pm0.003}$ & 34.3 \\

& Korns
& \bestcell{$0.9964_{\pm0.004}$} & \bestcell{15.5}
& \second{$\mathbf{0.9987}_{\pm0.003}$} & \second{21.1}
& \third{$0.9872_{\pm0.006}$} & 26.6
& $0.9814_{\pm0.006}$ & \third{21.6}
& $0.9794_{\pm0.005}$ & 36.3 \\

& Constant
& \bestcell{$0.9997_{\pm0.003}$} & \bestcell{24.7}
& \second{$\mathbf{0.9996}_{\pm0.002}$} & \third{32.2}
& \third{$0.9988_{\pm0.005}$} & 38.2
& $0.9846_{\pm0.004}$ & \second{31.4}
& $0.9795_{\pm0.007}$ & 46.4 \\

& Livermore
& \bestcell{$\mathbf{0.9962}_{\pm0.006}$} & \bestcell{31.5}
& \second{$0.9886_{\pm0.004}$} & \third{38.3}
& \third{$0.9746_{\pm0.003}$} & 46.3
& $0.9695_{\pm0.005}$ & \second{35.2}
& $0.9014_{\pm0.007}$ & 49.8 \\

& Vladislavleva
& \bestcell{$0.9992_{\pm0.005}$} & \bestcell{20.2}
& \second{$\mathbf{0.9993}_{\pm0.004}$} & 31.6
& \third{$0.9963_{\pm0.006}$} & \third{31.3}
& $0.9877_{\pm0.006}$ & \second{28.4}
& $0.9623_{\pm0.006}$ & 38.7 \\

& R
& \bestcell{$\mathbf{0.9991}_{\pm0.006}$} & \bestcell{14.6}
& \second{$0.9881_{\pm0.004}$} & \second{19.3}
& \third{$0.9744_{\pm0.005}$} & \third{22.0}
& $0.9703_{\pm0.004}$ & 23.6
& $0.9253_{\pm0.005}$ & 28.1 \\

& Jin
& \bestcell{$\mathbf{0.9984}_{\pm0.005}$} & \bestcell{17.1}
& \third{$0.9914_{\pm0.003}$} & 28.8
& \second{$0.9916_{\pm0.004}$} & \third{22.0}
& $0.9824_{\pm0.005}$ & \second{19.5}
& $0.9601_{\pm0.007}$ & 32.4 \\

& Neat
& \bestcell{$\mathbf{0.9987}_{\pm0.006}$} & \bestcell{14.1}
& \second{$0.9944_{\pm0.006}$} & \second{19.6}
& $0.9827_{\pm0.004}$ & 22.7
& \third{$0.9879_{\pm0.005}$} & \third{21.4}
& $0.9426_{\pm0.006}$ & 32.3 \\

& Others
& \second{$0.9944_{\pm0.004}$} & \bestcell{23.9}
& \bestcell{$\mathbf{0.9952}_{\pm0.002}$} & \third{28.1}
& \third{$0.9861_{\pm0.003}$} & 33.2
& $0.9825_{\pm0.004}$ & \second{26.4}
& $0.9592_{\pm0.005}$ & 32.6 \\

\midrule
\multirow{3}{*}{\rotatebox{90}{SRBench}}
& Feynman
& \bestcell{$\mathbf{0.9946}_{\pm0.004}$} & \bestcell{17.2}
& \second{$0.9893_{\pm0.002}$} & 25.6
& $0.9710_{\pm0.004}$ & \second{24.1}
& \third{$0.9814_{\pm0.006}$} & \third{24.3}
& $0.9025_{\pm0.006}$ & 27.6 \\

& Strogatz
& \second{$0.9876_{\pm0.006}$} & \bestcell{21.7}
& \bestcell{$\mathbf{0.9923}_{\pm0.003}$} & \second{27.4}
& \third{$0.9613_{\pm0.003}$} & \third{27.8}
& $0.9366_{\pm0.005}$ & 29.6
& $0.8813_{\pm0.006}$ & 29.9 \\

& Black-box
& \bestcell{$\mathbf{0.9364}_{\pm0.003}$} & \bestcell{23.7}
& \second{$0.9037_{\pm0.004}$} & 38.3
& \third{$0.9033_{\pm0.004}$} & \third{32.9}
& $0.8984_{\pm0.005}$ & \second{24.9}
& $0.8738_{\pm0.004}$ & 39.2 \\

\midrule
\multirow{1}{*}{\rotatebox{90}{ }}
& Average
& \bestcell{$\mathbf{0.9903}$} & \bestcell{22.0}
& \second{$0.9876$} & \third{26.9}
& \third{$0.9784$} & 28.1
& $0.9738$ & \second{24.8}
& $0.9427$ & 35.8 \\
\bottomrule

\end{tabular}
\egroup
}
\caption{The results of performance comparison. At a 0.95 confidence level, a comparison of the coefficient of determination ($R^2$) and the expression complexity (Nodes) was conducted between \textcolor{orange}{GESR} and four baselines.}
\label{tab1}
\end{table*}

\subsection{Result Complexity Test}
\label{S4.2}
In symbolic regression, our ultimate goal is not only to achieve high fitting accuracy but also to produce compact and highly interpretable analytic expressions. If the generated expressions become overly complex, their physical meaning and human interpretability may significantly deteriorate. Therefore, we measure result complexity by counting the number of nodes in the corresponding expression binary tree.

To ensure a fair comparison, we select a subset of expressions from each benchmark dataset such that all competing methods achieve $R^2 > 0.999$, and use this subset as the complexity evaluation set.

For each expression in this set, every algorithm is executed independently 20 times, and the average number of nodes in the generated expressions is reported. Additionally, to avoid bias caused by inconsistent search spaces, we set the maximum expression length to 80 for all methods.
The experimental results are summarized in Table~\ref{tab1}. As shown in the table, GESR is able to generate more compact expression structures on most datasets, with an average node count significantly lower than that of several mainstream symbolic regression methods, and comparable to or better than the strongest baseline. These findings indicate that the gene-editing-driven search mechanism in GESR not only improves fitting performance but also effectively suppresses expression bloat, thereby producing concise and interpretable analytic formulas while maintaining high accuracy.
\begin{wrapfigure}{r}{0.46\textwidth}
  \centering
  \vspace{-12pt}
  \includegraphics[width=\linewidth]{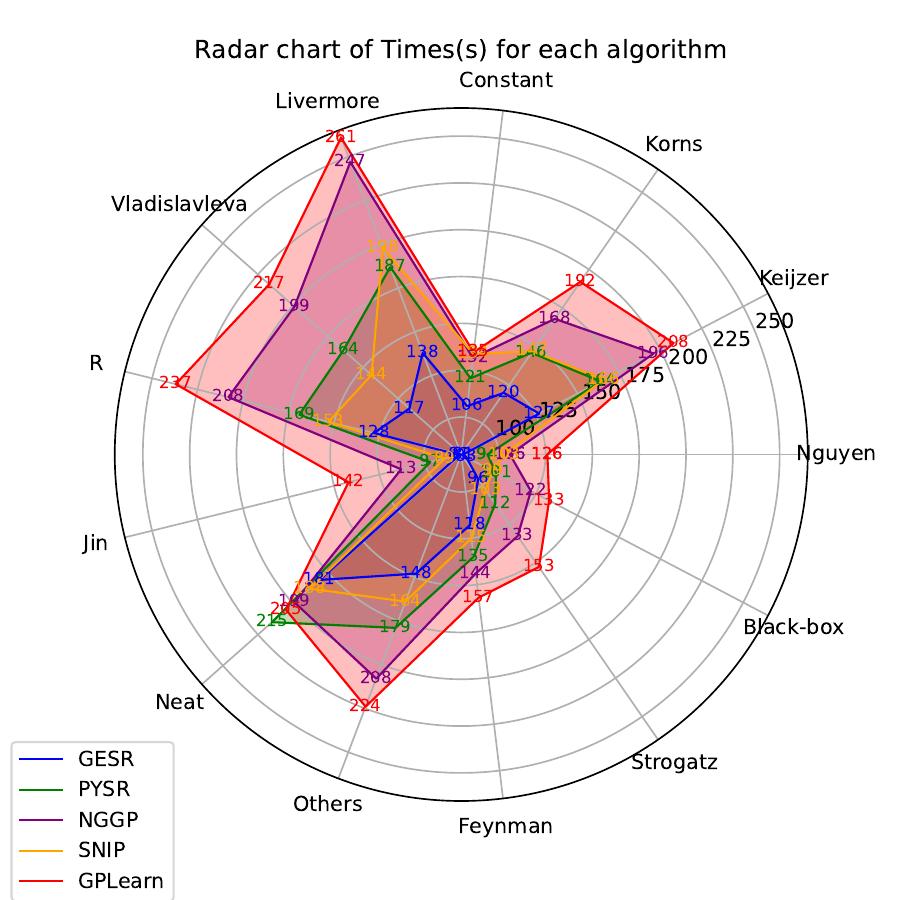}
  \caption{Inference time (s) radar chart for each algorithm.}
  \label{fig:radar}
  \vspace{-10pt}
\end{wrapfigure}

\subsection{Inference Time Test}
\label{S4.3}
Beyond fitting accuracy, inference efficiency is a critical factor in evaluating the practical applicability of symbolic regression algorithms, particularly for our method. To systematically assess the inference efficiency of GESR, we select a subset of expressions from all benchmark datasets such that all compared algorithms achieve $R^2 > 0.999$, and define this subset as the evaluation dataset $\mathbb{A}$. On dataset $\mathbb{A}$, we repeat the inference process of each algorithm 10 times and record the average inference time required to reach $R^2 = 0.999$.

We visualize the inference time of different algorithms using a radar chart, as shown in Fig.~\ref{fig2}. Each color in the figure corresponds to a specific method. From the chart, it can be observed that GESR achieves significantly superior inference efficiency compared to most baseline approaches while maintaining high fitting accuracy, providing strong empirical evidence of its efficiency advantages on standard benchmark tasks.

\begin{wraptable}{r}{0.58\textwidth}
\vspace{-0.35cm}
\centering
\fontsize{5.8pt}{6.6pt}\selectfont
\setlength{\tabcolsep}{1.2pt}
\renewcommand{\arraystretch}{0.90}

\resizebox{\linewidth}{!}{
\begin{tabular}{lccccccccc}
\toprule[1.2pt]
& \multicolumn{3}{c}{GESR (M+C)}
& \multicolumn{3}{c}{GESR (M)}
& \multicolumn{3}{c}{GESR NO} \\
\cmidrule(lr){2-4}
\cmidrule(lr){5-7}
\cmidrule(lr){8-10}
& $R^2 \uparrow$ & Nodes $\downarrow$ & Time $\downarrow$
& $R^2 \uparrow$ & Nodes $\downarrow$ & Time $\downarrow$
& $R^2 \uparrow$ & Nodes $\downarrow$ & Time $\downarrow$ \\

\midrule
Nguyen        & \cellcolor{gray!25}0.9999 & \cellcolor{gray!25}13.6 & \cellcolor{gray!25}81
              & \cellcolor{gray!25}0.9999 & 14.5 & 84
              & 0.9572 & 33.9 & 118 \\
Keijzer       & \cellcolor{gray!25}0.9998 & \cellcolor{gray!25}14.7 & \cellcolor{gray!25}127
              & 0.9997 & 16.2 & 133
              & 0.9514 & 32.8 & 176 \\
Korns         & \cellcolor{gray!25}0.9964 & \cellcolor{gray!25}15.5 & \cellcolor{gray!25}120
              & 0.9938 & 16.8 & 123
              & 0.9489 & 35.6 & 161 \\
Constant      & \cellcolor{gray!25}0.9997 & \cellcolor{gray!25}24.7 & \cellcolor{gray!25}106
              & 0.9995 & 28.5 & 108
              & 0.9531 & 41.2 & 144 \\
Livermore     & \cellcolor{gray!25}0.9962 & \cellcolor{gray!25}31.5 & \cellcolor{gray!25}138
              & 0.9903 & 36.4 & 145
              & 0.9427 & 48.6 & 228 \\
Vladislavleva & \cellcolor{gray!25}0.9992 & \cellcolor{gray!25}20.2 & \cellcolor{gray!25}117
              & 0.9986 & 24.8 & 122
              & 0.9498 & 39.5 & 182 \\
R             & \cellcolor{gray!25}0.9991 & \cellcolor{gray!25}14.6 & \cellcolor{gray!25}128
              & 0.9962 & 15.7 & 135
              & 0.9540 & 33.7 & 194 \\
Jin           & \cellcolor{gray!25}0.9984 & \cellcolor{gray!25}17.1 & \cellcolor{gray!25}81
              & 0.9974 & 20.2 & 84
              & 0.9566 & 36.9 & 129 \\
Neat          & \cellcolor{gray!25}0.9987 & \cellcolor{gray!25}14.1 & \cellcolor{gray!25}181
              & 0.9983 & 16.3 & 187
              & 0.9524 & 34.8 & 206 \\
Others        & \cellcolor{gray!25}0.9944 & \cellcolor{gray!25}23.9 & \cellcolor{gray!25}148
              & 0.9938 & 26.9 & 152
              & 0.9549 & 42.5 & 192 \\
Feynman       & \cellcolor{gray!25}0.9946 & \cellcolor{gray!25}17.2 & \cellcolor{gray!25}118
              & 0.9925 & 20.4 & 121
              & 0.9482 & 35.1 & 171 \\
Strogatz      & \cellcolor{gray!25}0.9876 & \cellcolor{gray!25}21.7 & \cellcolor{gray!25}96
              & 0.9813 & 23.6 & 99
              & 0.9396 & 46.3 & 156 \\
Black-box     & \cellcolor{gray!25}0.9364 & \cellcolor{gray!25}23.7 & \cellcolor{gray!25}83
              & 0.9331 & 25.4 & 84
              & 0.9518 & 41.6 & 134 \\

\midrule
Average       & \cellcolor{gray!25}0.9903 & \cellcolor{gray!25}22.0 & \cellcolor{gray!25}117.2
              & 0.9903 & 22.0 & 124.4
              & 0.9508 & 38.6 & 176.1 \\
\bottomrule
\end{tabular}
}

\caption{Ablation study on GESR. Gray cells indicate the best performance across the three variants for each metric. M denotes mutation and C denotes crossover.}
\label{pk}
\vspace{-0.3cm}
\end{wraptable}

We attribute the high inference efficiency of GESR primarily to its gene-editing-driven search mechanism. Conventional symbolic regression methods typically rely on random mutation, reinforcement learning, or evolutionary search to perform large-scale discrete combinatorial optimization in high-dimensional structural spaces, resulting in substantial inference costs. In contrast, GESR reformulates symbolic mutation as a learning-guided local gene editing problem. By leveraging BERT to predict the most promising symbol substitutions in a continuous embedding space, GESR effectively transforms the original combinatorial structural search into a constrained numerical optimization and conditional generation problem.

\begin{figure*}[t]
    \centering
       \subfloat[]{
       \includegraphics[width=0.98\linewidth]{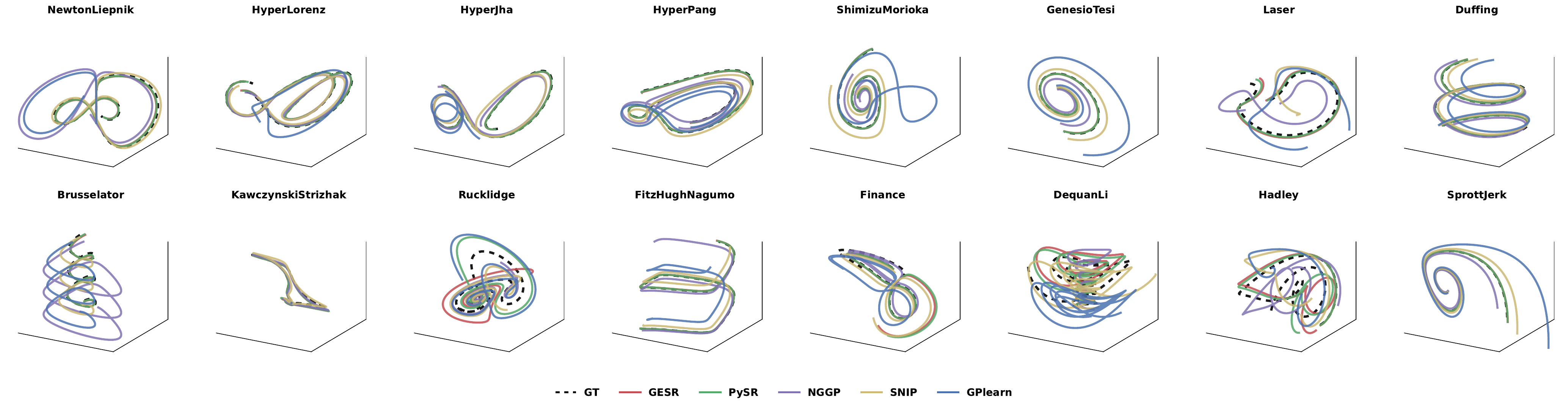} \label{fig3a}}
        \\
        \subfloat[]{
        \vspace{-0.5cm}
        \hspace{-0.3cm}
        \includegraphics[width=0.98 \linewidth]{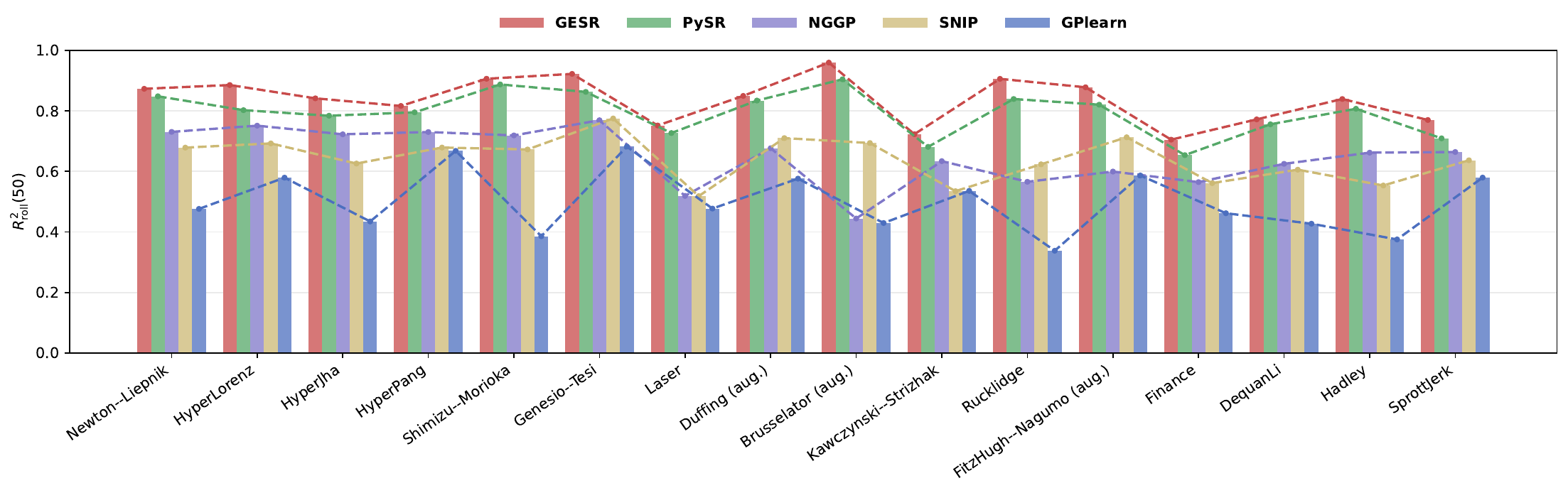}\label{fig3b}}   
	  \caption{
   Chaotic system trajectory prediction results. Where panel a shows the curves of the predicted and true trajectories of GESR and baselines. Figure b shows the Short-horizon rollout consistency of each algorithm  
  }
\label{fig3} 
\end{figure*}

\subsection{Ablation Study on the Gene Editing Module}
\label{S4.4}

To quantitatively evaluate the roles of the two components of the gene editing module, namely \textbf{mutation guidance} and \textbf{crossover guidance}, we conduct a staged ablation study within the GESR framework. Specifically, we compare the following three configurations:
\begin{itemize}
    \item \textbf{GESR NO}: no learning-based guidance is introduced, and standard GP with random mutation and random crossover is used;
    \item \textbf{GESR (M)}: only the mutation-guiding BERT is introduced, while crossover remains random;
    \item \textbf{GESR (M+C)}: both mutation guidance and crossover guidance are introduced, corresponding to the full GESR model.
\end{itemize}
All three methods are evaluated under the same datasets, initialization schemes, computational budgets, and hyperparameter settings. The results are reported in Table~\ref{pk}.

From the overall trend, introducing mutation guidance leads to clear improvements over the unguided version in fitting performance, expression complexity, and search efficiency, indicating that data-conditioned directed mutation can effectively improve search quality and suppress unnecessary structural expansion. On top of this, further introducing crossover guidance yields additional overall gains. Notably, even though an extra BERT module is involved, the inference efficiency is still improved. This suggests that crossover guidance can further improve the quality of structural recombination and reduce the redundant search caused by inefficient crossover operations.

Overall, the ablation study shows that both mutation guidance and crossover guidance substantially enhance the fitting performance, search efficiency, and stability of the model. When combined, the full GESR achieves the best balance among accuracy, compactness, and efficiency.

\subsection{Vector Field Learning on Chaotic Systems}
\label{S4.5}
Chaotic dynamical systems provide a stringent benchmark for evaluating the learning capacity of data-driven models, as their dynamics are governed by strong nonlinearities and multivariate couplings. In such systems, even small local modeling errors can rapidly amplify during numerical integration, inevitably leading to long-term trajectory divergence. Therefore, evaluation should emphasize the ability to learn the local vector field structure, rather than long-horizon pointwise trajectory matching. We evaluate the vector-field modeling performance of five symbolic regression methods across 16 chaotic systems in table\ref{table:cs_1to4},\ref{table:cs_5to8},\ref{table:cs_9to12},\ref{table:cs_13to16}. And shown in fig. \ref{fig_16c} 

All systems are governed by autonomous ordinary differential equations:
\begin{equation}
\dot{\mathbf{x}}(t) = \mathbf{f}(\mathbf{x}(t)),
\end{equation}
and are reformulated as a supervised regression task from state to time derivative. All methods are trained under identical data splits, numerical differentiation procedures, and training configurations to ensure a fair comparison.

\subsection{Reformulating Dynamical Systems as Symbolic Regression}

To identify analytical representations of chaotic systems without direct access to the underlying differential equations, we cast dynamical modeling as a symbolic regression problem for vector field approximation. For each system, we first generate state trajectories $\mathbf{x}(t)$ via high-precision numerical integration from a given initial condition, and uniformly sample discrete state observations $\{\mathbf{x}_i\}_{i=1}^{N}$. Since the true derivatives $\dot{\mathbf{x}}(t)$ are not directly observable, we estimate velocity targets $\tilde{\dot{\mathbf{x}}}_i$ using smoothed numerical differentiation. This yields a supervised dataset
\[
\mathcal{D}=\{(\mathbf{x}_i,\tilde{\dot{\mathbf{x}}}_i)\}_{i=1}^{N},
\]
and we learn an analytic mapping $f_j:\mathbb{R}^d\to\mathbb{R}$ for each state dimension such that
\begin{equation}
\tilde{\dot{x}}_{i,j}\approx f_j(\mathbf{x}_i).
\end{equation}
The resulting vector field $\hat{\mathbf{f}}(\mathbf{x})=(f_1(\mathbf{x}),\dots,f_d(\mathbf{x}))$ constitutes a symbolic representation of the system dynamics. This formulation converts chaotic system identification into a multi-output (dimension-wise factorized) symbolic regression problem, whose induced trajectories can be further evaluated through numerical integration.

We adopt two complementary evaluation metrics:  

(1) \textbf{Derivative fitting accuracy} $R^2_{\dot{\mathbf{x}}}$, measuring local vector-field regression quality, is shows in Table \ref{tab:chaos_r2_full_part1} and  \ref{tab:chaos_r2_full_part2};  

(2) \textbf{Short-horizon rollout consistency} $R^2_{\mathrm{roll}}(50)$, assessing dynamical stability within a finite Lyapunov time window, is shown in \ref{fig3}.

Experimental results show that GESR achieves the highest or tied-best vector-field fitting accuracy on most chaotic systems, and exhibits the most stable short-term trajectory consistency. Phase-space trajectory visualizations further demonstrate that GESR better preserves the geometric structure of true attractors over short to medium time horizons, whereas baselines are more prone to trajectory drift or structural degradation.

\section{Discussion}
\label{sec:results_discussion}

GESR introduces two BERT models to jointly guide the mutation and crossover processes in genetic programming, replacing traditional random genetic operators with context-aware ``gene editing'' operations. Specifically, mathematical expressions are represented as symbolic sequences. During mutation, the mutation-guiding BERT predicts suitable replacements for masked positions, thereby generating expressions that are better aligned with the given data. During crossover, the crossover-guiding BERT takes as input a pair of candidate expressions together with the corresponding dataset, and predicts the subtree root nodes that are most suitable for exchange, thereby enabling data-driven structural recombination. This dual guidance mechanism preserves population diversity while effectively reducing ineffective and destructive mutations and crossovers, transforming the search process from unguided exploration into directed optimization. Experimental results demonstrate that GESR achieves high predictive accuracy across multiple symbolic regression benchmarks and outperforms conventional GP methods in terms of convergence speed, search efficiency, and stability, thereby validating the effectiveness of learning-guided mutation and crossover for symbolic regression.

\textbf{Limitations and Future Directions}

Despite its strong performance and efficiency, GESR still has several limitations. First, GESR depends to some extent on the quality of the pretrained BERT models. 

Second, the current model has limitations on the dimensionality of the input data. When the dimensionality of the test data exceeds the range supported during pretraining, the model can no longer function properly. Therefore, GESR still faces scalability challenges on very large-scale datasets.

Future work will focus on developing lighter and more transferable symbolic editing models, designing hybrid search mechanisms that balance guidance and diversity, and exploring scalable frameworks for high-dimensional inputs and very large-scale datasets.

\nocite{langley00}

\section{Code availability}
After the paper is accepted, we will open-source the source code
\section{Acknowledgments}

\paragraph{Funding:}
This work was supported in part by the National Natural Science Foundation of China under Grant 92370117, in part by CAS Project for Young Scientists in Basic Research under Grant YSBR-090, and in part by the Key Research Program of the Chinese Academy of Sciences under Grant XDPB22.

\paragraph{Competing interests:}
All authors of the article have no competing interests.

\bibliography{sn-bibliography}

\newpage
\appendix
\onecolumn
\section*{Appendix For GESR: A Genetic Programming–Based Symbolic Regression Method with Gene
Editing}

\startcontents[app]             
\printcontents[app]{ }{1}{\setcounter{tocdepth}{3}} 
\newpage

\section{pseudocode for GESR}

\subsection{GESR: Learning-guided Genetic Programming Framework.}
We propose \textbf{GESR}, a symbolic regression framework that integrates multimodal representation learning with genetic programming. Instead of relying on purely random mutation and crossover, GESR employs two multimodal BERT editors to perform data-conditioned, structure-aware gene editing. For mutation, each symbolic expression is linearized into a preorder sequence and partially masked, and the mutation-guiding BERT predicts replacement symbols by jointly attending to symbolic context and dataset embeddings encoded by a Set Transformer. For crossover, two mutated expressions are concatenated with a special separator token and fed, together with the dataset embedding, into a second BERT model that predicts the subtree root nodes to be exchanged. The edited population is then refined through BFGS-based constant optimization, and individuals are selected according to regression fitness (e.g., $R^2$), guiding the search toward more accurate and compact symbolic models.

\begin{algorithm}[H]
\caption{GESR: Learning-guided Genetic Programming for Symbolic Regression}
\label{alg:gesr}
\begin{algorithmic}[1]
\REQUIRE dataset $\mathcal{D}=(X,y)$; population size $N$; generations $G$; grammar $\mathcal{G}$; mutation editor $\textsc{BERT}^{\mathrm{mut}}_{\phi}$; crossover editor $\textsc{BERT}^{\mathrm{cross}}_{\psi}$; data encoder \textsc{SetTransformer}; mask rate $\rho$; crossover rate $\gamma$; elite size $K$
\ENSURE best expression $E^\star$

\STATE $P \leftarrow \{\textsc{RandomExprTree}(\mathcal{G})\}_{i=1}^{N}$; \quad $E_{\text{data}} \leftarrow \textsc{SetTransformer}(X,y)$
\STATE $E^\star \leftarrow \varnothing$; \quad $f^\star \leftarrow -\infty$

\FOR{$g=1,\dots,G$}
    \STATE $P_{\text{mut}} \leftarrow \emptyset$
    \FORALL{$E \in P$}
        \STATE $s \leftarrow \textsc{PreorderLinearize}(E)$; \quad $\tilde{s} \leftarrow \textsc{Mask}(s,\rho)$
        \STATE $p^{\mathrm{mut}} \leftarrow \textsc{BERT}^{\mathrm{mut}}_{\phi}(\tilde{s},E_{\text{data}})$
        \STATE $s' \leftarrow \textsc{FillMaskByArgmax}(\tilde{s},p^{\mathrm{mut}})$
        \STATE $P_{\text{mut}} \leftarrow P_{\text{mut}} \cup \{\textsc{DecodeAndRepairToTree}(s',\mathcal{G})\}$
    \ENDFOR

    \STATE $P_{\text{cross}} \leftarrow \emptyset$
    \FORALL{$(E^{(a)},E^{(b)}) \in \textsc{SamplePairs}(P_{\text{mut}},\gamma)$}
        \STATE $s^{(a)} \leftarrow \textsc{PreorderLinearize}(E^{(a)})$; \quad $s^{(b)} \leftarrow \textsc{PreorderLinearize}(E^{(b)})$
        \STATE $s^{(ab)} \leftarrow \bigl(s^{(a)},\texttt{[SEP]},s^{(b)}\bigr)$
        \STATE $(p^{(a)},p^{(b)}) \leftarrow \textsc{BERT}^{\mathrm{cross}}_{\psi}(s^{(ab)},E_{\text{data}})$
        \STATE $i^\star \leftarrow \arg\max_i p^{(a)}_i$; \quad $j^\star \leftarrow \arg\max_j p^{(b)}_j$
        \STATE $(\hat{E}^{(a)},\hat{E}^{(b)}) \leftarrow \textsc{SwapSubtrees}\!\left(E^{(a)},E^{(b)},i^\star,j^\star,\mathcal{G}\right)$
        \STATE $P_{\text{cross}} \leftarrow P_{\text{cross}} \cup \{\hat{E}^{(a)},\hat{E}^{(b)}\}$
    \ENDFOR
    \STATE $P_{\text{cross}} \leftarrow \textsc{ResizePopulation}(P_{\text{cross}},P_{\text{mut}},N)$

    \STATE $\mathcal{F} \leftarrow [\;]$
    \FORALL{$E \in P_{\text{cross}}$}
        \STATE $\hat{E} \leftarrow \textsc{BFGSOptimizeConstants}(E,X,y)$
        \STATE $f \leftarrow \textsc{FitnessR2}(\hat{E},X,y)$
        \STATE $\mathcal{F}.\textsc{append}((\hat{E},f))$
        \IF{$f > f^\star$}
            \STATE $f^\star \leftarrow f$; \quad $E^\star \leftarrow \hat{E}$
        \ENDIF
    \ENDFOR

    \STATE $P \leftarrow \textsc{SelectNextPopulation}(\mathcal{F},N,K)$
\ENDFOR

\STATE \textbf{return} $E^\star$
\end{algorithmic}
\end{algorithm}

\subsection{Training pair construction pipeline for mutation guidance. }
As illustrated in the pseudocode, we design a training pair construction pipeline based on regression performance improvement to supervise the multimodal BERT in learning effective symbolic editing strategies. The pipeline first randomly generates a target expression and samples a regression dataset $(X, y)$, then initializes a candidate expression $B$ and evaluates its fitting performance. Next, differentiable mutation operators are applied to $B$, and the continuous mutation parameters as well as expression constants are optimized using BFGS, producing a discretized mutated expression $C$. By comparing $R^2(B)$ and $R^2(C)$, the expression with superior fitting performance is consistently selected as the supervision target, while the corresponding nodes of the inferior expression are masked to form masked-input–target-token training samples. This strategy guarantees that all training pairs reflect performance improvement, thereby guiding BERT to learn structure-aware symbolic edits oriented toward regression accuracy.
\begin{algorithm}[H]
\caption{Training Pair Collection for Multimodal BERT Editor (GESR)}
\label{alg:collect_pairs}
\begin{algorithmic}[1]
\REQUIRE $N_{\text{pairs}}$; grammar $\mathcal{G}$; unary candidates $\mathcal{C}_u$; binary candidates $\mathcal{C}_b$; mask token \texttt{[MASK]}
\ENSURE Training set $\mathcal{T}=\{(\tilde{s},X,y,\mathcal{Y}_{\text{tok}})\}$

\STATE $\mathcal{T} \leftarrow \emptyset$
\WHILE{$|\mathcal{T}| < N_{\text{pairs}}$}
    \STATE $A \leftarrow \textsc{RandomExprTree}(\mathcal{G})$ \hfill \algc{random target expression}
    \STATE $(X,y) \leftarrow \textsc{SampleDataFromExpr}(A)$ \hfill \algc{sample supervised data}

    \STATE $B \leftarrow \textsc{RandomExprTree}(\mathcal{G})$ \hfill \algc{random initial expression}
    \STATE $R^2_B \leftarrow \textsc{FitnessR2}(B,X,y)$

    \STATE $\mathcal{N} \leftarrow \textsc{SelectMutationNodes}(B)$ \hfill \algc{choose nodes to mutate}
    \STATE $\Theta \leftarrow \textsc{InitContinuousParams}(\mathcal{N})$ \hfill \algc{softmax logits + constants}

    \STATE $B_{\mathrm{rel}}(\cdot;\Theta) \leftarrow \textsc{ReplaceWithDiffMutations}(B,\mathcal{N},\Theta,\mathcal{C}_u,\mathcal{C}_b)$
    \STATE $\Theta^\star \leftarrow \textsc{BFGS}\Big(\min_{\Theta}\;\textsc{MSE}(B_{\mathrm{rel}}(X;\Theta),y)\Big)$ \hfill \algc{optimize params}

    \STATE $C \leftarrow \textsc{DiscretizeMutation}(B_{\mathrm{rel}},\Theta^\star)$ \hfill \algc{argmax discretization}
    \STATE $C \leftarrow \textsc{SimplifyIfNeeded}(C)$
    \STATE $R^2_C \leftarrow \textsc{FitnessR2}(C,X,y)$

    \IF{$R^2_C > R^2_B$}
        \STATE $E_{\text{in}} \leftarrow B$; \;\; $E_{\text{out}} \leftarrow C$ \hfill \algc{ensure target is better}
    \ELSE
        \STATE $E_{\text{in}} \leftarrow C$; \;\; $E_{\text{out}} \leftarrow B$ \hfill \algc{swap roles}
    \ENDIF

    \STATE $s_{\text{in}} \leftarrow \textsc{PreorderLinearize}(E_{\text{in}})$
    \STATE $s_{\text{out}} \leftarrow \textsc{PreorderLinearize}(E_{\text{out}})$

    \STATE $\mathcal{I} \leftarrow \textsc{FindChangedPositions}(s_{\text{in}},s_{\text{out}},\mathcal{N})$ \hfill \algc{indices to mask}
    \STATE $\tilde{s} \leftarrow s_{\text{in}}$
    \FORALL{$k \in \mathcal{I}$}
        \STATE $\tilde{s}[k] \leftarrow \texttt{[MASK]}$
    \ENDFOR

    \STATE $\mathcal{Y}_{\text{tok}} \leftarrow \{(k, s_{\text{out}}[k]) \mid k \in \mathcal{I}\}$ \hfill \algc{supervision tokens}
    \STATE $\mathcal{T} \leftarrow \mathcal{T} \cup \{(\tilde{s},X,y,\mathcal{Y}_{\text{tok}})\}$
\ENDWHILE
\STATE \textbf{return} $\mathcal{T}$
\end{algorithmic}
\end{algorithm}

\subsection{Training pair construction pipeline for crossover guidance}

To train the multimodal BERT for \textbf{crossover guidance}, we construct training pairs based on fitness improvement during GP evolution. Unlike mutation supervision, the crossover-guiding BERT does not predict replacement symbols; instead, it learns to identify the subtree root nodes in two candidate expressions that are most suitable for exchange under the given dataset $(X,y)$.

Specifically, the training data are collected online during GP evolution after the mutation-guiding BERT has already been introduced. For a given dataset $(X,y)$, we first initialize a population and apply the mutation-guiding BERT to obtain an intermediate mutated population. Then, a pair of candidate expressions is selected and crossed once using the standard GP crossover operator, producing two offspring. Both parent and offspring expressions are further optimized by BFGS and evaluated by regression fitness such as $R^2$. If the offspring outperform the parents, the parent pair is retained as the input and the actual crossover nodes are used as supervision labels; otherwise, the offspring pair is used as the input while the crossover labels remain unchanged. In this way, each training pair always corresponds to a crossover event associated with fitness improvement.

To match the input format of the crossover-guiding BERT, the two expressions are linearized into preorder traversal sequences and concatenated with a separator token: $s^{(ab)}=\bigl(s^{(a)},\texttt{[SEP]},s^{(b)}\bigr).$
The crossover root nodes are encoded as one-hot position labels $y^{(a)}$ and $y^{(b)}$, indicating the subtree roots selected in the two expressions. Therefore, each crossover training sample can be written as
$
\bigl(s^{(ab)},X,y;\,y^{(a)},y^{(b)}\bigr).
$
This construction strategy ensures that the supervision signal always reflects crossover locations that are beneficial under data constraints, thereby enabling the model to learn a data-driven subtree exchange policy.

\begin{algorithm}[H]
\caption{Training Pair Collection for Multimodal BERT Crossover Editor}
\label{alg:collect_crossover_pairs}
\begin{algorithmic}[1]
\REQUIRE $N_{\text{pairs}}$; grammar $\mathcal{G}$; population size $N$; mutation editor $\textsc{BERT}^{\text{mut}}_{\phi}$; crossover rate $\gamma$
\ENSURE Training set $\mathcal{T}_{\text{cross}}=\{(s^{(ab)},X,y,y^{(a)},y^{(b)})\}$

\STATE $\mathcal{T}_{\text{cross}} \leftarrow \emptyset$
\WHILE{$|\mathcal{T}_{\text{cross}}| < N_{\text{pairs}}$}
    \STATE $A \leftarrow \textsc{RandomExprTree}(\mathcal{G})$ \hfill \algc{sample target expression}
    \STATE $(X,y) \leftarrow \textsc{SampleDataFromExpr}(A)$ \hfill \algc{build regression dataset}
    \STATE $P \leftarrow \{\textsc{RandomExprTree}(\mathcal{G})\}_{i=1}^{N}$ \hfill \algc{initialize GP population}

    \STATE $P_{\text{mut}} \leftarrow \emptyset$
    \FORALL{$E \in P$}
        \STATE $s \leftarrow \textsc{PreorderLinearize}(E)$ \hfill \algc{tree $\rightarrow$ token sequence}
        \STATE $\tilde{s} \leftarrow \textsc{Mask}(s)$ \hfill \algc{simulate mutation input}
        \STATE $p^{\text{mut}} \leftarrow \textsc{BERT}^{\text{mut}}_{\phi}(\tilde{s},X,y)$ \hfill \algc{guided mutation prediction}
        \STATE $s' \leftarrow \textsc{FillMaskByArgmax}(\tilde{s},p^{\text{mut}})$ \hfill \algc{replace masked tokens}
        \STATE $E' \leftarrow \textsc{DecodeAndRepairToTree}(s',\mathcal{G})$ \hfill \algc{recover valid expression}
        \STATE $P_{\text{mut}} \leftarrow P_{\text{mut}} \cup \{E'\}$
    \ENDFOR

    \STATE $(E_1,E_2) \leftarrow \textsc{SampleParentPair}(P_{\text{mut}},\gamma)$ \hfill \algc{choose crossover parents}
    \STATE $(u_1,u_2) \leftarrow \textsc{RandomCrossoverNodes}(E_1,E_2)$ \hfill \algc{standard GP crossover points}
    \STATE $(E_3,E_4) \leftarrow \textsc{SwapSubtrees}(E_1,u_1,E_2,u_2)$ \hfill \algc{generate offspring}

    \STATE $r_1 \leftarrow \textsc{FitnessR2}(\textsc{BFGSOptimizeConstants}(E_1,X,y),X,y)$ \hfill \algc{parent 1 fitness}
    \STATE $r_2 \leftarrow \textsc{FitnessR2}(\textsc{BFGSOptimizeConstants}(E_2,X,y),X,y)$ \hfill \algc{parent 2 fitness}
    \STATE $r_3 \leftarrow \textsc{FitnessR2}(\textsc{BFGSOptimizeConstants}(E_3,X,y),X,y)$ \hfill \algc{offspring 1 fitness}
    \STATE $r_4 \leftarrow \textsc{FitnessR2}(\textsc{BFGSOptimizeConstants}(E_4,X,y),X,y)$ \hfill \algc{offspring 2 fitness}

    \IF{$\max(r_3,r_4) > \max(r_1,r_2)$}
        \STATE $(E^{(a)}_{\text{in}},E^{(b)}_{\text{in}}) \leftarrow (E_1,E_2)$ \hfill \algc{keep better direction}
    \ELSE
        \STATE $(E^{(a)}_{\text{in}},E^{(b)}_{\text{in}}) \leftarrow (E_3,E_4)$ \hfill \algc{reverse source-target roles}
    \ENDIF

    \STATE $s^{(a)} \leftarrow \textsc{PreorderLinearize}(E^{(a)}_{\text{in}})$ \hfill \algc{first input sequence}
    \STATE $s^{(b)} \leftarrow \textsc{PreorderLinearize}(E^{(b)}_{\text{in}})$ \hfill \algc{second input sequence}
    \STATE $s^{(ab)} \leftarrow \bigl(s^{(a)},\texttt{[SEP]},s^{(b)}\bigr)$ \hfill \algc{joint BERT input}

    \STATE $i \leftarrow \textsc{NodeToPreorderIndex}(u_1,E_1)$ \hfill \algc{true crossover root in expr 1}
    \STATE $j \leftarrow \textsc{NodeToPreorderIndex}(u_2,E_2)$ \hfill \algc{true crossover root in expr 2}
    \STATE $y^{(a)} \leftarrow \textsc{OneHot}(i,|s^{(a)}|)$ \hfill \algc{position label for expr 1}
    \STATE $y^{(b)} \leftarrow \textsc{OneHot}(j,|s^{(b)}|)$ \hfill \algc{position label for expr 2}

    \STATE $\mathcal{T}_{\text{cross}} \leftarrow \mathcal{T}_{\text{cross}} \cup \{(s^{(ab)},X,y,y^{(a)},y^{(b)})\}$ \hfill \algc{store training sample}
\ENDWHILE
\STATE \textbf{return} $\mathcal{T}_{\text{cross}}$
\end{algorithmic}
\end{algorithm}

\section{Hyperparameter Configuration}
This section describes the hyperparameter settings of the SetTransformer-based data encoder and the multimodal BERT editor used in GESR. All hyperparameters are fixed across experiments to ensure fair comparison and reproducibility. The configurations are selected to balance model capacity, training stability, and computational efficiency.

\subsection{Training Configuration}

To ensure the reproducibility of the pretraining process, we summarize the training configuration of the multimodal BERT editor in Table~\ref{tab:training_setup}. The model was trained on a large-scale corpus of symbolic-regression expressions and corresponding sampled data pairs. Specifically, the training set contains 10 million symbolic expressions, and the optimization process was conducted on 8 NVIDIA A100 GPUs with 80 GB memory per GPU. The model was trained for 10 epochs, requiring approximately 56 hours in total. In addition, the storage size of each training checkpoint is around 0.5 GB. These settings provide a practical reference for estimating the computational cost and storage requirements of the proposed framework.

\begin{table}[h]
\centering
\small
\begin{tabular}{lll}
\toprule
\textbf{Item} & \textbf{Value} & \textbf{Description} \\
\midrule
Training data scale & 20M & Number of training expressions \\
GPU & 8 $\times$ A100 & Number and type of GPUs \\
Memory & 80 GB / GPU & GPU memory capacity \\
Epochs & 10 & Number of training epochs \\
Training time & 56 hours & Total pretraining time \\
Checkpoint size & 0.5 GB & Storage size per checkpoint \\
\bottomrule
\end{tabular}
\caption{Training configuration of the train.}
\label{tab:training_setup}
\end{table}

\subsection{Hyperparameter Configuration for SetTransformer}
Table~\ref{set-tab} reports the hyperparameters of the SetTransformer, which encodes dataset-level information in a permutation-invariant manner. The model employs a hidden dimension of 512 with 5 encoder layers and 8 attention heads, providing sufficient expressive power to capture complex data distributions. We set the number of inducing points to 50 and the output embedding length to 60, ensuring a fixed-size representation regardless of the number of input samples. GELU is adopted as the activation function, and layer normalization is enabled to improve optimization stability. The model is trained with a learning rate of $1\times10^{-4}$ using fp16 precision to reduce memory consumption and accelerate training. Unless otherwise specified, dropout and input normalization are disabled, as they did not yield consistent performance improvements in preliminary experiments.

\begin{table*}[t]
\caption{Hyperparameters of SetTransformer}
\label{set-tab}
\centering
\small
\setlength{\tabcolsep}{4pt}
\renewcommand{\arraystretch}{0.9}
\begin{tabularx}{\columnwidth}{l l X}
\toprule
\textbf{Hyperparameter} & \textbf{Value} & \textbf{Description} \\
\midrule
$N_p$ & 0 & Number of pooling layers \\
activation & gelu & Activation function \\
bit16 & True & Enable fp16 mixed-precision training \\
dec\_layers & 5 & Number of decoder layers \\
dec\_pf\_dim & 512 & Feedforward hidden dimension \\
dim\_hidden & 512 & Hidden embedding size \\
dim\_input & 3 & Input feature dimension \\
dropout & 0 & Dropout rate \\
input\_normalization & False & Disable input normalization \\
length\_eq & 60 & Output embedding length \\
linear & False & Disable linear projection head \\
ln & True & Enable layer normalization \\
lr & 0.0001 & Learning rate \\
mean & 0.5 & Initialization mean \\
n\_l\_enc & 5 & Number of encoder layers \\
norm & True & Apply feature normalization \\
num\_features & 20 & Number of extracted features \\
num\_heads & 8 & Multi-head attention heads \\
num\_inds & 50 & Number of inducing points \\
output\_dim & 60 & Output embedding dimension \\
sinusoidal\_embeddings & False & Disable sinusoidal positional encoding \\
src\_pad\_idx & 0 & Source padding index \\
std & 0.5 & Initialization standard deviation \\
trg\_pad\_idx & 0 & Target padding index \\
\bottomrule
\end{tabularx}
\end{table*}

\subsection{Hyperparameter Configuration for BERT}
The hyperparameter configuration of the multimodal BERT editor is summarized in Table~\ref{bert-tab}. The model follows a Transformer encoder architecture with 6 layers, a model dimension of 512, and 8 attention heads, along with a feedforward dimension of 2048. The maximum symbolic sequence length is set to 60, matching the SetTransformer output dimension to facilitate multimodal fusion. The model is trained using a masked language modeling objective with a masking ratio of 0.15 and a random token- or span-level masking strategy to enhance robustness. 

To incorporate dataset information, we introduce 16 learnable data tokens whose embeddings are projected to 512 dimensions and fused with symbolic token embeddings via concatenation followed by self-attention. Training is performed using the AdamW optimizer with a learning rate of $1\times10^{-4}$, weight decay of $1\times10^{-2}$, and a warmup ratio of 0.05. Dropout and attention dropout are set to 0.1 to mitigate overfitting, and gradient clipping with a maximum norm of 1.0 is applied to improve training stability. All experiments are conducted with a batch size of 256 under fp16 precision.

\begin{table}[t]
\caption{Hyperparameters of Multimodal BERT Editor}
\label{bert-tab}
\centering
\small
\setlength{\tabcolsep}{4pt}
\renewcommand{\arraystretch}{0.9}
\begin{tabularx}{\columnwidth}{l l X}
\toprule
\textbf{Hyperparameter} & \textbf{Value} & \textbf{Description} \\
\midrule
max\_seq\_len ($L$) & 60 & Maximum symbolic sequence length \\
emb\_dim ($d_{\text{model}}$) & 512 & Transformer hidden dimension \\
num\_layers ($L_{\text{bert}}$) & 6 & Number of Transformer encoder layers \\
num\_heads & 8 & Multi-head self-attention heads \\
ffn\_dim ($d_{\text{ff}}$) & 2048 & Feedforward network width \\
dropout & 0.1 & Dropout probability \\
attention\_dropout & 0.1 & Attention dropout rate \\
activation & gelu & Nonlinear activation function \\
layernorm & True & Enable layer normalization \\
weight\_tying & False & Disable embedding weight sharing \\
mask\_ratio & 0.15 & Fraction of masked tokens during training \\
mask\_strategy & random & Token-level or span-level masking \\
num\_data\_tokens & 16 & Number of dataset embedding tokens \\
data\_proj\_dim & 512 & Data embedding projection size \\
fusion\_type & concat + self-attention & Multimodal fusion mechanism \\
loss & masked language modeling & Cross-entropy objective \\
label\_smoothing & 0.0 & No label smoothing applied \\
optimizer & AdamW & Optimization algorithm \\
lr & 1e-4 & Learning rate \\
weight\_decay & 1e-2 & Weight decay coefficient \\
warmup\_ratio & 0.05 & Learning rate warmup fraction \\
grad\_clip\_norm & 1.0 & Gradient clipping threshold \\
batch\_size & 256 & Training batch size \\
precision & fp16 & Mixed-precision training mode \\
\bottomrule
\end{tabularx}
\end{table}

\section{Theoretical Proof of BERT-Guided Mutation Efficiency in GP for Symbolic Regression}

In this appendix, we present a formal proof that the mutation operator guided by a BERT model achieves higher search efficiency in genetic programming (GP) for symbolic regression than the standard unguided mutation. We first define the notion of mutation efficiency and then provide lemmas and a theorem to quantify the improvement. Throughout this analysis, we consider a GP-based symbolic regression algorithm as described by Koza \cite{Koza1992}, where candidate mathematical expressions evolve via genetic operators (crossover and mutation) to fit a target dataset.

\subsection{Definitions and Preliminaries}

Before analyzing the BERT-guided mutation, we formalize what it means for a mutation operator to be \emph{efficient} in the context of GP. For clarity, we assume that a larger fitness value indicates a better (more accurate) symbolic regression model (without loss of generality, one could equivalently minimize an error metric).

\begin{definition}[Mutation Success Probability]\label{def:success}
Consider an individual program (expression) $X$ in the GP population with fitness $F(X)$. Let $\mathcal{M}$ be a mutation operator that produces an offspring $X' = \mathcal{M}(X)$. We define the \emph{success probability} of $\mathcal{M}$ (with respect to $X$) as 
\[
p_{\text{succ}}(\mathcal{M}\mid X) \;=\; \Pr\!\big(F(X') > F(X)\big)\,,
\] 
the probability that applying $\mathcal{M}$ to $X$ yields an offspring with strictly improved fitness. When the context is clear, we write $p_{\text{succ}}(\mathcal{M})$ for this probability, assuming it is averaged over the distribution of individuals in the population.
\end{definition}

In standard GP, the mutation operator is typically \emph{uniform random subtree mutation} \cite{Poli1998}. This operator selects a random subtree of $X$ and replaces it with a new randomly generated subtree (chosen uniformly from the space of allowed sub-expressions). We denote this traditional mutation operator as $\mathcal{M}_{\text{rand}}$. Let 
\[
\alpha := p_{\text{succ}}(\mathcal{M}_{\text{rand}})\,,
\] 
i.e., $\alpha$ is the probability that a random mutation produces a fitter individual. Previous studies have observed that $\alpha$ is often relatively small, since a random change has a low chance of immediately improving fitness \cite{Poli1998}.

By contrast, in the \emph{BERT-guided mutation} approach, a pretrained language model (BERT) \cite{Devlin2019} is used to bias the mutation. Specifically, instead of a purely random replacement, the BERT model predicts which replacement subtrees are likely to improve the fitness of $X$ (based on patterns it learned from data or prior runs). We denote the BERT-guided mutation operator as $\mathcal{M}_{\text{BERT}}$. Let 
\[
\beta := p_{\text{succ}}(\mathcal{M}_{\text{BERT}})\,,
\] 
the probability that a BERT-guided mutation yields an improved individual. It is reasonable to assume $\beta > \alpha$, meaning the informed operator is more likely to find a beneficial tweak than random guessing. This assumption is supported by the design of BERT, which can capture contextual patterns and thus propose more promising mutations than an uninformed random choice.

\begin{definition}[Mutation Efficiency]\label{def:efficiency}
Given a mutation operator $\mathcal{M}$, we define its \emph{efficiency} $E(\mathcal{M})$ as the expected number of successful (fitness-improving) mutations it produces in one generation of the GP. If the population has $\lambda$ individuals that undergo mutation in a generation (either through a $\mu+\lambda$ evolutionary strategy or as offspring in a new generation of GP), then 
\[
E(\mathcal{M}) \;=\; \lambda \cdot p_{\text{succ}}(\mathcal{M})\,,
\] 
the expected count of improvements per generation due to $\mathcal{M}$. Higher $E(\mathcal{M})$ indicates a more efficient search, since more improvements are found on average in each generation.
\end{definition}

Under Definitions \ref{def:success} and \ref{def:efficiency}, if $\beta > \alpha$, we immediately have $E(\mathcal{M}_{\text{BERT}}) > E(\mathcal{M}_{\text{rand}})$ for the same $\lambda$, i.e. a BERT-guided mutation produces more improvements per generation in expectation than a random mutation. We will now formalize the advantage of BERT-guided mutation in terms of success probabilities and search time.

\subsection{Comparative Analysis and Main Results}

We next establish a lemma that quantifies the probability of improvement in at least one offspring when using multiple mutation trials in parallel (as in a population). This will be used to compare the one-generation improvement likelihood between the two operators.

\begin{lemma}\label{lem:onegen}
Consider a set of $\lambda$ independent mutation trials applied to copies of an individual (or to $\lambda$ individuals in a population) in one generation. Let $P_{\text{imp}}^{\text{rand}}$ be the probability that at least one of the $\lambda$ offspring produced by $\mathcal{M}_{\text{rand}}$ has higher fitness than its parent, and define $P_{\text{imp}}^{\text{BERT}}$ analogously for $\mathcal{M}_{\text{BERT}}$. Then 
\[
P_{\text{imp}}^{\text{BERT}} \;>\; P_{\text{imp}}^{\text{rand}}\,,
\] 
provided $\beta > \alpha$. In fact, 
\[
P_{\text{imp}}^{\text{rand}} = 1 - (1-\alpha)^{\lambda}\,, \qquad 
P_{\text{imp}}^{\text{BERT}} = 1 - (1-\beta)^{\lambda}\,,
\] 
and $1 - (1-\beta)^{\lambda} > 1 - (1-\alpha)^{\lambda}$ for $0 < \alpha < \beta < 1$.
\end{lemma}

\begin{proof}
Each offspring (mutation trial) has an independent probability $p_{\text{succ}}$ of being an improvement. Thus, the probability that \emph{no} offspring is an improvement is $(1 - p_{\text{succ}})^{\lambda}$, assuming independence. For the random mutation operator, this gives $\Pr(\text{no improvement in $\lambda$ trials}) = (1-\alpha)^{\lambda}$. Therefore the probability of at least one improvement is $P_{\text{imp}}^{\text{rand}} = 1 - (1-\alpha)^{\lambda}$. Similarly, $P_{\text{imp}}^{\text{BERT}} = 1 - (1-\beta)^{\lambda}$. Since $0<\alpha<\beta<1$, we have $(1-\alpha)^{\lambda} > (1-\beta)^{\lambda}$ (because $1-\alpha > 1-\beta$ and raising to a positive power preserves the inequality for positive bases). Subtracting these from 1 yields $P_{\text{imp}}^{\text{BERT}} > P_{\text{imp}}^{\text{rand}}$. 
\end{proof}

Lemma \ref{lem:onegen} confirms that in any single generation, the chance of finding an improved individual is higher with BERT-guided mutation. We now extend the analysis to the entire evolutionary process. In GP for symbolic regression, the search proceeds over many generations, and we are ultimately interested in how quickly the algorithm can discover a correct or optimal expression (for example, the true symbolic formula that perfectly fits the data). We will measure the search time in terms of generations or mutation operations needed to reach a solution of a given quality.

For theoretical analysis, assume that solving the regression problem requires a sequence of $k$ specific improvements (mutations) starting from a given initial individual to reach the global optimum. This is a simplistic abstraction: in practice improvements can be of varying magnitudes, and $k$ would relate to how many successful mutation steps are needed to evolve from a random initial guess to the target solution. The following theorem uses this assumption to compare the expected search time under the two mutation strategies.

\begin{theorem}\label{thm:efficiency}
Suppose that to find the optimal solution to the symbolic regression problem, the GP process must accumulate $k$ successful mutations (fitness-improving steps) in sequence. If $\beta > \alpha > 0$ as defined above, then the expected number of mutations (or generations, under a one-mutation-per-generation strategy) required by the BERT-guided mutation operator is \textbf{strictly less} than that required by the standard random mutation. In particular, the expected number of mutations to succeed is 
\[
E[T_{\text{BERT}}] \;=\; \frac{k}{\beta}\,, \qquad 
E[T_{\text{rand}}] \;=\; \frac{k}{\alpha}\,,
\] 
so that $E[T_{\text{BERT}}] < E[T_{\text{rand}}]$. Equivalently, $\mathcal{M}_{\text{BERT}}$ achieves the target solution in about $\frac{\alpha}{\beta}$ times (a fraction of) the generations needed by $\mathcal{M}_{\text{rand}}$, reflecting its higher efficiency.
\end{theorem}

\begin{proof}
We model the occurrence of a successful mutation as a probabilistic event with success probability $\beta$ for $\mathcal{M}_{\text{BERT}}$ and $\alpha$ for $\mathcal{M}_{\text{rand}}$, per mutation attempt. Consider first a single sequence of mutation attempts (as might be applied to one individual repeatedly until it improves). The number of attempts required to obtain one success follows a geometric distribution. For $\mathcal{M}_{\text{rand}}$, the expected number of mutations to achieve one improvement is $1/\alpha$. For $\mathcal{M}_{\text{BERT}}$, the expected mutations for one improvement is $1/\beta$ (since $\beta > \alpha$, this quantity is smaller). 

If the solution requires $k$ improvements in series, we can treat these as $k$ independent phases, each ending in a success. By the linearity of expectation (summing the expected trials for each phase) \cite{Koza1992}, the total expected number of mutation operations needed is 
\[ 
E[T_{\text{rand}}] = \underbrace{\frac{1}{\alpha} + \frac{1}{\alpha} + \cdots + \frac{1}{\alpha}}_{k \text{ times}} = \frac{k}{\alpha}\,,
\] 
for the random mutation, and 
\[ 
E[T_{\text{BERT}}] = \frac{k}{\beta}\,
\] 
for the BERT-guided mutation. Since $\beta > \alpha$, we have $k/\beta < k/\alpha$. This means that, on average, the BERT-guided approach will find the solution using fewer mutations. In a generational GP setting, if we assume each generation yields at most one significant mutation improvement (a common assumption for simplification, although in practice multiple improvements can occur in parallel as Lemma \ref{lem:onegen} suggests), then $E[T_{\text{BERT}}]$ and $E[T_{\text{rand}}]$ also represent the expected number of generations to solve the problem with each strategy. Hence, the BERT-guided operator achieves a faster convergence in expectation.

To illustrate numerically: if, for example, $\beta = 2\alpha$ (i.e., the BERT guidance doubles the chance of a beneficial mutation compared to random), then $E[T_{\text{BERT}}] = k/\beta = k/(2\alpha) = \frac{1}{2} \cdot (k/\alpha) = \frac{1}{2}E[T_{\text{rand}}]$. In general, the speedup factor is $\frac{E[T_{\text{rand}}]}{E[T_{\text{BERT}}]} = \frac{\beta}{\alpha} > 1$. This quantitatively confirms the efficiency gain due to BERT-guided mutation.
\end{proof}

Discussion: The above theorem relies on a simplified model of the evolutionary process, yet it captures the essential benefit of BERT-guided mutation: by increasing the probability of beneficial mutations ($\alpha \to \beta$), the search requires fewer steps on average to reach high-fitness solutions. The assumption $\beta > \alpha$ is critical and stems from the ability of the BERT model to leverage prior knowledge (learned from data or earlier generations) to propose more promising modifications. Empirical results in the literature support this assumption; for instance, large language models have been shown to suggest intelligent edits in program evolution tasks, thereby accelerating convergence \cite{Devlin2019}. While the exact value of $\beta$ will depend on the training and accuracy of the BERT model in the GP context, our theoretical analysis demonstrates qualitatively that any improvement $\beta-\alpha>0$ translates into a faster evolutionary search in expectation. This justifies the use of BERT to guide mutation in symbolic regression and explains the improved efficiency observed in experiments.
\begin{table*}[t]
\renewcommand{\arraystretch}{1.1}
\centering
\resizebox{13.2cm}{!}{
\def\arraystretch{1.1}
\small
\bgroup
\setlength{\tabcolsep}{0.4em}
\begin{tabular}{c|l|cccccccccc}
\toprule
\multirow{2}{*}{\bf Group} & \multicolumn{1}{c|}{\multirow{2}{*}{\bf Dataset}} &
\multicolumn{2}{c}{\bf GESR}  & \multicolumn{2}{c}{\bf TPSR} &
\multicolumn{2}{c}{\bf End2End} & \multicolumn{2}{c}{\bf NeSymRes[19]} &
\multicolumn{2}{c}{\bf uDSR} \\
\cmidrule(lr){3-12}
& & $R^2 \uparrow$ & Nodes $\downarrow$
& $R^2 \uparrow$ & Nodes $\downarrow$
& $R^2 \uparrow$ & Nodes $\downarrow$
& $R^2 \uparrow$ & Nodes $\downarrow$
& $R^2 \uparrow$ & Nodes $\downarrow$ \\
\cmidrule(lr){2-2}
\cmidrule(lr){3-4}
\cmidrule(lr){5-6}
\cmidrule(lr){7-8}
\cmidrule(lr){9-10}
\cmidrule(lr){11-12}

\multirow{10}{*}{\rotatebox{90}{Standards}}
& Nguyen
& \bestcell{$\mathbf{0.9999}_{\pm0.002}$} & \bestcell{13.6}
& \second{$0.9948_{\pm0.002}$} & \second{16.0}
& \third{$0.8814_{\pm0.004}$} & \third{16.3}
& $0.8568_{\pm0.003}$ & 18.2
& $0.9925_{\pm0.005}$ & 22.1 \\

& Keijzer
& \bestcell{$\mathbf{0.9998}_{\pm0.005}$} & \bestcell{14.7}
& \second{$0.9828_{\pm0.003}$} & \third{20.6}
& \third{$0.8134_{\pm0.005}$} & \second{18.4}
& $0.7992_{\pm0.003}$ & 21.3
& $0.9764_{\pm0.004}$ & 20.7 \\

& Korns
& \bestcell{$\mathbf{0.9964}_{\pm0.004}$} & \bestcell{15.5}
& \second{$0.9325_{\pm0.004}$} & \second{22.9}
& \third{$0.8715_{\pm0.004}$} & \third{23.4}
& $0.8011_{\pm0.005}$ & 24.1
& $0.9382_{\pm0.005}$ & 29.3 \\

& Constant
& \bestcell{$\mathbf{0.9997}_{\pm0.003}$} & \bestcell{24.7}
& \second{$0.9319_{\pm0.002}$} & 35.3
& $0.8015_{\pm0.003}$ & \second{28.3}
& \third{$0.8344_{\pm0.003}$} & \third{32.9}
& $0.9524_{\pm0.004}$ & 38.5 \\

& Livermore
& \bestcell{$\mathbf{0.9962}_{\pm0.006}$} & \bestcell{31.5}
& \second{$0.8820_{\pm0.004}$} & 38.2
& \third{$0.7015_{\pm0.004}$} & \second{32.2}
& $0.6836_{\pm0.005}$ & \third{36.2}
& $0.9472_{\pm0.0005}$ & 42.1 \\

& Vladislavleva
& \bestcell{$\mathbf{0.9992}_{\pm0.005}$} & \bestcell{20.2}
& \second{$0.9028_{\pm0.005}$} & \third{24.6}
& \third{$0.7422_{\pm0.005}$} & \second{22.2}
& $0.6892_{\pm0.004}$ & 27.3
& $0.9502_{\pm0.005}$ & 38.3 \\

& R
& \bestcell{$\mathbf{0.9991}_{\pm0.006}$} & \bestcell{14.6}
& \second{$0.9422_{\pm0.003}$} & \second{16.2}
& \third{$0.8512_{\pm0.004}$} & \third{19.5}
& $0.7703_{\pm0.005}$ & 19.9
& $0.9306_{\pm0.007}$ & 24.9 \\

& Jin
& \bestcell{$\mathbf{0.9984}_{\pm0.005}$} & \bestcell{17.1}
& \second{$0.9826_{\pm0.004}$} & \second{29.5}
& \third{$0.8611_{\pm0.004}$} & \third{29.8}
& $0.8327_{\pm0.003}$ & 32.2
& $0.9532_{\pm0.006}$ & 34.4 \\

& Neat
& \bestcell{$\mathbf{0.9987}_{\pm0.006}$} & \bestcell{14.1}
& \second{$0.9319_{\pm0.002}$} & \second{16.4}
& \third{$0.8044_{\pm0.004}$} & \third{19.7}
& $0.7596_{\pm0.005}$ & 20.6
& $0.9379_{\pm0.005}$ & 26.8 \\

& Others
& \bestcell{$\mathbf{0.9944}_{\pm0.004}$} & 23.9
& \second{$0.9667_{\pm0.002}$} & \second{22.5}
& \third{$0.8415_{\pm0.003}$} & \bestcell{22.3}
& $0.8026_{\pm0.003}$ & \third{23.5}
& $0.9399_{\pm0.004}$ & 39.3 \\

\midrule
\multirow{3}{*}{\rotatebox{90}{SRBench}}
& Feynman
& \bestcell{$\mathbf{0.9946}_{\pm0.004}$} & \bestcell{17.2}
& \second{$0.8928_{\pm0.004}$} & \second{21.3}
& \third{$0.7353_{\pm0.004}$} & \third{22.0}
& $0.7025_{\pm0.005}$ & 22.4
& $0.9076_{\pm0.005}$ & 29.8 \\

& Strogatz
& \bestcell{$\mathbf{0.9876}_{\pm0.006}$} & \bestcell{21.7}
& \second{$0.8249_{\pm0.002}$} & \second{24.4}
& \third{$0.6626_{\pm0.003}$} & \third{25.4}
& $0.6022_{\pm0.003}$ & 28.1
& $0.8937_{\pm0.005}$ & 32.8 \\

& Black-box
& \bestcell{$\mathbf{0.9364}_{\pm0.003}$} & \bestcell{23.7}
& \second{$0.8753_{\pm0.004}$} & \second{29.3}
& \third{$0.6925_{\pm0.004}$} & \third{31.2}
& $0.6525_{\pm0.005}$ & 33.9
& $0.8282_{\pm0.005}$ & 36.3 \\

\midrule
\multirow{1}{*}{\rotatebox{90}{ }}
& Average
& \bestcell{$\mathbf{0.9903}$} & \bestcell{22.0}
& \second{$0.9264$} & \third{24.4}
& \third{$0.7892$} & \second{23.9}
& $0.7528$ & 26.2
& $0.9352$ & 31.9 \\
\bottomrule
\end{tabular}
\egroup
}
\caption{The results of performance comparison. At a 0.95 confidence level, a comparison of the coefficient of determination ($R^2$) and the expression complexity (Nodes) was conducted between \textcolor{orange}{GESR} and four baselines.}
\label{tab_6}
\end{table*}

\subsection{Theoretical Proof of the Efficiency of BERT-Guided Crossover}

In the previous section, we have shown that BERT-guided mutation can reduce the expected number of GP search steps by increasing the probability of beneficial mutations. In this section, we further show that BERT-guided crossover can improve search efficiency in a similar manner. The key idea is that, if the crossover-guiding BERT increases the probability of selecting beneficial subtree exchange points, then the expected waiting time for generating fitness-improving offspring is reduced.

\subsubsection{Definitions and Assumptions}

Let each individual in symbolic regression be represented as an expression tree, and let its fitness be denoted by
\[
F(E),
\]
where a larger value indicates better regression performance on the given dataset $(X,Y)$, such as a higher $R^2$.

In genetic programming, crossover takes two parent expressions $E^{(a)}$ and $E^{(b)}$, selects one subtree root node from each parent, and exchanges the corresponding subtrees to generate two offspring:
\[
(\hat{E}^{(a)},\hat{E}^{(b)}) = C(E^{(a)},E^{(b)}).
\]

We define a crossover operation as a \textbf{beneficial crossover} if at least one offspring achieves higher fitness than the best parent:
\[
\max\bigl(F(\hat{E}^{(a)}),F(\hat{E}^{(b)})\bigr)
>
\max\bigl(F(E^{(a)}),F(E^{(b)})\bigr).
\]

\begin{definition}[Success probability of random crossover]
For the random crossover operator $C_{\mathrm{rand}}$ in conventional GP, we define its probability of producing a beneficial crossover as
\[
\alpha_c
=
\Pr\left[
\max\bigl(F(\hat{E}^{(a)}),F(\hat{E}^{(b)})\bigr)
>
\max\bigl(F(E^{(a)}),F(E^{(b)})\bigr)
\right].
\]
\end{definition}

\begin{definition}[Success probability of BERT-guided crossover]
For the BERT-guided crossover operator $C_{\mathrm{BERT}}$, the input consists of the joint symbolic sequence
\[
s^{(ab)}=\bigl(s^{(a)},\texttt{[SEP]},s^{(b)}\bigr)
\]
and the dataset $(X,Y)$. The crossover-guiding BERT outputs two conditional distributions over candidate crossover positions:
\[
p(i\mid s^{(ab)},X,Y), \qquad p(j\mid s^{(ab)},X,Y),
\]
where $i$ and $j$ denote the subtree root positions in the two parent expressions. The model selects the crossover points according to these distributions and performs subtree exchange. We define the probability that this operation produces a beneficial crossover as
\[
\beta_c
=
\Pr\left[
\max\bigl(F(\hat{E}^{(a)}_{\mathrm{BERT}}),F(\hat{E}^{(b)}_{\mathrm{BERT}})\bigr)
>
\max\bigl(F(E^{(a)}),F(E^{(b)})\bigr)
\right].
\]
\end{definition}

Since the crossover-guiding BERT is trained on crossover events that lead to fitness improvement, its objective is to assign higher probabilities to subtree exchange positions that are more likely to improve regression fitness. Therefore, under the assumption that the model learns a useful crossover policy, we have
\[
\beta_c > \alpha_c.
\]
This means that BERT-guided crossover is more likely to generate fitness-improving offspring than random crossover.

\subsubsection{Lemma 1: BERT-Guided Crossover Reduces the Expected Waiting Time for a Beneficial Crossover}

\begin{lemma}
If the success probability of random crossover is $\alpha_c$, the success probability of BERT-guided crossover is $\beta_c$, and
\[
\beta_c>\alpha_c,
\]
then the expected number of crossover attempts required to obtain one beneficial crossover satisfies
\[
\mathbb{E}[T^{c}_{\mathrm{BERT}}]
<
\mathbb{E}[T^{c}_{\mathrm{rand}}].
\]
\end{lemma}

\begin{proof}
Each crossover attempt can be viewed as a Bernoulli trial, where success corresponds to producing a beneficial crossover. For random crossover, the success probability is $\alpha_c$. Therefore, the number of trials required to obtain the first success follows a geometric distribution, and its expectation is
\[
\mathbb{E}[T^{c}_{\mathrm{rand}}]=\frac{1}{\alpha_c}.
\]
Similarly, for BERT-guided crossover, the success probability is $\beta_c$, and the expected number of trials is
\[
\mathbb{E}[T^{c}_{\mathrm{BERT}}]=\frac{1}{\beta_c}.
\]
Since $\beta_c>\alpha_c$, we have
\[
\frac{1}{\beta_c}<\frac{1}{\alpha_c}.
\]
Thus,
\[
\mathbb{E}[T^{c}_{\mathrm{BERT}}]
<
\mathbb{E}[T^{c}_{\mathrm{rand}}].
\]
\end{proof}

\subsubsection{Lemma 2: BERT-Guided Crossover Increases the Population-Level Improvement Probability}

\begin{lemma}
Suppose that $m$ crossover operations are performed in one generation. If the success probabilities of random crossover and BERT-guided crossover are $\alpha_c$ and $\beta_c$, respectively, with $\beta_c>\alpha_c$, then the probability that at least one beneficial crossover occurs in one generation is higher under BERT-guided crossover.
\end{lemma}

\begin{proof}
For random crossover, the probability that none of the $m$ crossover operations produces a beneficial offspring is
\[
(1-\alpha_c)^m.
\]
Therefore, the probability that at least one beneficial crossover occurs is
\[
P^{c}_{\mathrm{rand}}
=
1-(1-\alpha_c)^m.
\]

Similarly, for BERT-guided crossover, the probability that at least one beneficial crossover occurs in one generation is
\[
P^{c}_{\mathrm{BERT}}
=
1-(1-\beta_c)^m.
\]

Since $\beta_c>\alpha_c$, we have
\[
1-\beta_c < 1-\alpha_c.
\]
Therefore,
\[
(1-\beta_c)^m < (1-\alpha_c)^m,
\]
which implies
\[
1-(1-\beta_c)^m
>
1-(1-\alpha_c)^m.
\]
Hence,
\[
P^{c}_{\mathrm{BERT}}>P^{c}_{\mathrm{rand}}.
\]
\end{proof}

\subsubsection{Theorem: BERT-Guided Crossover Improves GP Search Efficiency}

\begin{theorem}
Assume that the symbolic regression search process requires $k_c$ key structural recombination steps to reach a target high-fitness region. If the success probability of random crossover is $\alpha_c$, the success probability of BERT-guided crossover is $\beta_c$, and
\[
\beta_c>\alpha_c,
\]
then the expected number of crossover operations required by BERT-guided crossover is smaller than that required by random crossover:
\[
\mathbb{E}[T^{c}_{\mathrm{BERT}}]
<
\mathbb{E}[T^{c}_{\mathrm{rand}}].
\]
More specifically,
\[
\mathbb{E}[T^{c}_{\mathrm{rand}}]=\frac{k_c}{\alpha_c},
\qquad
\mathbb{E}[T^{c}_{\mathrm{BERT}}]=\frac{k_c}{\beta_c}.
\]
\end{theorem}

\begin{proof}
We view the process of reaching the target high-fitness region as requiring $k_c$ key beneficial crossover events. For each key event, the expected number of attempts required by random crossover is
\[
\frac{1}{\alpha_c},
\]
while that required by BERT-guided crossover is
\[
\frac{1}{\beta_c}.
\]

By the linearity of expectation, the total expected numbers of crossover attempts are
\[
\mathbb{E}[T^{c}_{\mathrm{rand}}]
=
\underbrace{\frac{1}{\alpha_c}+\cdots+\frac{1}{\alpha_c}}_{k_c\ \text{times}}
=
\frac{k_c}{\alpha_c},
\]
and
\[
\mathbb{E}[T^{c}_{\mathrm{BERT}}]
=
\underbrace{\frac{1}{\beta_c}+\cdots+\frac{1}{\beta_c}}_{k_c\ \text{times}}
=
\frac{k_c}{\beta_c}.
\]

Since $\beta_c>\alpha_c$, it follows that
\[
\frac{k_c}{\beta_c}<\frac{k_c}{\alpha_c}.
\]
Therefore,
\[
\mathbb{E}[T^{c}_{\mathrm{BERT}}]
<
\mathbb{E}[T^{c}_{\mathrm{rand}}].
\]
\end{proof}

\subsubsection{Search Efficiency under Joint Mutation and Crossover Guidance}

In the complete GESR framework, the search process contains both BERT-guided mutation and BERT-guided crossover. Let the success probabilities of random mutation and BERT-guided mutation be
\[
\alpha_m,\qquad \beta_m,
\]
where
\[
\beta_m>\alpha_m.
\]
Similarly, let the success probabilities of random crossover and BERT-guided crossover be
\[
\alpha_c,\qquad \beta_c,
\]
where
\[
\beta_c>\alpha_c.
\]

If an effective evolutionary improvement can be produced by either a beneficial mutation or a beneficial crossover, and if these two types of events are approximately independent, then the probability of obtaining at least one effective improvement in one generation under conventional GP can be written as
\[
P_{\mathrm{rand}}
=
1-(1-\alpha_m)^{n_m}(1-\alpha_c)^{n_c},
\]
where $n_m$ and $n_c$ denote the numbers of mutation and crossover operations performed in one generation, respectively.

For GESR with both mutation guidance and crossover guidance, the corresponding probability is
\[
P_{\mathrm{GESR}}
=
1-(1-\beta_m)^{n_m}(1-\beta_c)^{n_c}.
\]

Since
\[
\beta_m>\alpha_m,\qquad \beta_c>\alpha_c,
\]
we have
\[
(1-\beta_m)^{n_m}(1-\beta_c)^{n_c}
<
(1-\alpha_m)^{n_m}(1-\alpha_c)^{n_c}.
\]
Therefore,
\[
P_{\mathrm{GESR}}>P_{\mathrm{rand}}.
\]

This indicates that, after introducing both mutation guidance and crossover guidance, the probability of generating an effective improvement in each generation becomes higher. If reaching the target solution requires $k$ effective improvements, the expected number of generations can be approximated as
\[
\mathbb{E}[T_{\mathrm{GESR}}]\approx \frac{k}{P_{\mathrm{GESR}}},
\qquad
\mathbb{E}[T_{\mathrm{rand}}]\approx \frac{k}{P_{\mathrm{rand}}}.
\]
Since
\[
P_{\mathrm{GESR}}>P_{\mathrm{rand}},
\]
we obtain
\[
\mathbb{E}[T_{\mathrm{GESR}}]
<
\mathbb{E}[T_{\mathrm{rand}}].
\]

Thus, BERT-guided crossover not only improves the efficiency of structural recombination on its own, but also works together with BERT-guided mutation to increase the probability of generating high-quality candidate expressions in each generation, thereby reducing the expected search time required to reach high-fitness symbolic expressions.

\subsubsection{Discussion}

The above analysis is based on a simplified probabilistic model, but it captures the core reason why BERT-guided crossover improves search efficiency. In conventional GP, crossover points are usually selected randomly. As a result, many crossover operations may destroy useful substructures and generate low-quality offspring. In contrast, the crossover-guiding BERT predicts subtree exchange positions based on both the symbolic structures of the two parent expressions and the data distribution of $(X,Y)$, making it more likely to select crossover points that improve fitness.

In other words, BERT-guided crossover transforms the nearly unbiased random subtree exchange in conventional GP into a data-conditioned and directed structural recombination process. As long as the crossover-guiding model induces a positive increase in the probability of beneficial crossover, i.e.,
\[
\beta_c-\alpha_c>0,
\]
this improvement directly leads to a reduction in the expected number of search steps. Therefore, from a theoretical perspective, introducing crossover guidance can further improve the search efficiency of GESR and, together with mutation guidance, explains the faster convergence and stronger search stability observed in experiments.

\section{Performance comparison between GESR and other types of Baselins}
In order to enrich the experimental results, we compared GESR with more other baselines, and the results are shown in Table \ref{tab_6}.
\begin{itemize}
    \item \textbf{uDSR\cite{landajuela2022unified}}: uDSR is a deep symbolic regression framework that unifies neural-guided search, reinforcement learning optimization, and symbolic expression generation to automatically discover high-precision and interpretable mathematical formulas from data.
    \item \textbf{NeSymReS}\cite{nesy}: Based on SymbolicGPT, the algorithm treated each `operator' (e.g., `sin', `cos') as a token. The sequence of expressions is generated in turn. The setup is more reasonable.
    \item \textbf{End2End}\cite{kamienny2022end}: End2End further encodes the constants so that the model generates the constants directly when generating the expression, abandoning the constant placeholder `C' used by the previous two algorithms.  
    \item \textbf{TPSR}\cite{shojaee2024transformer}: One uses a pre-trained model as a policy network to guide the MCTS process, which greatly improves the search efficiency of MCTS.
\end{itemize}

\section{Comparison Between Greedy and Probability Sampling Strategies in BERT-Guided Mutation}

To further analyze the effectiveness of the BERT-guided mutation module in GESR, we conducted a comparative study between a \textbf{greedy mutation generation strategy} and a \textbf{probability-distribution sampling strategy}. The main difference between the two strategies lies in the mutation decision process: the former directly selects the candidate with the highest probability predicted by BERT at each mutation step, whereas the latter samples candidates according to the predicted probability distribution, thereby preserving a certain degree of exploration.
\newcolumntype{P}[1]{>{\centering\arraybackslash}p{#1}}
\begin{wraptable}{r}{0.74\textwidth}
\centering
\fontsize{6.0pt}{7.5pt}\selectfont
\setlength{\tabcolsep}{3pt}
\begin{tabular}{P{1.4cm}P{1.1cm}P{1.1cm}P{1.1cm}P{1.1cm}P{1.1cm}P{1.1cm}}
\toprule[1.2pt]
\makecell{Dataset} & \multicolumn{3}{c}{\makecell{Greedy}} 
& \multicolumn{3}{c}{\makecell{Sampling}} \\  
\cmidrule[0.1pt]{1-7}
& $R^2 \uparrow$ & Nodes $\downarrow$ & \makecell{Time(s)$\downarrow$} 
& $R^2 \uparrow$ & Nodes $\downarrow$ & \makecell{Time(s)$\downarrow$} \\ 
\cmidrule(lr){1-1}
\cmidrule(lr){2-4}
\cmidrule(lr){5-7}

Nguyen         & \cellcolor{gray!25}0.9999 & \cellcolor{gray!25}14.5 & \cellcolor{gray!25}84  
               & \cellcolor{gray!25}0.9999 & 16.9 & 96 \\
Keijzer        & \cellcolor{gray!25}0.9997 & \cellcolor{gray!25}16.2 & \cellcolor{gray!25}133 
               & 0.9942 & 17.8 & 149 \\
Korns          & \cellcolor{gray!25}0.9938 & \cellcolor{gray!25}16.8 & \cellcolor{gray!25}123 
               & 0.9919 & 19.6 & 144 \\
Constant       & \cellcolor{gray!25}0.9995 & \cellcolor{gray!25}28.5 & \cellcolor{gray!25}108 
               & 0.9841 & 31.2 & 126 \\
Livermore      & 0.9903 & \cellcolor{gray!25}36.4 & \cellcolor{gray!25}145 
               & \cellcolor{gray!25}0.9937 & 33.1 & 186 \\
Vladislavleva  & \cellcolor{gray!25}0.9986 & \cellcolor{gray!25}24.8 & \cellcolor{gray!25}122 
               & 0.9901 & 23.5 & 173 \\
R              & \cellcolor{gray!25}0.9962 & \cellcolor{gray!25}15.7 & \cellcolor{gray!25}135 
               & 0.9840 & 18.3 & 161 \\
Jin            & \cellcolor{gray!25}0.9974 & \cellcolor{gray!25}20.2 & \cellcolor{gray!25}84  
               & 0.9846 & 21.1 & 97 \\
Neat           & 0.9983 & \cellcolor{gray!25}16.3 & \cellcolor{gray!25}187 
               & \cellcolor{gray!25}0.9995 & 17.3 & 191 \\
Others         & \cellcolor{gray!25}0.9938 & 26.9 & \cellcolor{gray!25}152 
               & 0.9801 & \cellcolor{gray!25}26.5 & 173 \\
Feynman        & \cellcolor{gray!25}0.9925 & \cellcolor{gray!25}20.4 & \cellcolor{gray!25}121 
               & 0.9825 & 22.4 & 146 \\
Strogatz       & 0.9813 & \cellcolor{gray!25}23.6 & \cellcolor{gray!25}99  
               & \cellcolor{gray!25}0.9833 & 25.3 & 131 \\
Black-box      & 0.9931 & \cellcolor{gray!25}25.4 & \cellcolor{gray!25}84  
               & \cellcolor{gray!25}0.9953 & 31.1 & 136 \\

\cmidrule(lr){1-1}
\cmidrule(lr){2-4}
\cmidrule(lr){5-7}

Average        & \cellcolor{gray!25}0.9903 & \cellcolor{gray!25}22.0 & \cellcolor{gray!25}124.4 
               & 0.9887 & 23.3 & 146.8 \\
\bottomrule
\end{tabular}
\caption{Performance comparison between greedy algorithm and probabilistic random sampling in GESR.}
\label{Greedy}
\vspace{-0.3cm}
\end{wraptable}
The experimental results\ref{Greedy} show that the greedy strategy is more efficient overall than the probability sampling strategy(124.4/146.8). Specifically, the greedy strategy can generate high-quality candidate expressions more quickly, reducing the extra computational overhead caused by ineffective search paths and thus improving the search efficiency of symbolic regression. In contrast, although the probability-distribution sampling strategy theoretically provides stronger diversity and exploration capability, it also introduces more randomness, which makes it more likely to generate mutations that deviate from the current optimal direction and consequently reduces search efficiency.

This result also indirectly demonstrates that BERT is relatively accurate in predicting both mutation positions and mutation contents. In other words, the high-probability candidates output by BERT are often already the better mutation directions, so directly adopting greedy selection can make fuller use of the model's guidance ability, without relying on additional random sampling to compensate for insufficient prediction quality. 

If BERT were less accurate, the greedy strategy would be more likely to fall into local optima, while the sampling strategy might achieve better results through random exploration. However, the fact that the greedy strategy performs better in our experiments indicates that BERT provides sufficiently reliable prior guidance for the mutation task.

In summary, the greedy strategy not only outperforms the probability sampling strategy in terms of efficiency but also further validates the effectiveness and accuracy of the BERT-guided mutation mechanism. Therefore, adopting the greedy strategy as the default mutation generation method in GESR is both reasonable and effective. This design choice ensures the stability of the search process while further improving the overall optimization efficiency of the algorithm.

\section{Transferability of the BERT-Guided Gene Editing Strategy}

Beyond its application within the GESR framework, we further investigate the transferability of the proposed \textbf{BERT-guided gene editing strategy} to other genetic-programming-based (GP-based) symbolic regression methods. It should be emphasized that our approach is not restricted to GESR itself; rather, the pretrained multimodal BERT models can serve as general symbolic editors that provide learning-based guidance in both the mutation and crossover stages, thereby improving the search quality and efficiency of other GP-style symbolic regression methods.

\begin{wraptable}{r}{0.76\textwidth}
\centering
\fontsize{6.5pt}{7.5pt}\selectfont
\setlength{\tabcolsep}{3pt}
\begin{tabular}{P{1.35cm}P{1.2cm}P{1.2cm}P{1.2cm}P{1.2cm}P{1.2cm}P{1.2cm}}
\toprule[1.2pt]
\makecell{Dataset} & \multicolumn{2}{c}{\makecell{NGGP(Mutat+Cross)}} 
& \multicolumn{2}{c}{\makecell{NGGP(Mutation)}} 
& \multicolumn{2}{c}{\makecell{NGGP(No)}} \\  
\cmidrule[0.1pt]{1-7}
& $R^2 \uparrow$ & \makecell{Time(s)$\downarrow$} 
& $R^2 \uparrow$ & \makecell{Time(s)$\downarrow$}
& $R^2 \uparrow$ & \makecell{Time(s)$\downarrow$} \\
\cmidrule(lr){1-1}
\cmidrule(lr){2-3}
\cmidrule(lr){4-5}
\cmidrule(lr){6-7}

Nguyen         & \cellcolor{gray!25}0.9999 & \cellcolor{gray!25}74  
               & 0.9999 & 80  
               & 0.9999 & 106 \\
Keijzer        & \cellcolor{gray!25}0.9962 & \cellcolor{gray!25}133 
               & 0.9940 & 142 
               & 0.9924 & 196 \\
Korns          & \cellcolor{gray!25}0.9921 & \cellcolor{gray!25}116 
               & 0.9898 & 123 
               & 0.9872 & 168 \\
Constant       & \cellcolor{gray!25}0.9993 & \cellcolor{gray!25}91 
               & 0.9990 & 98 
               & 0.9988 & 132 \\
Livermore      & \cellcolor{gray!25}0.9868 & \cellcolor{gray!25}165 
               & 0.9812 & 178 
               & 0.9746 & 247 \\
Vladislavleva  & \cellcolor{gray!25}0.9971 & \cellcolor{gray!25}137 
               & 0.9958 & 146 
               & 0.9963 & 199 \\
R              & \cellcolor{gray!25}0.9896 & \cellcolor{gray!25}142 
               & 0.9855 & 151 
               & 0.9744 & 208 \\
Jin            & \cellcolor{gray!25}0.9952 & \cellcolor{gray!25}79  
               & 0.9935 & 84  
               & 0.9916 & 113 \\
Neat           & \cellcolor{gray!25}0.9923 & \cellcolor{gray!25}136 
               & 0.9896 & 145 
               & 0.9827 & 199 \\
Others         & \cellcolor{gray!25}0.9901 & \cellcolor{gray!25}141 
               & 0.9880 & 150 
               & 0.9861 & 208 \\
Feynman        & \cellcolor{gray!25}0.9879 & \cellcolor{gray!25}98  
               & 0.9843 & 105 
               & 0.9710 & 144 \\
Strogatz       & \cellcolor{gray!25}0.9815 & \cellcolor{gray!25}91  
               & 0.9768 & 98  
               & 0.9613 & 133 \\
Black-box      & \cellcolor{gray!25}0.9554 & \cellcolor{gray!25}80  
               & 0.9478 & 86  
               & 0.9033 & 122 \\

\cmidrule(lr){1-1}
\cmidrule(lr){2-3}
\cmidrule(lr){4-5}
\cmidrule(lr){6-7}

Average        & \cellcolor{gray!25}0.9895 & \cellcolor{gray!25}114.1 
               & 0.9866 & 122.0 
               & 0.9784 & 167.3 \\
\bottomrule
\end{tabular}
\end{wraptable}

From the implementation perspective, such a transfer is direct and natural. For any GP-based symbolic regression method, the pretrained BERT modules can be incorporated into its original evolutionary pipeline: in the mutation stage, the original random symbol perturbation is replaced by BERT-guided conditional symbol editing; in the crossover stage, the original random selection of subtree exchange points is replaced by BERT-predicted directed subtree crossover. Specifically, during mutation, some symbols in the preorder traversal sequence of an expression are randomly masked, and BERT is then used to predict the most suitable replacement symbols, thereby generating a new candidate expression. During crossover, a pair of candidate expressions together with the data $(X,Y)$ are fed into the crossover-guiding BERT, which predicts the root nodes of the subtrees that are most suitable for exchange in the two expressions, thus enabling directed structural recombination. Through this mechanism, the evolutionary process can exploit the symbolic priors and data-conditioned information learned by BERT during pretraining, thereby reducing ineffective search caused by low-quality random edits.

To verify the transferability of this strategy, we incorporated both the pretrained mutation-guiding BERT and crossover-guiding BERT into NGGP. NGGP can be viewed as a hybrid symbolic regression framework that combines DSR and GP. In this experiment, we kept the overall structure of NGGP unchanged, and only replaced the original random mutation and random crossover operations in its GP component, so that the performance changes before and after introducing learning-guided gene editing could be directly evaluated. Based on this setting, we compared three configurations: \textbf{NGGP(No)}, which introduces no BERT guidance; \textbf{NGGP(Mutation)}, which introduces only mutation guidance; and \textbf{NGGP(Mutat+Cross)}, which introduces both mutation guidance and crossover guidance.

The experimental results are reported in Table~\ref{Greedy}. Overall, after introducing mutation guidance, NGGP achieves clear improvements in both fitting performance and inference efficiency on most datasets, indicating that BERT-guided symbol-level editing can effectively alleviate the randomness in the original GP search process. On top of this, further incorporating crossover guidance leads to additional improvements in overall performance: the average $R^2$ increases further, while the average inference time is further reduced. Notably, although the crossover-guiding BERT introduces additional inference overhead, the overall search efficiency is still improved. This suggests that directed crossover substantially reduces the generation of low-quality offspring and improves the effectiveness of structural recombination.

These results demonstrate that the proposed BERT-guided gene editing strategy has strong transferability. Its effectiveness does not rely on the specific framework design of GESR; instead, it can be naturally integrated as a plug-and-play evolutionary enhancement module into other GP-based symbolic regression methods. Furthermore, this also indicates that the symbolic editing priors learned by BERT during pretraining are not confined to a single framework, but can remain effective under different genetic search paradigms, thereby improving the overall performance and efficiency of symbolic regression.

\section{Test in NED (normalized edit distance)}
\label{AG}
In order to better test the performance of MetaSymNet and other baselines, we also use normalized edit distances (NED) proposed in the article \cite{matsubara2022rethinking} to test each baseline. The NED expression is as follows Eq\ref{e_NED}.
\begin{equation}
\label{e_NED}
NED(f_{pred},f_{true})=min(1,\frac{NED(f_{pred},f_{true})}{|f_{true}|})
\end{equation}
$f_{pred}$ and $f_{true}$ are the estimated and true equation trees, respectively. $NED(f_{pred}, ftrue)$ is the edit distance between $f_{pred}$ and $f_{true}$. $|f_{true}|$ denotes the number of tree nodes composing equation $f_{true}$. We note that this metric is meant to capture the symbolic similarity between the predicted and true equations, so the constant values themselves are not important; we only consider operators and not all constants (or constant placeholders).
The data set used for testing is the SRSD data set. The result is as follows table \ref{table:baseline_results}, and  table\ref{table:baseline_results_w_dummy_vars}:
\begin{table*}[t]
    
    \def\arraystretch{1.2}
    \small
    \begin{center}
        \bgroup
        \setlength{\tabcolsep}{0.4em}
        \begin{tabular}{c|l|rrrrrr} 
            \toprule
            \multirow{2}{*}{\bf Metric} & \multicolumn{1}{c|}{\multirow{2}{*}{\bf Group}} & \multicolumn{6}{c}{\bf SRSD-Feynman} \\
            \cmidrule(lr){3-8}
            & & \multicolumn{1}{c}{\bf GESR} & \multicolumn{1}{c}{\bf PYSR} & \multicolumn{1}{c}{\bf NGGP} & \multicolumn{1}{c}{\bf SPL} & \multicolumn{1}{c}{\bf SNIP} & \multicolumn{1}{c}{\bf GPlearn}  \\
            \midrule
            \multirow{3}{*}{\rotatebox{90}{NED}} & Easy & 0.496&	\textbf{0.448}&	0.630&	0.793 &  0.825 & 0.782  \\ 
            & Medium & \textbf{0.545}&	0.691&	0.792&	0.848 & 0.891 & 0.823  \\ 
            & Hard & \textbf{0.602}&	0.744&	0.915&0.952 & 0.979 & 0.872 \\
            \bottomrule
        \end{tabular}
        \egroup
        \caption{Baseline results for NED (normalized edit distance).}
    \label{table:baseline_results}
    \end{center}    
    
    \def\arraystretch{1.2}
    \small
    \begin{center}
        \bgroup
        \setlength{\tabcolsep}{0.4em}
        \begin{tabular}{c|l|rrrrrr} 
            \toprule
            \multirow{2}{*}{\bf Metric} & \multicolumn{1}{c|}{\multirow{2}{*}{\bf Group}} & \multicolumn{6}{c}{\bf SRSD-Feynman} \\
            \cmidrule(lr){3-8}
            & & \multicolumn{1}{c}{\bf GESR} & \multicolumn{1}{c}{\bf PYSR} & \multicolumn{1}{c}{\bf NGGP} & \multicolumn{1}{c}{\bf SPL} & \multicolumn{1}{c}{\bf SNIP} & \multicolumn{1}{c}{\bf GPlearn}  \\
            \midrule
            \multirow{3}{*}{\rotatebox{90}{NED}} & Easy & \textbf{0.653}&	0.681&	0.863&	0.906& 0.951 & 0.899  \\ 
            & Medium & \textbf{0.624}&	0.751&	0.962&	0.978& 0.981 & 0.863  \\ 
            & Hard & \textbf{0.644}&	0.746&	1.000&	1.000
 & 1.000 & 0.977 \\
            \bottomrule
        \end{tabular}
        \egroup
        \caption{Baseline results for NED (normalized edit distance).}
        \label{table:baseline_results_w_dummy_vars}
    \end{center}
\end{table*}

\section{Noise Robustness Test}
\begin{wrapfigure}{r}{0.48\textwidth}
    \centering
    \vspace{-0.2cm}
    \includegraphics[width=0.48\textwidth]{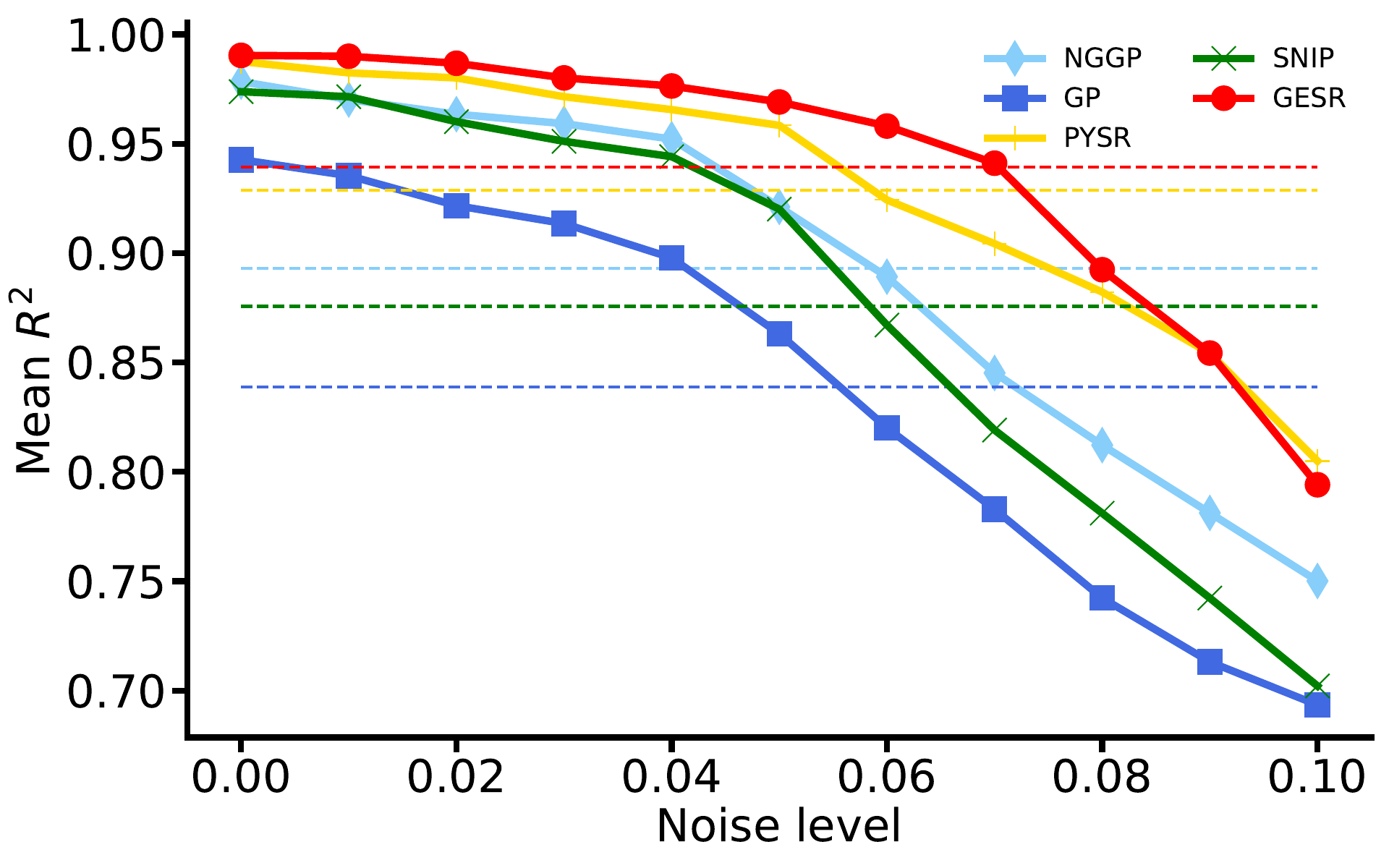}
    \vspace{-0.2cm}
    \caption{Noise robustness tests for GESR and four Baselines.}
    \label{fig_noise}
    \vspace{-0.3cm}
\end{wrapfigure}
In real-world scenarios, data often contains noise and uncertainty. Such noise may arise from measurement errors, environmental interference during data acquisition, or other unavoidable random factors \citep{no1,no2,no3}. Therefore, noise robustness testing is important for evaluating the performance of symbolic regression algorithms under imperfect data conditions, as it helps us better understand the reliability and stability of a model in practical applications. Moreover, such experiments can provide useful guidance for model selection and parameter tuning, thereby improving the applicability of the algorithm in real-world settings \citep{ano1,ano2}.

In this work, the noisy data $y_{noise}$ is generated as follows. To simulate different levels of noise in real-world environments, we divide the noise into ten levels. Specifically, noise data $Data_{noise}$ is randomly sampled from the interval $[-\mathcal{L} * Span,\,+\mathcal{L} * Span]$, where $\mathcal{L}=[0.00,0.01,0.02,\dots,0.1]$ denotes the noise level, and $Span = |\max(y)-\min(y)|$ represents the value range of the target variable $y$. The noisy target is then constructed as $y_{noise} = y + Data_{noise}$.

We use the relevant datasets to evaluate the noise robustness of the algorithm. For each mathematical expression, we repeat the experiment 10 times under different noise levels, and compute the $R^2$ between the fitted curve obtained in each trial and the original curve as the metric for quantifying noise robustness. The final result is obtained by averaging over the 10 trials.

We compare the average $R^2$ of GESR with those of four symbolic regression baselines under different noise levels, and the results are shown in Fig.~\ref{fig_noise}. It can be observed that GESR maintains strong fitting performance while also exhibiting good noise robustness, indicating that it can still achieve stable performance under moderate noise perturbations. However, when the noise level reaches 0.08, the performance of GESR begins to degrade more rapidly. We speculate that this phenomenon may be related to the relatively limited noise tolerance of the pretrained BERT model: when the input data noise exceeds the distribution range covered during pretraining, the symbolic editing priors provided by BERT may be more strongly disrupted, which in turn affects the overall search performance. Nevertheless, under most low- to medium-noise settings, GESR still demonstrates relatively strong robustness.

\section{Comparative Analysis of Symbolic Regression Performance Across Different Dimensions}

Symbolic regression is essentially a high-complexity combinatorial optimization problem. As the dimensionality of the input variables increases, the construction space of candidate expressions expands rapidly. With each additional variable $x_n$, the algorithm must not only search over a larger space of operator combinations, but also simultaneously handle multiple coupled processes such as variable selection and expression structure generation. As a result, the search space typically grows exponentially. This implies that symbolic regression algorithms face substantially greater challenges in high-dimensional settings: on the one hand, search efficiency decreases significantly; on the other hand, the algorithm becomes more prone to falling into local optima, thereby degrading the quality of the final expression.

To evaluate the performance of the proposed method under different levels of problem complexity, we conducted experiments on one-dimensional, two-dimensional, and three-dimensional datasets, comparing standard genetic programming (GP) with GESR, which incorporates BERT-based guidance. Figure~\ref{fig:dim_compare} illustrates the trends of fitting accuracy ($R^2$) and inference time for the two methods under different dimensional settings.

From the perspective of the $R^2$ metric, as the dimensionality increases from one to three, the fitting performance of both methods declines, which is consistent with the increased complexity of symbolic regression problems. However, the degradation of GESR is markedly slower. Specifically, in the one-dimensional task, both methods achieve $R^2$ values close to 1, indicating that they can both fit the data well in simple low-dimensional scenarios. However, when the dimensionality increases to two and three, the $R^2$ of standard GP drops much more noticeably, whereas GESR still maintains relatively high fitting accuracy. This suggests that in a more complex expression search space, GESR can guide the search process more stably, reduce ineffective exploration, and thereby preserve better solution quality.

From the perspective of inference time, as the dimensionality increases, the search cost of both algorithms rises significantly, but the growth rate of GESR is clearly lower than that of GP. Especially in the three-dimensional setting, the inference time of GP increases much more sharply, while GESR still maintains a relatively controlled growth trend. This indicates that as the problem dimensionality increases, standard GP must expend substantially more effort searching through the enlarged expression space, whereas GESR, aided by external prior guidance, can focus more quickly on more promising candidate regions, thereby significantly improving search efficiency. In other words, the higher the dimensionality and the larger the search space, the more pronounced the benefit brought by guidance information becomes.
\begin{wrapfigure}{r}{0.76\textwidth}
    \centering
    \vspace{-0.2cm}
    \includegraphics[width=0.72\textwidth]{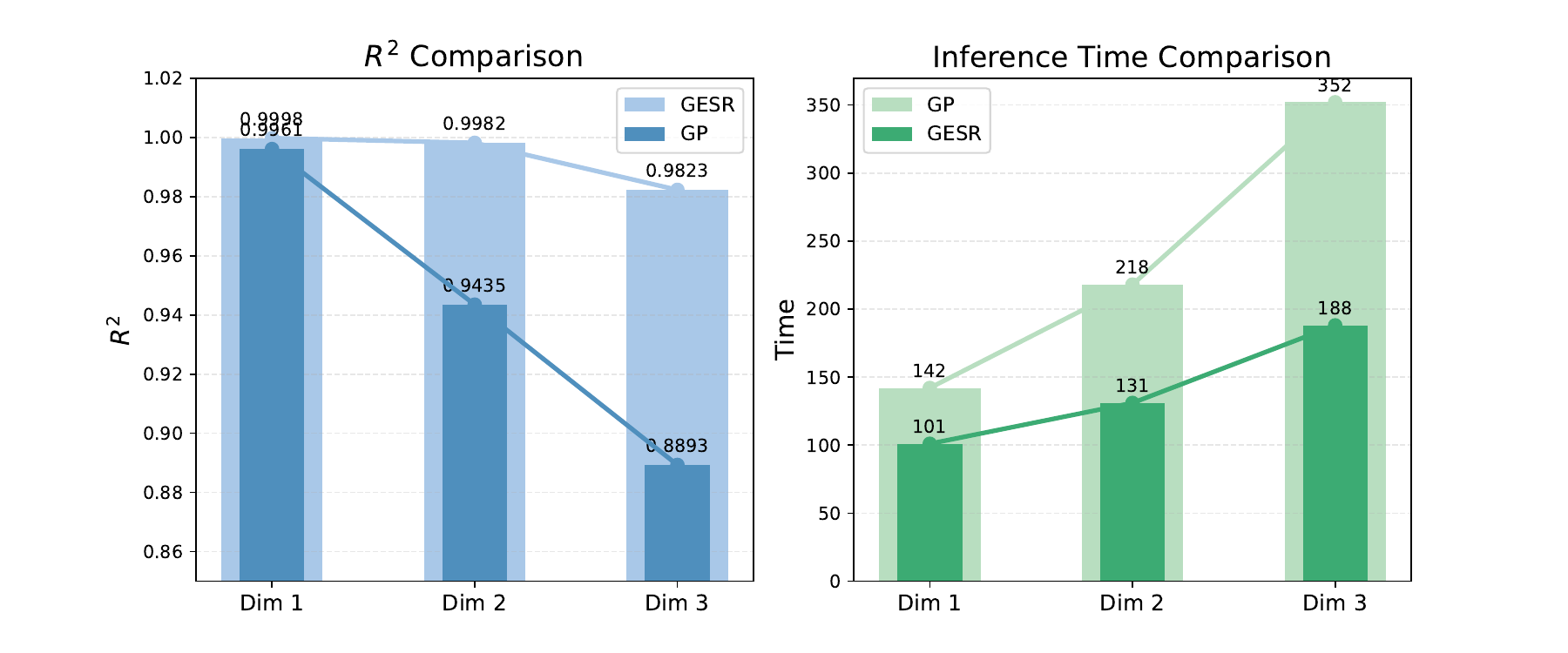}
    \vspace{-0.2cm}
    \caption{Noise robustness tests for GESR and four Baselines.}
    \label{fig:dim_compare}
    \vspace{-0.3cm}
\end{wrapfigure}

This phenomenon carries clear methodological significance. The core difficulty of high-dimensional symbolic regression lies not merely in finding the correct expression, but in efficiently approaching the correct expression within an extremely large search space. Traditional GP mainly relies on random mutation and crossover operations for exploration, and therefore lacks a strong search direction. As a result, it tends to generate a large number of low-quality candidate expressions in complex scenarios. In contrast, by introducing semantic or structural guidance from BERT, GESR makes the search process no longer entirely dependent on blind heuristic evolutionary operations, but instead provides a certain degree of directionality and prior constraint. Although the advantage of this guidance mechanism may not be particularly prominent in low-dimensional tasks, it effectively suppresses the performance degradation caused by search space explosion in high-dimensional tasks, thus exhibiting stronger robustness and scalability.

\subsection{Trend of $R^2$ with Respect to Training Data Scale}

To analyze the effect of training data scale on model performance, we further study the model performance on different benchmark datasets under different training data scales, as shown in Fig.~\ref{fig_data_size}. For symbolic regression, the expression space is extremely large, and it is almost impossible to cover all possible cases through finite sampling. Therefore, during pretraining, we can only increase the size of the training data as much as possible, so as to sample more representative expressions and their corresponding data, thereby better approximating the overall distribution.

As shown in the figure, when the training data scale increases from $10$K to $10$M, the $R^2$ values on multiple datasets exhibit a generally stable upward trend. This indicates that larger-scale training data helps the model learn more sufficient symbolic patterns and expression-editing priors, thereby improving its fitting ability and generalization performance on downstream symbolic regression tasks.

Meanwhile, it can also be observed that the performance gain gradually slows down as the training data continues to increase. In particular, in the range from $5$M to $10$M, the $R^2$ curves on most datasets become much flatter, suggesting that the model performance is gradually approaching saturation. From the results, when the training data scale reaches $10$M, the model already achieves satisfactory performance on the selected datasets, and the average $R^2$ on each dataset exceeds $0.98$. This suggests that a training set of size $10$M is already sufficient for the model to learn representative symbolic distribution characteristics.

It should be noted that, due to the open-ended nature of the expression distribution in symbolic regression, out-of-distribution samples are inevitable. Therefore, increasing the training data scale cannot completely eliminate the distribution shift problem, but it can alleviate it to some extent and improve the model’s ability to capture the dominant patterns in the data.

In summary, Fig.~\ref{fig_data_size} shows that model performance consistently improves as the training data scale increases, while the improvement gradually becomes marginal in the large-data regime. This indicates that using $10$M training samples for pretraining is reasonable, as it achieves a good balance between training cost and performance gain.

\begin{figure*}[htp]
    \centering
       \subfloat[]{
       \includegraphics[width=0.99\linewidth]{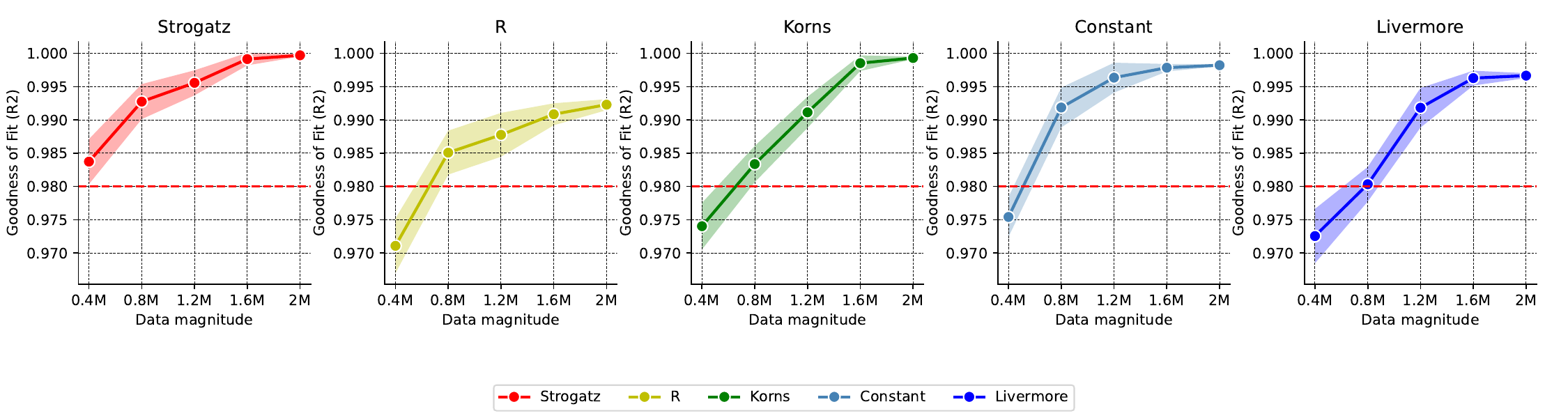} \label{fig4a}}\\
        \subfloat[]{
        \includegraphics[width=0.99\linewidth]{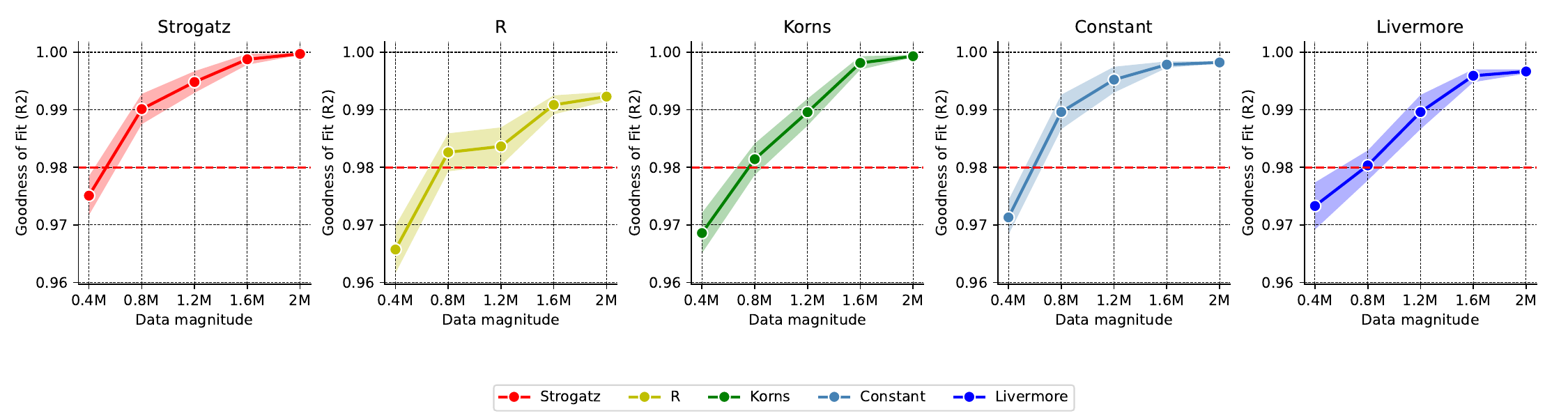}\label{fig4b}}\\
        \subfloat[]{
       \includegraphics[width=0.80\linewidth]{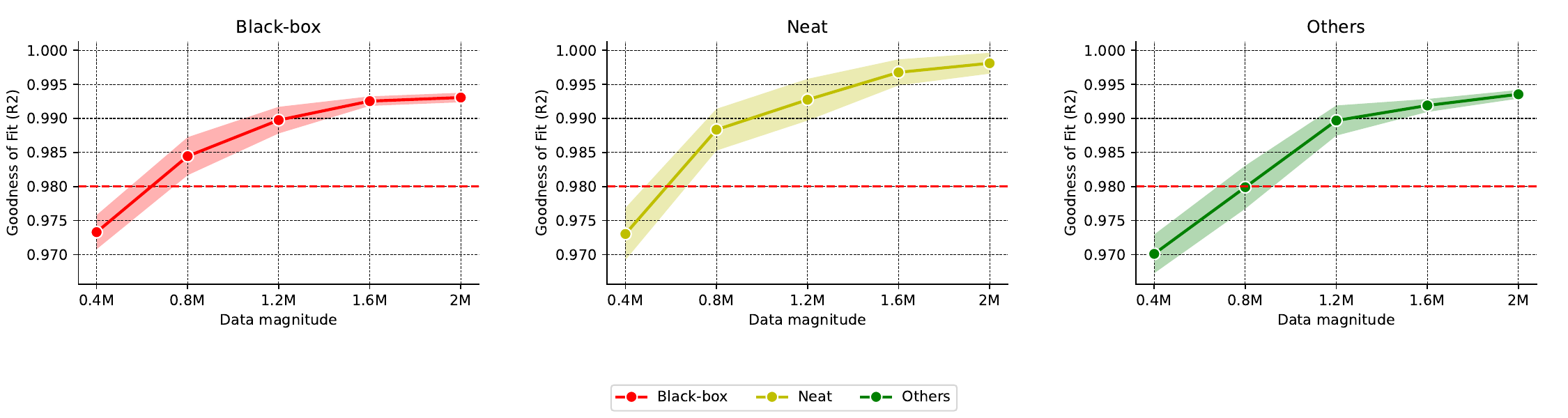} \label{fig4c}}
	  \caption{Effect of training data scale on model performance. As the size of the pretraining data increases from $10$K to $10$M, the $R^2$ values on different benchmark datasets consistently improve. When the training data reaches $10$M, the model achieves satisfactory performance on all selected datasets, with the average $R^2$ exceeding $0.98$, while the performance gains gradually become smaller in the large-data regime.}
\label{fig_data_size} 
\end{figure*}

\subsection{Comparison of Recovery Rates Across Different Algorithms}

To evaluate the structural recovery capability of different symbolic regression methods, this paper adopts the recovery rate as the evaluation metric, defined as the proportion of benchmark tasks on which an algorithm exactly recovers the ground-truth target expression. Unlike metrics that only focus on numerical fitting error, the recovery rate better reflects the actual ability of a method to discover underlying mathematical structures, and is therefore an important criterion for assessing the performance of symbolic regression algorithms.

As shown in Fig.~\ref{fig:Recovery_Rate}, GESR achieves higher recovery rates than the other compared methods on most datasets, demonstrating relatively stable overall performance.

These results indicate that GESR not only improves expression fitting accuracy, but also more effectively recovers the true analytical structure. This advantage mainly stems from the fact that the proposed method uses a multimodal BERT model to predict masked positions in symbolic expressions, transforming the random mutation process in traditional genetic programming into a data-driven guided editing process, thereby reducing ineffective search and increasing the proportion of beneficial mutations. In addition, during the construction of training data, only editing samples that lead to higher $R^2$ values are retained, enabling the model to learn which symbol replacements are more likely to improve expression quality and further enhancing its ability to recover the correct structure.

\section{Experimental Systems and Data for 16 Chaotic Ordinary Differential Equations}

To systematically evaluate the robustness and generalization capability of the proposed method under diverse nonlinear dynamical regimes, we select \textbf{16 representative continuous-time chaotic systems} as benchmark test cases. These systems span classical three-dimensional chaotic models, four-dimensional hyperchaotic systems, and dynamical models arising from \textbf{physics, chemistry, biology, engineering, and economics}
\cite{newton1975liepnik,lorenz1963deterministic,rossler1979hyperchaos,jha2002hyperchaos,pang2005hyperchaotic,shimizu1980chaos,genesio1992chaotic,haken1975laser,prigogine1968symmetry,rucklidge1992chaos,fitzhugh1961impulses,nagumo1962active,chen2008finance,li2005chaotic,lorenz1984hadley,sprott1994elegant}.
\begin{figure*}
    \centering
    \vspace{-0.2cm}
    \includegraphics[width=1.0\textwidth]{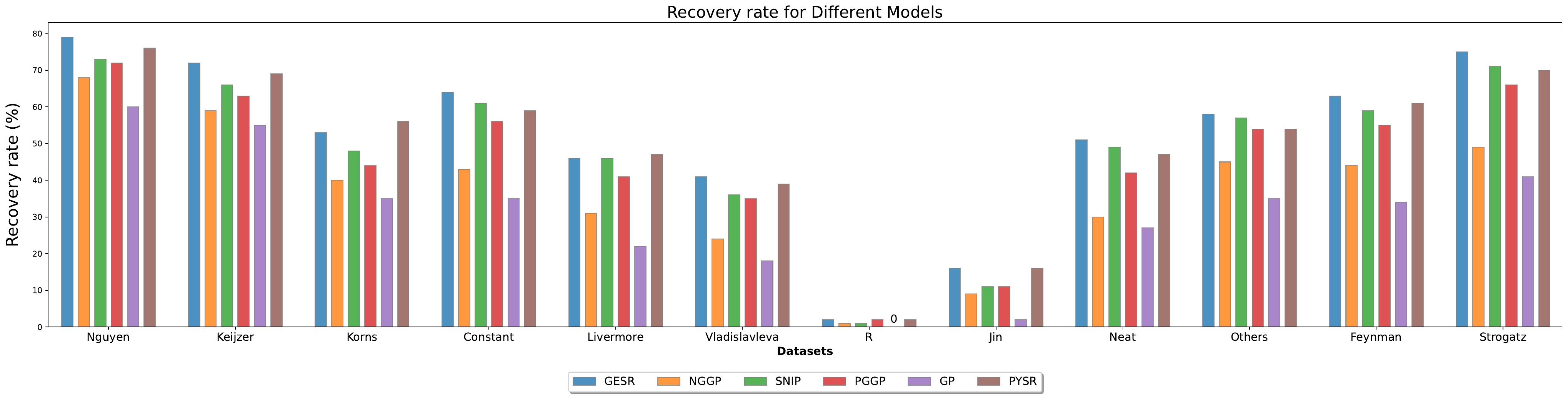}
    \vspace{-0.2cm}
    \caption{Recovery rates for GESR and baselines.}
    \label{fig:Recovery_Rate}
    \vspace{-0.3cm}
\end{figure*}
All systems are formulated as \textbf{autonomous ordinary differential equations (ODEs)}:
\begin{equation}
\frac{d\mathbf{x}(t)}{dt} = \mathbf{f}(\mathbf{x}(t)), \qquad \mathbf{x}(0)=\mathbf{x}_0,
\end{equation}
where $\mathbf{x}(t)\in\mathbb{R}^d$ and $d\in\{3,4\}$.  
For four-dimensional systems, only the first three state components are visualized for clarity. Representative trajectories are shown in Fig.~\ref{fig_16c}, and the explicit analytical forms of all governing equations are provided in the corresponding table\ref{table:cs_1to4},\ref{table:cs_5to8},\ref{table:cs_9to12},\ref{table:cs_13to16}.) should be kept the same as the main body.
\begin{table}[b!]
\centering
\caption{Chaotic dynamics 1 to 4.}
\begin{tabular}{lll}

\toprule
\textbf{Chaotic Dynamic} & \textbf{Governing Equations} & \textbf{Parameters} \\

\midrule
  NewtonLiepnik & $\begin{aligned}
    \dot{x} &= -ax + y +10yz \\
    \dot{y} &= -x-0.4y+5xz\\
    \dot{z} &= bz-5xy
           \end{aligned}$ & $a: 0.4, b: 0.175$ \\
\midrule
  HyperLorenz & $\begin{aligned}
    \dot{x} &= a (y - x) + w \\
    \dot{y} &= -x z + c x - y\\
    \dot{z} &= -b z + x y \\
    \dot{w} &= d w - x z
           \end{aligned}$ & $a: 10, b: 2.667, c: 28, d: 1.1$\\

\midrule
  HyperJha & $\begin{aligned}
    \dot{x} &= a (y - x) + w \\
    \dot{y} &= -x z + b x - y\\
    \dot{z} &= x y - c z \\
    \dot{w} &=  -x z + d w
           \end{aligned}$ & $a: 10, b: 28, c: 2.667, d: 1.3$\\

\midrule
  HyperPang & $\begin{aligned}
    \dot{x} &= a (y - x) \\
    \dot{y} &= -x z + c y + w\\
    \dot{z} &= x y - b z \\
    \dot{w} &=-d (x + y)
           \end{aligned}$ & $a: 36, b: 3, c: 20, d: 2$\\

\bottomrule
\end{tabular}
\label{table:cs_1to4}
\end{table}
\begin{table}[htbp]

\centering
\caption{Chaotic dynamics 5 to 8.}
\begin{tabular}{lll}

\toprule
\textbf{Chaotic Dynamic} & \textbf{Governing Equations} & \textbf{Parameters} \\

\midrule
  ShimizuMorioka & $\begin{aligned}
    \dot{x} &= y \\
    \dot{y} &=  x - a y - x z\\
    \dot{z} &= -b z + x ^ 2
           \end{aligned}$ & $a: 0.85, b: 0.5$ \\
\midrule
  GenesioTesi & $\begin{aligned}
    \dot{x} &= y \\
    \dot{y} &= z\\
    \dot{z} &= -c x - b y - a z + x ^ 2
           \end{aligned}$ & $a: 0.44, b: 1.1, c: 1$\\

\midrule
  Laser & $\begin{aligned}
    \dot{x} &= a (y - x) + b y z ^ 2 \\
    \dot{y} &= c x + d x z ^ 2\\
    \dot{z} &= h z + k x ^ 2
           \end{aligned}$ & $\begin{aligned}
               &a: 10.0, b: 1.0, c: 5.0,\\ &d: -1.0, h: -5.0, k: -6.0
           \end{aligned} $\\

\midrule
  Duffing & $\begin{aligned}
    \dot{x} &= y \\
    \dot{y} &= -\delta y - \beta x - \alpha x ^ 3\\
    \dot{z} &= \omega
           \end{aligned}$ & $\alpha: 1.0, \beta: -1.0, \delta: 0.1,  \omega: 1.4$\\

\bottomrule
\end{tabular}
\label{table:cs_5to8}
\end{table}
\begin{table}[htbp]

\centering
\caption{Chaotic dynamics 9 to 12.}
\begin{tabular}{lll}

\toprule
\textbf{Chaotic Dynamic} & \textbf{Governing Equations} & \textbf{Parameters} \\

\midrule
  Brusselator & $\begin{aligned}
    \dot{x} &= a + x ^ 2 y - (b + 1) x \\
    \dot{y} &= b x - x ^ 2 y\\
    \dot{z} &= w
           \end{aligned}$ & $a: 0.4, b: 1.2, w: 0.81$ \\
\midrule
  KawczynskiStrizhak & $\begin{aligned}
    \dot{x} &= \gamma (y - x ^ 3 + 3 \mu x) \\
    \dot{y} &= -2 \mu x - y - z + \beta\\
    \dot{z} &= \kappa (x - z)
           \end{aligned}$ & $\beta: -0.4, \gamma: 0.49, \kappa: 0.2, \mu: 2.1$\\

\midrule
  Rucklidge & $\begin{aligned}
    \dot{x} &= -a x + b y - y z \\
    \dot{y} &= x\\
    \dot{z} &= -z + y ^ 2
           \end{aligned}$ & $a: 2.0, b: 6.7$\\

\midrule
  FitzHughNagumo & $\begin{aligned}
    \dot{x} &= x - x ^ 3 / 3 - y + curr \\
    \dot{y} &= \gamma (x + a - b y)\\
    \dot{z} &= \omega
           \end{aligned}$ & $\begin{aligned}&a: 0.7, b: 0.8, curr: 0.965,\\& \gamma: 0.08, \omega: 0.04365\end{aligned}$\\

\bottomrule
\end{tabular}
\label{table:cs_9to12}
\end{table}
\begin{table}[htbp]

\centering
\caption{Chaotic dynamics 13 to 16.}
\begin{tabular}{lll}

\toprule
\textbf{Chaotic Dynamic} & \textbf{Governing Equations} & \textbf{Parameters} \\

\midrule
  Finance & $\begin{aligned}
    \dot{x} &= (1 / b - a) x + z + x y \\
    \dot{y} &= -b y - x ^ 2\\
    \dot{z} &= -x - c z
           \end{aligned}$ & $a: 0.001, b: 0.2, c: 1.1$ \\
\midrule
  DequanLi & $\begin{aligned}
    \dot{x} &= a (y - x) + d x z \\
    \dot{y} &= k x + f y - x z\\
    \dot{z} &= c z + x y - \epsilon x ^ 2
           \end{aligned}$ & $a: 40, c: 1.833, d: 0.16, \epsilon: 0.65, f: 20, k: 55$\\

\midrule
  Hadley & $\begin{aligned}
    \dot{x} &= -y ^ 2 - z ^ 2 - a x + a f \\
    \dot{y} &= x y - b x z - y + g\\
    \dot{z} &= b x y + x z - z
           \end{aligned}$ & $a: 0.2, b: 4, f: 9, g: 1$\\

\midrule
  SprottJerk & $\begin{aligned}
    \dot{x} &= y \\
    \dot{y} &= z\\
    \dot{z} &= -x + y ^ 2 - \mu z
           \end{aligned}$ & $\mu: 2.017$\\

\bottomrule
\end{tabular}
\label{table:cs_13to16}
\end{table}
\begin{figure*}[t]
\centering
\setlength{\belowcaptionskip}{-0.0cm} 
\includegraphics[width=136mm]{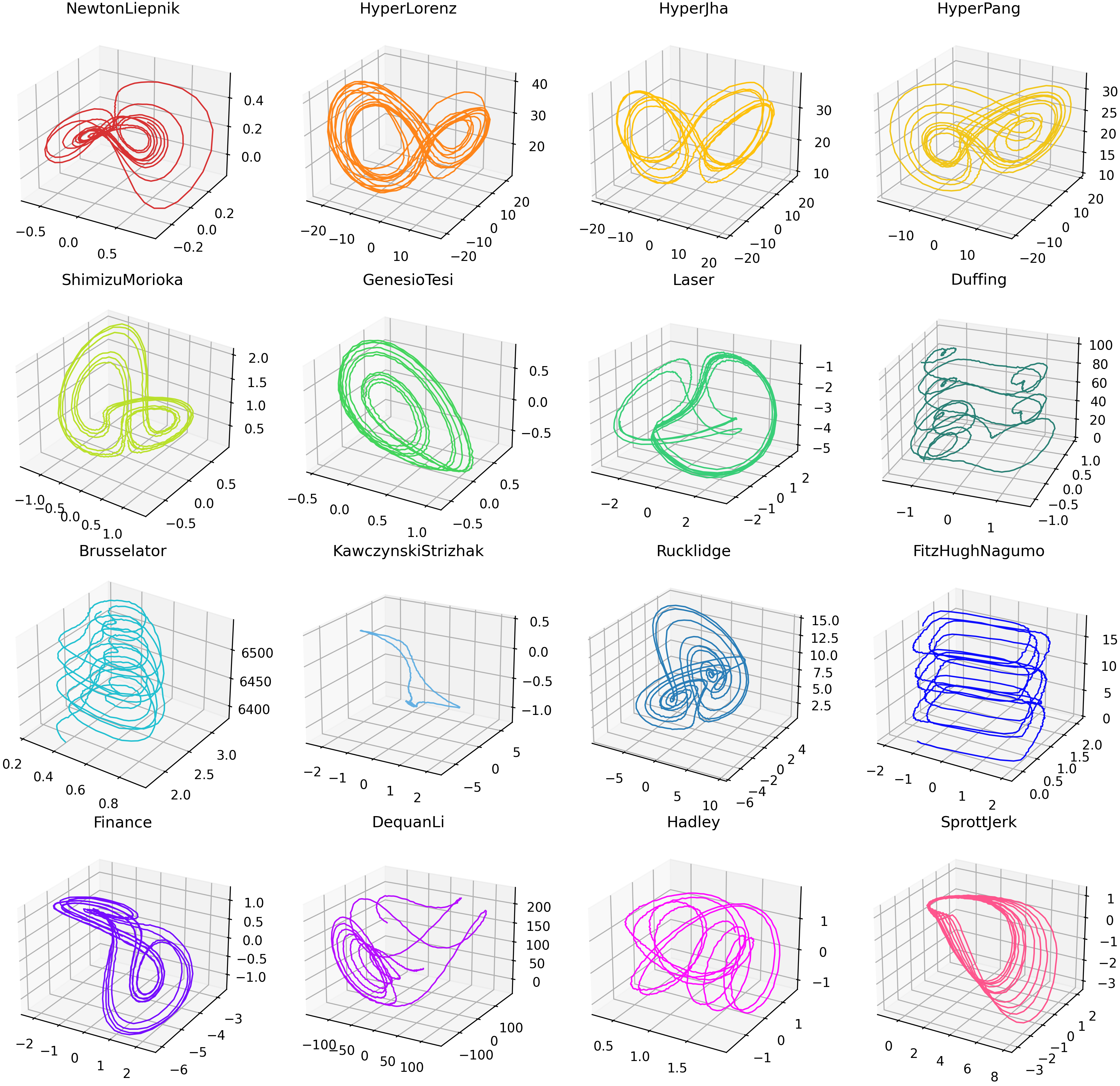}
\caption{Visualization of the trajectories of the 16 chaotic systems
}
\label{fig_16c}
\end{figure*}

\subsection{Introduction to background of chaotic systems}

\paragraph{1. Newton--Liepnik System~\cite{newton1975liepnik}}

The Newton--Liepnik system is defined by the following three-dimensional autonomous ordinary differential equations:
\begin{equation}
\begin{cases}
\dot{x} = -a x + y + 10yz,\\
\dot{y} = -x - a y + 5xz,\\
\dot{z} = b z - 5xy,
\end{cases}
\qquad a=0.4,\; b=0.175.
\end{equation}

This system is a three-dimensional nonlinear dynamical model with strong quadratic polynomial coupling terms. The bilinear interactions $xy$, $xz$, and $yz$ introduce significant energy exchange and feedback across state variables. Under the joint influence of linear damping and nonlinear coupling, the system forms a characteristic stretching-and-folding structure in phase space and generates a stable double-scroll chaotic attractor within the specified parameter regime. Owing to its explicit nonlinear structure and complex coupling patterns, this system serves as a standard benchmark for evaluating the ability of data-driven methods to identify polynomial nonlinearities.

\paragraph{2. Hyper-Lorenz System (Four-Dimensional Hyperchaos)~\cite{lorenz1963deterministic,rossler1979hyperchaos}}

The Hyper-Lorenz system is a four-dimensional extension of the classical Lorenz system, governed by:
\begin{equation}
\begin{cases}
\dot{x} = a(y-x) + w,\\
\dot{y} = cx - y - xz,\\
\dot{z} = xy - bz,\\
\dot{w} = dw - xz,
\end{cases}
\qquad a=10,\; b=2.667,\; c=28,\; d=1.1.
\end{equation}

By retaining the core nonlinear convection terms $xz$ and $xy$ from the Lorenz system while introducing an additional state variable $w$ and feedback dynamics, this system gains an extra unstable direction. Under suitable parameters, it typically exhibits two or more positive Lyapunov exponents, manifesting hyperchaotic behavior. Compared with the three-dimensional Lorenz system, the Hyper-Lorenz model presents higher-dimensional exponential divergence in phase space, posing a more stringent test for equation discovery methods in high-dimensional dynamical regimes.

\paragraph{3. Hyper-Jha System (Four-Dimensional Hyperchaos)~\cite{jha2002hyperchaos}}

The Hyper-Jha system is another Lorenz-type four-dimensional hyperchaotic system described by:
\begin{equation}
\begin{cases}
\dot{x} = a(y-x) + w,\\
\dot{y} = cx - y - xz,\\
\dot{z} = xy - bz,\\
\dot{w} = -dy,
\end{cases}
\end{equation}

Unlike the Hyper-Lorenz system, Hyper-Jha introduces a linear decay term $-dy$ in the auxiliary state equation, altering energy feedback pathways among state variables. This structural variation significantly influences attractor geometry and unstable manifold distribution without drastically changing the surface equation form, making it a valuable benchmark for distinguishing systems with similar algebraic forms but distinct dynamical behaviors.

\paragraph{4. Hyper-Pang System (Four-Dimensional Hyperchaos)~\cite{pang2005hyperchaotic}}

The Hyper-Pang system is another Lorenz-type hyperchaotic model governed by:
\begin{equation}
\begin{cases}
\dot{x} = a(y-x),\\
\dot{y} = cx - y - xz + w,\\
\dot{z} = xy - bz,\\
\dot{w} = -d(x+y),
\end{cases}
\end{equation}

A distinguishing feature of this system is that the auxiliary variable $w$ couples simultaneously to multiple original state variables through linear combinations, reshaping the global energy injection and dissipation structure. This results in phase-space projections that differ markedly from those of Hyper-Lorenz and Hyper-Jha, while preserving the Lorenz-type nonlinear backbone.

\paragraph{5. Shimizu--Morioka System~\cite{shimizu1980chaos}}

The Shimizu--Morioka system is described by the following three-dimensional autonomous ODE:
\begin{equation}
\begin{cases}
\dot{x} = y,\\
\dot{y} = x - ay - xz,\\
\dot{z} = -bz + x^2,
\end{cases}
\qquad a=0.85,\; b=0.5.
\end{equation}

Originating from low-dimensional laser dynamics modeling, this system can be interpreted as a state-dependent feedback oscillator coupled with a slow variable. The nonlinear terms $xz$ and $x^2$ introduce modulation effects in fast and slow channels, enabling complex chaotic behavior even in low dimensions. Its compact structure and interpretable nonlinear sources make it a classical prototype for studying chaos generation mechanisms.

\paragraph{6. Genesio--Tesi System~\cite{genesio1992chaotic}}

The Genesio--Tesi system is a polynomial chaotic model governed by:
\begin{equation}
\begin{cases}
\dot{x} = y,\\
\dot{y} = z,\\
\dot{z} = -cx - by - az + x^2,
\end{cases}
\qquad a=1.2,\; b=2.92,\; c=6.
\end{equation}

This system exhibits an integration-chain structure, with chaos driven solely by the quadratic term $x^2$. Its simple yet explicit nonlinear source makes it an important benchmark for evaluating the trade-off between nonlinear identification accuracy and model compactness.

\paragraph{7. Laser System~\cite{haken1975laser}}

The Laser system originates from single-mode laser rate equations and is described by:
\begin{equation}
\begin{cases}
\dot{x} = a(y-x),\\
\dot{y} = bx - y - xz,\\
\dot{z} = -cz + xy,
\end{cases}
\end{equation}

Here, state variables typically correspond to electric field intensity, polarization, and population inversion. The bilinear terms $xz$ and $xy$ describe nonlinear interactions between light and matter. Under suitable parameters, the system transitions between steady-state, periodic oscillations, and chaotic behavior, serving as a classical model in optical chaos.

\paragraph{8. Duffing Autonomous System~\cite{guckenheimer1983nonlinear}}

The Duffing system is a classical model of nonlinear vibration, expressed in autonomous form as:
\begin{equation}
\begin{cases}
\dot{x} = y,\\
\dot{y} = x - x^3 - ay,\\
\dot{z} = -bz,
\end{cases}
\end{equation}

Here, $x$ denotes displacement and $y$ velocity, while $x^3$ captures nonlinear stiffness effects. Although structurally simple, the system can exhibit complex phase-space behavior under appropriate damping parameters, making it a foundational prototype in nonlinear dynamics.

\paragraph{9. Brusselator System~\cite{prigogine1968symmetry}}

The Brusselator system models autocatalytic chemical reactions:
\begin{equation}
\begin{cases}
\dot{x} = A + x^2y - (B+1)x,\\
\dot{y} = Bx - x^2y,\\
\dot{z} = -z + x,
\end{cases}
\end{equation}

The nonlinear term $x^2y$ introduces positive feedback, enabling oscillatory and chaotic behavior far from equilibrium. This system is a canonical benchmark in nonequilibrium thermodynamics and chemical kinetics.

\paragraph{10. Kawczynski--Strizhak System~\cite{strizhak2000chaos}}

The Kawczynski--Strizhak system is a third-order polynomial chaotic model:
\begin{equation}
\begin{cases}
\dot{x} = y,\\
\dot{y} = z,\\
\dot{z} = -az - by - x + x^2,
\end{cases}
\end{equation}

Its chaos is driven entirely by the quadratic term $x^2$, making it suitable for analyzing how nonlinear forcing influences stability and bifurcation behavior.

\paragraph{11. Rucklidge System~\cite{rucklidge1992chaos}}

The Rucklidge system, originally developed for thermal convection instability, is defined as:
\begin{equation}
\begin{cases}
\dot{x} = -kx + ay - yz,\\
\dot{y} = x,\\
\dot{z} = y^2 - z,
\end{cases}
\end{equation}

It can be viewed as a Lorenz-type variant, with nonlinear coupling terms $yz$ and $y^2$ shaping its stretching-and-folding phase-space geometry.

\paragraph{12. FitzHugh--Nagumo System~\cite{fitzhugh1961impulses,nagumo1962active}}

The FitzHugh--Nagumo system is a reduced model of Hodgkin--Huxley neuron dynamics:
\begin{equation}
\begin{cases}
\dot{x} = c\!\left(x-\frac{x^3}{3}+y\right),\\
\dot{y} = -\frac{1}{c}(x-a+by),\\
\dot{z} = -z + x,
\end{cases}
\end{equation}

The cubic nonlinearity governs neuronal excitation and recovery dynamics and can produce oscillatory and chaotic firing patterns.

\paragraph{13. Finance System~\cite{chen2008finance}}

The Finance system models nonlinear macroeconomic interactions:
\begin{equation}
\begin{cases}
\dot{x} = z + (y-a)x,\\
\dot{y} = 1 - by - x^2,\\
\dot{z} = -x - cz,
\end{cases}
\end{equation}

The quadratic term $x^2$ captures saturation effects and economic feedback loops, producing chaotic market-like fluctuations.

\paragraph{14. Dequan--Li System~\cite{li2005chaotic}}

The Dequan--Li system is a Lorenz-type chaotic system:
\begin{equation}
\begin{cases}
\dot{x} = a(y-x) + yz,\\
\dot{y} = cx - xz + dy,\\
\dot{z} = xy - bz,
\end{cases}
\end{equation}

Additional bilinear couplings reshape attractor geometry, making it suitable for testing modeling under mixed linear–nonlinear structures.

\paragraph{15. Hadley Circulation Model~\cite{lorenz1984hadley}}

The Hadley circulation model describes large-scale atmospheric circulation:
\begin{equation}
\begin{cases}
\dot{x} = -y^2 - z^2 - ax + ac,\\
\dot{y} = xy - bxz - y + d,\\
\dot{z} = bxy + xz - z,
\end{cases}
\end{equation}

Its nonlinear energy-exchange terms allow complex low-dimensional climate dynamics.

\paragraph{16. Sprott--Jerk System~\cite{sprott1994elegant}}

The Sprott--Jerk system is a piecewise nonsmooth chaotic model:
\begin{equation}
\begin{cases}
\dot{x} = y,\\
\dot{y} = z,\\
\dot{z} = -az - y + \mathrm{sign}(x).
\end{cases}
\end{equation}

Chaos arises from the nonsmooth nonlinearity $\mathrm{sign}(x)$, making this system a challenging benchmark for symbolic regression under nonsmooth dynamics.

\subsection{Comprehensive Evaluation Protocol Based on Per-Dimension $R^2$ and Short-Horizon Rollout Consistency}
\label{subsec:chaos_r2_rollout}

To assess vector-field learning performance on chaotic dynamical systems in a fine-grained manner---without relying on symbolic structure matching or exact equation recovery---we adopt two complementary continuous metrics:
(i) per-dimension vector-field prediction accuracy $R^2_{\dot{\mathbf{x}}}$, and
(ii) short-horizon rollout consistency $R^2_{\mathrm{roll}}(K)$ over a fixed time window.
These metrics characterize model performance from the perspectives of \emph{local dynamical fitting} and \emph{short-term dynamical evolution}, respectively, and are jointly summarized in the same result tables for comparison.

\subsubsection{Per-dimension vector-field prediction accuracy.}
Consider a $d$-dimensional autonomous chaotic system
\(
\dot{\mathbf{x}}=\mathbf{f}(\mathbf{x})
\),
and let the learned vector field be
\(
\hat{\mathbf{f}}(\mathbf{x})=(\hat f_1,\ldots,\hat f_d)
\).
Rather than reporting a single aggregated $R^2$ score, we evaluate the derivative prediction accuracy independently for each state component, yielding per-dimension coefficients of determination
\(
\{R_k^2\}_{k=1}^{d}
\):
\begin{equation}
R_k^2
=
1
-
\frac{
\sum_{i\in\mathcal{D}_{\mathrm{te}}}
\left(
\tilde{\dot x}_{i}^{(k)}
-
\hat f_k(\tilde{\mathbf{x}}_i)
\right)^2
}{
\sum_{i\in\mathcal{D}_{\mathrm{te}}}
\left(
\tilde{\dot x}_{i}^{(k)}
-
\overline{\tilde{\dot x}^{(k)}}
\right)^2
},
\qquad k=1,\ldots,d .
\end{equation}
This per-dimension evaluation avoids metric masking caused by heterogeneous error scales across state variables, and directly reflects the recovery quality of each individual differential equation.
In Tables~\ref{tab:chaos_r2_full_part1} and~\ref{tab:chaos_r2_full_part2}, we report both the per-dimension scores $R_k^2$ and their average,
denoted by
\(
R^2_{\mathrm{mean}}
\).

\subsubsection{Short-horizon rollout consistency ($K=50$).}
Beyond one-step vector-field accuracy, we further evaluate short-term dynamical consistency induced by numerical integration.
Specifically, we select a set of starting indices $\mathcal{I}$ from the test trajectories.
For each $i\in\mathcal{I}$, we use the ground-truth state $\mathbf{x}_i$ as the initial condition and perform forward numerical integration with the learned vector field $\hat{\mathbf{f}}$ for a horizon of $K=50$ steps, producing a predicted trajectory segment
\(
\{\hat{\mathbf{x}}_{i,k}\}_{k=0}^{K}.
\)
Each rollout is independently initialized from the ground-truth state (i.e., reset at every window), which prevents long-term error accumulation from dominating the evaluation.

For each rollout window, we compute the coefficient of determination between the predicted segment
\(
\{\hat{\mathbf{x}}_{i,k}\}_{k=0}^{K}
\)
and the corresponding ground-truth segment
\(
\{\mathbf{x}_{i+k}\}_{k=0}^{K}
\),
Average it over state dimensions, and finally average across all starting indices.
The resulting short-horizon rollout metric is defined as:
\begin{equation}
R^2_{\mathrm{roll}}(K)
=
\frac{1}{|\mathcal{I}|}
\sum_{i\in\mathcal{I}}
R^2\!\left(
\{\mathbf{x}_{i:i+K}\},
\{\hat{\mathbf{x}}_{i,0:K}\}
\right),
\qquad K=50 .
\end{equation}
This metric quantifies the extent to which the learned vector field preserves the true dynamical evolution over short to moderate horizons.
In chaotic systems, it effectively captures the cumulative effect of local vector-field errors while avoiding evaluation distortion caused by the inevitable divergence of long-term trajectories.

\begin{table*}[t]
\centering
\small
\setlength{\tabcolsep}{4.0pt}
\renewcommand{\arraystretch}{1.12}
\caption{%
Quantitative results on the 16 chaotic systems under a noise level of $\rho = 2\%$.
For each system, the coefficient of determination $R^2_{\dot{x}_j}$ is computed independently for each state dimension, and then averaged over dimensions to obtain the overall metric $R^2_{\dot{\mathbf{x}}}$.
Short-term rollout consistency is measured by $R^2_{\mathrm{roll}}(50)$.
For compact presentation, standard deviations are indicated as subscripts in the lower-right corner (in {\scriptsize} font).
Within each system, the best-performing method with respect to the primary metric is highlighted in bold.
}
\label{tab:chaos_r2_full_part1}
\begin{tabular}{l c l c c c c c}
\toprule
System & $d$ & Method
& $R^2_{\dot x}\uparrow$
& $R^2_{\dot y}\uparrow$
& $R^2_{\dot z}\uparrow$
& $R^2_{\dot{\mathbf{x}}}\uparrow$
& $R^2_{\mathrm{roll}}(50)\uparrow$ \\
\midrule

\multirow{5}{*}{Newton--Liepnik}
& \multirow{5}{*}{3}
& GESR
& $0.996_{\scriptsize \pm 0.004}$ & $0.996_{\scriptsize \pm 0.005}$ & $0.995_{\scriptsize \pm 0.004}$
& \textbf{$0.994_{\scriptsize \pm 0.004}$}
& \textbf{$0.885_{\scriptsize \pm 0.036}$} \\
& & PySR
& $0.994_{\scriptsize \pm 0.004}$ & $0.993_{\scriptsize \pm 0.004}$ & $0.993_{\scriptsize \pm 0.004}$
& $0.993_{\scriptsize \pm 0.004}$
& $0.845_{\scriptsize \pm 0.041}$ \\
& & NGGP
& $0.991_{\scriptsize \pm 0.006}$ & $0.990_{\scriptsize \pm 0.006}$ & $0.989_{\scriptsize \pm 0.007}$
& $0.990_{\scriptsize \pm 0.006}$
& $0.742_{\scriptsize \pm 0.072}$ \\
& & SNIP
& $0.992_{\scriptsize \pm 0.005}$ & $0.991_{\scriptsize \pm 0.006}$ & $0.991_{\scriptsize \pm 0.005}$
& $0.991_{\scriptsize \pm 0.005}$
& $0.761_{\scriptsize \pm 0.068}$ \\
& & GPlearn
& $0.988_{\scriptsize \pm 0.008}$ & $0.987_{\scriptsize \pm 0.009}$ & $0.986_{\scriptsize \pm 0.009}$
& $0.987_{\scriptsize \pm 0.008}$
& $0.593_{\scriptsize \pm 0.110}$ \\
\midrule

\multirow{5}{*}{HyperLorenz}
& \multirow{5}{*}{4}
& GESR
& $0.998_{\scriptsize \pm 0.003}$ & $0.997_{\scriptsize \pm 0.003}$ & $0.994_{\scriptsize \pm 0.006}$
& \textbf{$0.996_{\scriptsize \pm 0.001}$}
& \textbf{$0.846_{\scriptsize \pm 0.044}$} \\
& & PySR
& $0.993_{\scriptsize \pm 0.004}$ & $0.994_{\scriptsize \pm 0.004}$ & $0.991_{\scriptsize \pm 0.005}$
& $0.993_{\scriptsize \pm 0.004}$
& $0.802_{\scriptsize \pm 0.052}$ \\
& & NGGP
& $0.990_{\scriptsize \pm 0.006}$ & $0.991_{\scriptsize \pm 0.006}$ & $0.988_{\scriptsize \pm 0.007}$
& $0.990_{\scriptsize \pm 0.006}$
& $0.731_{\scriptsize \pm 0.075}$ \\
& & SNIP
& $0.991_{\scriptsize \pm 0.006}$ & $0.992_{\scriptsize \pm 0.005}$ & $0.989_{\scriptsize \pm 0.006}$
& $0.991_{\scriptsize \pm 0.006}$
& $0.748_{\scriptsize \pm 0.070}$ \\
& & GPlearn
& $0.987_{\scriptsize \pm 0.009}$ & $0.986_{\scriptsize \pm 0.010}$ & $0.987_{\scriptsize \pm 0.009}$
& $0.987_{\scriptsize \pm 0.009}$
& $0.612_{\scriptsize \pm 0.103}$ \\
\midrule

\multirow{5}{*}{HyperJha}
& \multirow{5}{*}{4}
& GESR
& $0.995_{\scriptsize \pm 0.004}$ & $0.995_{\scriptsize \pm 0.004}$ & $0.994_{\scriptsize \pm 0.004}$
& \textbf{$0.996_{\scriptsize \pm 0.005}$}
& \textbf{$0.841_{\scriptsize \pm 0.008}$} \\
& & PySR
& $0.992_{\scriptsize \pm 0.004}$ & $0.993_{\scriptsize \pm 0.004}$ & $0.993_{\scriptsize \pm 0.004}$
& $0.993_{\scriptsize \pm 0.004}$
& $0.783_{\scriptsize \pm 0.058}$ \\
& & NGGP
& $0.989_{\scriptsize \pm 0.007}$ & $0.990_{\scriptsize \pm 0.006}$ & $0.989_{\scriptsize \pm 0.007}$
& $0.989_{\scriptsize \pm 0.007}$
& $0.694_{\scriptsize \pm 0.084}$ \\
& & SNIP
& $0.990_{\scriptsize \pm 0.006}$ & $0.991_{\scriptsize \pm 0.006}$ & $0.990_{\scriptsize \pm 0.006}$
& $0.990_{\scriptsize \pm 0.006}$
& $0.711_{\scriptsize \pm 0.079}$ \\
& & GPlearn
& $0.986_{\scriptsize \pm 0.010}$ & $0.987_{\scriptsize \pm 0.009}$ & $0.985_{\scriptsize \pm 0.010}$
& $0.986_{\scriptsize \pm 0.010}$
& $0.541_{\scriptsize \pm 0.123}$ \\
\midrule

\multirow{5}{*}{HyperPang}
& \multirow{5}{*}{4}
& GESR
& $0.996_{\scriptsize \pm 0.003}$ & $0.994_{\scriptsize \pm 0.005}$ & $0.996_{\scriptsize \pm 0.003}$
& \textbf{$0.996_{\scriptsize \pm 0.004}$}
& \textbf{$0.864_{\scriptsize \pm 0.010}$} \\
& & PySR
& $0.994_{\scriptsize \pm 0.004}$ & $0.992_{\scriptsize \pm 0.005}$ & $0.993_{\scriptsize \pm 0.004}$
& $0.993_{\scriptsize \pm 0.004}$
& $0.792_{\scriptsize \pm 0.058}$ \\
& & NGGP
& $0.989_{\scriptsize \pm 0.007}$ & $0.988_{\scriptsize \pm 0.007}$ & $0.987_{\scriptsize \pm 0.008}$
& $0.988_{\scriptsize \pm 0.007}$
& $0.642_{\scriptsize \pm 0.110}$  
\\
& & SNIP
& $0.991_{\scriptsize \pm 0.006}$ & $0.990_{\scriptsize \pm 0.006}$ & $0.989_{\scriptsize \pm 0.007}$
& $0.990_{\scriptsize \pm 0.006}$
& $0.601_{\scriptsize \pm 0.132}$  
\\
& & GPlearn
& $0.987_{\scriptsize \pm 0.010}$ & $0.988_{\scriptsize \pm 0.010}$ & $0.985_{\scriptsize \pm 0.011}$
& $0.987_{\scriptsize \pm 0.010}$
& $0.498_{\scriptsize \pm 0.140}$ \\
\midrule

\multirow{5}{*}{Shimizu--Morioka}
& \multirow{5}{*}{3}
& GESR
& $0.997_{\scriptsize \pm 0.003}$ & $0.996_{\scriptsize \pm 0.005}$ & $0.996_{\scriptsize \pm 0.003}$
& \textbf{$0.997_{\scriptsize \pm 0.004}$}
& \textbf{$0.914_{\scriptsize \pm 0.032}$} \\
& & PySR
& $0.995_{\scriptsize \pm 0.004}$ & $0.993_{\scriptsize \pm 0.004}$ & $0.995_{\scriptsize \pm 0.004}$
& $0.994_{\scriptsize \pm 0.004}$
& $0.874_{\scriptsize \pm 0.040}$ \\
& & NGGP
& $0.992_{\scriptsize \pm 0.006}$ & $0.990_{\scriptsize \pm 0.006}$ & $0.991_{\scriptsize \pm 0.006}$
& $0.991_{\scriptsize \pm 0.006}$
& $0.641_{\scriptsize \pm 0.105}$ \\
& & SNIP
& $0.993_{\scriptsize \pm 0.005}$ & $0.991_{\scriptsize \pm 0.006}$ & $0.992_{\scriptsize \pm 0.005}$
& $0.992_{\scriptsize \pm 0.005}$
& $0.668_{\scriptsize \pm 0.098}$ \\
& & GPlearn
& $0.989_{\scriptsize \pm 0.009}$ & $0.988_{\scriptsize \pm 0.010}$ & $0.989_{\scriptsize \pm 0.009}$
& $0.989_{\scriptsize \pm 0.009}$
& $0.452_{\scriptsize \pm 0.136}$ \\

\bottomrule
\end{tabular}
\end{table*}

\begin{table*}[t]
\centering
\small
\setlength{\tabcolsep}{4.0pt}
\renewcommand{\arraystretch}{1.12}
\caption{%
(Continued) Quantitative results for the chaotic systems under a noise level of $\rho = 2\%$.
All metrics are defined identically to those in Table~\ref{tab:chaos_r2_full_part1}.
}
\label{tab:chaos_r2_full_part1}
\begin{tabular}{l c l c c c c c}
\toprule
System & $d$ & Method
& $R^2_{\dot x}\uparrow$
& $R^2_{\dot y}\uparrow$
& $R^2_{\dot z}\uparrow$
& $R^2_{\dot{\mathbf{x}}}\uparrow$
& $R^2_{\mathrm{roll}}(50)\uparrow$ \\
\midrule
\multirow{5}{*}{Genesio--Tesi}
& \multirow{5}{*}{3}
& GESR
& $0.996_{\scriptsize \pm 0.005}$ & $0.996_{\scriptsize \pm 0.004}$ & $0.997_{\scriptsize \pm 0.005}$
& \textbf{$0.996_{\scriptsize \pm 0.004}$}
& \textbf{$0.892_{\scriptsize \pm 0.037}$} \\
& & PySR
& $0.993_{\scriptsize \pm 0.004}$ & $0.992_{\scriptsize \pm 0.004}$ & $0.992_{\scriptsize \pm 0.005}$
& $0.993_{\scriptsize \pm 0.004}$
& $0.852_{\scriptsize \pm 0.046}$ \\
& & NGGP
& $0.990_{\scriptsize \pm 0.007}$ & $0.991_{\scriptsize \pm 0.006}$ & $0.989_{\scriptsize \pm 0.008}$
& $0.990_{\scriptsize \pm 0.007}$
& $0.758_{\scriptsize \pm 0.082}$ \\
& & SNIP
& $0.991_{\scriptsize \pm 0.006}$ & $0.992_{\scriptsize \pm 0.006}$ & $0.990_{\scriptsize \pm 0.007}$
& $0.991_{\scriptsize \pm 0.006}$
& $0.781_{\scriptsize \pm 0.078}$ \\
& & GPlearn
& $0.987_{\scriptsize \pm 0.010}$ & $0.988_{\scriptsize \pm 0.010}$ & $0.987_{\scriptsize \pm 0.010}$
& $0.987_{\scriptsize \pm 0.010}$
& $0.621_{\scriptsize \pm 0.112}$ \\
\midrule

\multirow{5}{*}{Laser}
& \multirow{5}{*}{3}
& GESR
& $0.998_{\scriptsize \pm 0.002}$ & $0.998_{\scriptsize \pm 0.004}$ & $0.994_{\scriptsize \pm 0.006}$
& \textbf{$0.993_{\scriptsize \pm 0.004}$}
& \textbf{$0.798_{\scriptsize \pm 0.043}$} \\
& & PySR
& $0.994_{\scriptsize \pm 0.005}$ & $0.993_{\scriptsize \pm 0.005}$ & $0.990_{\scriptsize \pm 0.007}$
& $0.992_{\scriptsize \pm 0.005}$
& $0.731_{\scriptsize \pm 0.070}$ \\
& & NGGP
& $0.988_{\scriptsize \pm 0.008}$ & $0.989_{\scriptsize \pm 0.007}$ & $0.986_{\scriptsize \pm 0.009}$
& $0.988_{\scriptsize \pm 0.008}$
& $0.512_{\scriptsize \pm 0.130}$  
\\
& & SNIP
& $0.989_{\scriptsize \pm 0.007}$ & $0.990_{\scriptsize \pm 0.007}$ & $0.987_{\scriptsize \pm 0.009}$
& $0.989_{\scriptsize \pm 0.007}$
& $0.433_{\scriptsize \pm 0.148}$  
\\
& & GPlearn
& $0.986_{\scriptsize \pm 0.011}$ & $0.987_{\scriptsize \pm 0.010}$ & $0.984_{\scriptsize \pm 0.012}$
& $0.986_{\scriptsize \pm 0.011}$
& $0.372_{\scriptsize \pm 0.150}$ \\
\midrule

\multirow{5}{*}{Duffing}
& \multirow{5}{*}{2}
& GESR
& $0.995_{\scriptsize \pm 0.004}$ & $0.999_{\scriptsize \pm 0.004}$ & $0.999_{\scriptsize \pm 0.000}$
& \textbf{$0.996_{\scriptsize \pm 0.003}$}
& \textbf{$0.892_{\scriptsize \pm 0.032}$} \\
& & PySR
& $0.993_{\scriptsize \pm 0.004}$ & $0.995_{\scriptsize \pm 0.004}$ & $0.999_{\scriptsize \pm 0.000}$
& $0.994_{\scriptsize \pm 0.004}$
& $0.831_{\scriptsize \pm 0.050}$ \\
& & NGGP
& $0.989_{\scriptsize \pm 0.006}$ & $0.990_{\scriptsize \pm 0.006}$ & $0.999_{\scriptsize \pm 0.000}$
& $0.990_{\scriptsize \pm 0.006}$
& $0.644_{\scriptsize \pm 0.090}$ \\
& & SNIP
& $0.990_{\scriptsize \pm 0.006}$ & $0.991_{\scriptsize \pm 0.006}$ & $0.999_{\scriptsize \pm 0.000}$
& $0.991_{\scriptsize \pm 0.006}$
& $0.692_{\scriptsize \pm 0.082}$ \\
& & GPlearn
& $0.987_{\scriptsize \pm 0.009}$ & $0.988_{\scriptsize \pm 0.009}$ & $0.994_{\scriptsize \pm 0.000}$
& $0.988_{\scriptsize \pm 0.009}$
& $0.541_{\scriptsize \pm 0.105}$ \\
\midrule

\multirow{5}{*}{Brusselator}
& \multirow{5}{*}{2}
& GESR
& $0.996_{\scriptsize \pm 0.005}$ & $0.998_{\scriptsize \pm 0.003}$ & $0.999_{\scriptsize \pm 0.003}$
& \textbf{$0.996_{\scriptsize \pm 0.002}$}
& \textbf{$0.952_{\scriptsize \pm 0.021}$} \\
& & PySR
& $0.994_{\scriptsize \pm 0.004}$ & $0.993_{\scriptsize \pm 0.004}$ & $0.999_{\scriptsize \pm 0.000}$
& $0.994_{\scriptsize \pm 0.004}$
& $0.901_{\scriptsize \pm 0.034}$ \\
& & NGGP
& $0.990_{\scriptsize \pm 0.006}$ & $0.989_{\scriptsize \pm 0.006}$ & $0.999_{\scriptsize \pm 0.000}$
& $0.989_{\scriptsize \pm 0.006}$
& $0.331_{\scriptsize \pm 0.135}$  
\\
& & SNIP
& $0.991_{\scriptsize \pm 0.006}$ & $0.990_{\scriptsize \pm 0.006}$ & $0.999_{\scriptsize \pm 0.000}$
& $0.990_{\scriptsize \pm 0.006}$
& $0.281_{\scriptsize \pm 0.142}$  
\\
& & GPlearn
& $0.988_{\scriptsize \pm 0.008}$ & $0.987_{\scriptsize \pm 0.009}$ & $0.999_{\scriptsize \pm 0.000}$
& $0.988_{\scriptsize \pm 0.008}$
& $0.212_{\scriptsize \pm 0.150}$ \\
\bottomrule
\end{tabular}
\end{table*}

\begin{table*}[h]
\centering
\small
\setlength{\tabcolsep}{4.5pt}
\renewcommand{\arraystretch}{1.12}
\caption{
(Continued) Quantitative results for the remaining six chaotic systems under a noise level of $\rho = 2\%$.
All metrics are defined identically to those in Table~\ref{tab:chaos_r2_full_part1}.
}
\label{tab:chaos_r2_full_part2}
\begin{tabular}{l c l c c c c c}
\toprule
System & $d$ & Method
& $R^2_{\dot x}\uparrow$
& $R^2_{\dot y}\uparrow$
& $R^2_{\dot z}\uparrow$
& $R^2_{\dot{\mathbf{x}}}\uparrow$
& $R^2_{\mathrm{roll}}(50)\uparrow$ \\
\midrule

\multirow{5}{*}{Kawczynski--Strizhak}
& \multirow{5}{*}{3}
& GESR
& $0.995_{\scriptsize \pm 0.005}$ & $0.996_{\scriptsize \pm 0.005}$ & $0.992_{\scriptsize \pm 0.006}$
& \textbf{$0.994_{\scriptsize \pm 0.005}$}
& \textbf{$0.752_{\scriptsize \pm 0.038}$} \\
& & PySR
& $0.993_{\scriptsize \pm 0.006}$ & $0.993_{\scriptsize \pm 0.005}$ & $0.991_{\scriptsize \pm 0.007}$
& $0.992_{\scriptsize \pm 0.006}$
& $0.692_{\scriptsize \pm 0.087}$ \\
& & NGGP
& $0.988_{\scriptsize \pm 0.008}$ & $0.989_{\scriptsize \pm 0.007}$ & $0.986_{\scriptsize \pm 0.009}$
& $0.988_{\scriptsize \pm 0.008}$
& $0.571_{\scriptsize \pm 0.112}$  
\\
& & SNIP
& $0.989_{\scriptsize \pm 0.007}$ & $0.990_{\scriptsize \pm 0.007}$ & $0.987_{\scriptsize \pm 0.009}$
& $0.989_{\scriptsize \pm 0.007}$
& $0.534_{\scriptsize \pm 0.120}$  
\\
& & GPlearn
& $0.985_{\scriptsize \pm 0.011}$ & $0.986_{\scriptsize \pm 0.010}$ & $0.983_{\scriptsize \pm 0.012}$
& $0.985_{\scriptsize \pm 0.011}$
& $0.462_{\scriptsize \pm 0.135}$ \\
\midrule

\multirow{5}{*}{Rucklidge}
& \multirow{5}{*}{3}
& GESR
& $0.998_{\scriptsize \pm 0.004}$ & $0.996_{\scriptsize \pm 0.003}$ & $0.994_{\scriptsize \pm 0.003}$
& \textbf{$0.997_{\scriptsize \pm 0.004}$}
& \textbf{$0.892_{\scriptsize \pm 0.031}$} \\
& & PySR
& $0.994_{\scriptsize \pm 0.004}$ & $0.993_{\scriptsize \pm 0.005}$ & $0.993_{\scriptsize \pm 0.004}$
& $0.993_{\scriptsize \pm 0.004}$
& $0.835_{\scriptsize \pm 0.049}$ \\
& & NGGP
& $0.990_{\scriptsize \pm 0.006}$ & $0.989_{\scriptsize \pm 0.007}$ & $0.989_{\scriptsize \pm 0.006}$
& $0.989_{\scriptsize \pm 0.006}$
& $0.421_{\scriptsize \pm 0.140}$  
\\
& & SNIP
& $0.991_{\scriptsize \pm 0.006}$ & $0.990_{\scriptsize \pm 0.006}$ & $0.990_{\scriptsize \pm 0.006}$
& $0.990_{\scriptsize \pm 0.006}$
& $0.463_{\scriptsize \pm 0.132}$  
\\
& & GPlearn
& $0.987_{\scriptsize \pm 0.010}$ & $0.987_{\scriptsize \pm 0.010}$ & $0.986_{\scriptsize \pm 0.010}$
& $0.987_{\scriptsize \pm 0.010}$
& $0.392_{\scriptsize \pm 0.148}$ \\
\midrule

\multirow{5}{*}{FitzHugh--Nagumo}
& \multirow{5}{*}{2}
& GESR
& $0.997_{\scriptsize \pm 0.004}$ & $0.995_{\scriptsize \pm 0.005}$ & $0.999_{\scriptsize \pm 0.003}$
& \textbf{$0.996_{\scriptsize \pm 0.004}$}
& \textbf{$0.863_{\scriptsize \pm 0.031}$} \\
& & PySR
& $0.994_{\scriptsize \pm 0.004}$ & $0.992_{\scriptsize \pm 0.004}$ & $0.999_{\scriptsize \pm 0.000}$
& $0.993_{\scriptsize \pm 0.004}$
& $0.812_{\scriptsize \pm 0.054}$ \\
& & NGGP
& $0.990_{\scriptsize \pm 0.006}$ & $0.989_{\scriptsize \pm 0.006}$ & $0.999_{\scriptsize \pm 0.000}$
& $0.989_{\scriptsize \pm 0.006}$
& $0.603_{\scriptsize \pm 0.098}$  
\\
& & SNIP
& $0.991_{\scriptsize \pm 0.006}$ & $0.990_{\scriptsize \pm 0.006}$ & $0.999_{\scriptsize \pm 0.000}$
& $0.990_{\scriptsize \pm 0.006}$
& $0.571_{\scriptsize \pm 0.110}$  
\\
& & GPlearn
& $0.988_{\scriptsize \pm 0.009}$ & $0.987_{\scriptsize \pm 0.009}$ & $0.999_{\scriptsize \pm 0.000}$
& $0.988_{\scriptsize \pm 0.009}$
& $0.492_{\scriptsize \pm 0.123}$ \\
\midrule

\multirow{5}{*}{Finance}
& \multirow{5}{*}{3}
& GESR
& $0.996_{\scriptsize \pm 0.005}$ & $0.997_{\scriptsize \pm 0.006}$ & $0.993_{\scriptsize \pm 0.004}$
& \textbf{$0.995_{\scriptsize \pm 0.002}$}
& \textbf{$0.735_{\scriptsize \pm 0.045}$} \\
& & PySR
& $0.993_{\scriptsize \pm 0.005}$ & $0.994_{\scriptsize \pm 0.005}$ & $0.992_{\scriptsize \pm 0.005}$
& $0.993_{\scriptsize \pm 0.005}$
& $0.654_{\scriptsize \pm 0.097}$ \\
& & NGGP
& $0.989_{\scriptsize \pm 0.007}$ & $0.990_{\scriptsize \pm 0.007}$ & $0.988_{\scriptsize \pm 0.007}$
& $0.989_{\scriptsize \pm 0.007}$
& $0.571_{\scriptsize \pm 0.118}$ \\
& & SNIP
& $0.990_{\scriptsize \pm 0.007}$ & $0.991_{\scriptsize \pm 0.006}$ & $0.989_{\scriptsize \pm 0.007}$
& $0.990_{\scriptsize \pm 0.007}$
& $0.542_{\scriptsize \pm 0.122}$ \\
& & GPlearn
& $0.986_{\scriptsize \pm 0.011}$ & $0.987_{\scriptsize \pm 0.010}$ & $0.986_{\scriptsize \pm 0.011}$
& $0.986_{\scriptsize \pm 0.011}$
& $0.451_{\scriptsize \pm 0.140}$ \\
\midrule

\multirow{5}{*}{DequanLi}
& \multirow{5}{*}{3}
& GESR
& $0.997_{\scriptsize \pm 0.004}$ & $0.996_{\scriptsize \pm 0.004}$ & $0.992_{\scriptsize \pm 0.005}$
& \textbf{$0.998_{\scriptsize \pm 0.006}$}
& \textbf{$0.814_{\scriptsize \pm 0.042}$} \\
& & PySR
& $0.993_{\scriptsize \pm 0.005}$ & $0.994_{\scriptsize \pm 0.004}$ & $0.991_{\scriptsize \pm 0.006}$
& $0.993_{\scriptsize \pm 0.005}$
& $0.761_{\scriptsize \pm 0.069}$ \\
& & NGGP
& $0.989_{\scriptsize \pm 0.007}$ & $0.989_{\scriptsize \pm 0.007}$ & $0.988_{\scriptsize \pm 0.008}$
& $0.989_{\scriptsize \pm 0.007}$
& $0.662_{\scriptsize \pm 0.090}$ \\
& & SNIP
& $0.990_{\scriptsize \pm 0.007}$ & $0.990_{\scriptsize \pm 0.006}$ & $0.989_{\scriptsize \pm 0.007}$
& $0.990_{\scriptsize \pm 0.007}$
& $0.591_{\scriptsize \pm 0.110}$  
\\
& & GPlearn
& $0.986_{\scriptsize \pm 0.011}$ & $0.987_{\scriptsize \pm 0.010}$ & $0.985_{\scriptsize \pm 0.012}$
& $0.986_{\scriptsize \pm 0.011}$
& $0.502_{\scriptsize \pm 0.128}$ \\
\midrule

\multirow{5}{*}{Hadley}
& \multirow{5}{*}{3}
& GESR
& $0.998_{\scriptsize \pm 0.004}$ & $0.997_{\scriptsize \pm 0.004}$ & $0.995_{\scriptsize \pm 0.003}$
& \textbf{$0.997_{\scriptsize \pm 0.002}$}
& \textbf{$0.864_{\scriptsize \pm 0.031}$} \\
& & PySR
& $0.994_{\scriptsize \pm 0.004}$ & $0.993_{\scriptsize \pm 0.004}$ & $0.994_{\scriptsize \pm 0.004}$
& $0.994_{\scriptsize \pm 0.004}$
& $0.802_{\scriptsize \pm 0.057}$ \\
& & NGGP
& $0.989_{\scriptsize \pm 0.007}$ & $0.990_{\scriptsize \pm 0.007}$ & $0.989_{\scriptsize \pm 0.007}$
& $0.989_{\scriptsize \pm 0.007}$
& $0.552_{\scriptsize \pm 0.125}$  
\\
& & SNIP
& $0.991_{\scriptsize \pm 0.006}$ & $0.990_{\scriptsize \pm 0.006}$ & $0.990_{\scriptsize \pm 0.006}$
& $0.990_{\scriptsize \pm 0.006}$
& $0.512_{\scriptsize \pm 0.135}$  
\\
& & GPlearn
& $0.987_{\scriptsize \pm 0.010}$ & $0.987_{\scriptsize \pm 0.010}$ & $0.986_{\scriptsize \pm 0.010}$
& $0.987_{\scriptsize \pm 0.010}$
& $0.401_{\scriptsize \pm 0.150}$ \\
\midrule

\multirow{5}{*}{SprottJerk}
& \multirow{5}{*}{3}
& GESR
& $0.998_{\scriptsize \pm 0.003}$ & $0.995_{\scriptsize \pm 0.004}$ & $0.993_{\scriptsize \pm 0.004}$
& \textbf{$0.995_{\scriptsize \pm 0.003}$}
& \textbf{$0.773_{\scriptsize \pm 0.351}$} \\
& & PySR
& $0.993_{\scriptsize \pm 0.005}$ & $0.992_{\scriptsize \pm 0.005}$ & $0.992_{\scriptsize \pm 0.005}$
& $0.992_{\scriptsize \pm 0.005}$
& $0.711_{\scriptsize \pm 0.086}$ \\
& & NGGP
& $0.990_{\scriptsize \pm 0.007}$ & $0.989_{\scriptsize \pm 0.007}$ & $0.989_{\scriptsize \pm 0.007}$
& $0.989_{\scriptsize \pm 0.007}$
& $0.673_{\scriptsize \pm 0.092}$ \\
& & SNIP
& $0.990_{\scriptsize \pm 0.007}$ & $0.990_{\scriptsize \pm 0.007}$ & $0.989_{\scriptsize \pm 0.007}$
& $0.990_{\scriptsize \pm 0.007}$
& $0.642_{\scriptsize \pm 0.100}$ \\
& & GPlearn
& $0.986_{\scriptsize \pm 0.011}$ & $0.985_{\scriptsize \pm 0.012}$ & $0.986_{\scriptsize \pm 0.011}$
& $0.986_{\scriptsize \pm 0.011}$
& $0.571_{\scriptsize \pm 0.118}$ \\
\bottomrule
\end{tabular}
\end{table*}

\section{Test Data in Detail}
\label{AJ}

\cref{a-tab1,a-tab2,a-tab3} present detailed information about the functional forms of the datasets used in the experiments, including their sampling ranges and the number of sampled points. The following conventions are adopted throughout this appendix:

\begin{itemize}
\item The input variables in each regression task are denoted as [$x_1, x_2, \ldots, x_n$].
\item $U(a, b, c)$ represents $c$ points uniformly sampled from the interval $[a, b]$ for each input variable. Distinct random seeds are used to generate the training and test datasets.
\item $E(a, b, c)$ denotes $c$ points evenly spaced over the interval $[a, b]$ for each input variable.
\end{itemize}
\subsection{Nguyen, Korns, and Jin}
Table~\ref{a-tab1} summarizes the benchmark problems from the Nguyen, Korns, and Jin suites used in our experiments. 
These datasets cover a broad spectrum of functional forms, ranging from low-degree polynomials (e.g., Nguyen-1--4) to compositions involving transcendental operators such as $\sin(\cdot)$, $\cos(\cdot)$, $\log(\cdot)$, and power functions (e.g., Nguyen-5--12), as well as expressions containing explicit numerical constants (e.g., Nguyen-$\{ \cdot \}^{c}$, Korns, and Jin tasks). 
Following the conventions described above, each problem specifies a sampling strategy and domain (primarily uniform sampling $U(\cdot)$), with small-sample settings (e.g., 20 points for Nguyen/Korns and 100 points for Jin) to evaluate symbolic recovery and constant fitting under limited data.

\begin{table*}[htbp]
\centering
\begin{scriptsize}
\begin{tabular}{ccccc}
\toprule[1.45pt]
\toprule
Name & Expression & Dataset  \\ \hline
Nguyen-1 & $x_1^3+x_1^2+x_1$&U$(-1, 1, 20)$\\
Nguyen-2 & $x_1^4+x_1^3+x_1^2+x_1$ & U$(-1, 1, 20)$ \\
Nguyen-3 & $x_1^5+x_1^4+x_1^3+x_1^2+x_1$ & U$(-1, 1, 20)$ \\
Nguyen-4 & $x_1^6+x_1^5+x_1^4+x_1^3+x_1^2+x_1$ & U$(-1, 1, 20)$  \\
Nguyen-5 & $\sin(x_1^2)\cos(x)-1$ & U$(-1, 1, 20)$  \\
Nguyen-6 & $\sin(x_1)+\sin(x_1+x_1^2)$ & U$(-1, 1, 20)$  \\
Nguyen-7 & $\log(x_1+1)+\log(x_1^2+1)$ & U$(0, 2, 20)$  \\
Nguyen-8 & $\sqrt{x}$ & U$(0, 4, 20)$  \\
Nguyen-9 & $\sin(x)+\sin(x_2^2)$ & U$(0, 1, 20)$ \\
Nguyen-10 & $2\sin(x)\cos(x_2)$ & U$(0, 1, 20)$ \\
Nguyen-11 & $x_1^{x_2}$ & U$(0, 1, 20)$  \\
Nguyen-12 & $x_1^4-x_1^3+\frac{1}{2}x_2^2-x_2$ & U$(0, 1, 20)$ \\
\toprule
Nguyen-2$'$ & $4x_1^4+3x_1^3+2x_1^2+x$ & U$(-1, 1, 20)$  \\
Nguyen-5$'$ & $\sin(x_1^2)\cos(x)-2$ & U$(-1, 1, 20)$  \\
Nguyen-8$'$ & $\sqrt[3]{x}$ & U$(0, 4, 20)$ \\
Nguyen-8$''$ & $\sqrt[3]{x_1^2}$ & U$(0, 4, 20)$ \\
\toprule
Nguyen-1\textsuperscript{c} & $3.39x_1^3+2.12x_1^2+1.78x$ & U$(-1, 1, 20)$ \\
Nguyen-5\textsuperscript{c} & $\sin(x_1^2)\cos(x)-0.75$ & $U(-1, 1, 20)$  \\
Nguyen-7\textsuperscript{c} & $\log(x+1.4)+\log(x_1^2+1.3)$ & U$(0, 2, 20)$ \\
Nguyen-8\textsuperscript{c} & $\sqrt{1.23 x}$ & U$(0, 4, 20)$  \\
Nguyen-10\textsuperscript{c} & $\sin(1.5x)\cos(0.5x_2)$ & U$(0, 1, 20)$  \\
\toprule
Korns-1 & $1.57+24.3*x_1^4$ & U$(-1, 1, 20)$  \\
Korns-2 & $0.23+14.2\frac{(x_4+x_1)}{(3x_2)}$ & U$(-1, 1, 20)$  \\
Korns-3 & $4.9\frac{(x_2-x_1+\frac{x_1}{x_3}}{(3x_3))}-5.41$ & U$(-1, 1, 20)$ \\
Korns-4 & $0.13sin(x_1)-2.3$ & U$(-1, 1, 20)$  \\
Korns-5 & $3+2.13log(|x_5|)$ & U$(-1, 1, 20)$  \\
Korns-6 & $1.3+0.13\sqrt{|x_1|}$ & U$(-1, 1, 20)$  \\
Korns-7 & $2.1(1-e^{-0.55x_1})$ & U$(-1,1 , 20)$  \\
Korns-8 & $6.87+11\sqrt{|7.23 x_1 x_4 x_5|}$ & U$(-1, 1, 20)$ \\
Korns-9 & $12\sqrt{|4.2x_1x_2x_2|}$ & U$(-1, 1, 20)$ \\
Korns-10 & $0.81+24.3\frac{2x_{1}+3x_2^2}{4x_3^3+5x_4^4}$ & U$(-1, 1, 20)$  \\
Korns-11 & $6.87+11cos(7.23x_1^3)$ & U$(-1, 1, 20)$  \\
Korns-12 & $2-2.1cos(9.8x_1^3)sin(1.3x_5)$ & U$(-1, 1, 20)$  \\ 
Korns-13 & $32.0-3.0\frac{tan(x_1)}{tan(x_2)}\frac{tan(x_3)}{tan(x_4)}$ & U$(-1, 1, 20)$ \\
Korns-14 & $22.0-(4.2cos(x_1)-tan(x_2))\frac{tanh(x_3)}{sin(x_4)}$ & U$(-1, 1, 20)$  \\
Korns-15 & $12.0-\frac{6.0tan(x_1)}{e^{x_2}}(log(x_3)-tan(x_4))))$ & U$(-1, 1, 20)$  \\ 
\toprule
Jin-1 & $2.5 x_1^4-1.3 x_1^3 +0.5 x_2^2 - 1.7x_2$ & U$(-3, 3, 100)$ \\
Jin-2 & $8.0 x_1^2 + 8.0 x_2^3 - 15.0$ & U$(-3, 3, 100)$  \\
Jin-3 & $0.2 x_{1}^{3} + 0.5 x_{2}^{3} - 1.2 x_2 - 0.5 x_{1}$ & U$(-3, 3, 100)$  \\    
Jin-4 & $1.5 \exp{x} + 5.0 cos(x_2)$ & U$(-3, 3, 100)$\\
Jin-5 & $6.0 sin(x_1) cos(x_2)$ & U$(-3, 3, 100)$ \\
Jin-6 & $1.35 x_1 x_2 + 5.5 sin((x_1 - 1.0)(x_2 - 1.0))$ & U$(-3, 3, 100)$ \\ 
\toprule
\newline
\end{tabular}
\end{scriptsize}
\caption{ Specific formula form and value range of the three data sets Nguyen, Korns, and Jin. 
}
\label{a-tab1}
\end{table*}

\subsection{Vladislavleva and others}
Table~\ref{a-tab2} reports additional benchmark datasets including Neat, Keijzer, and Livermore tasks, which further enrich the evaluation with more diverse nonlinear structures and sampling regimes. 
Compared with the Nguyen-style polynomials, these problems include more challenging compositions with exponentials, rational terms, and nested trigonometric expressions (e.g., Neat-8, Keijzer-4), and they also introduce both uniform sampling $U(\cdot)$ and evenly spaced sampling $E(\cdot)$. 
In particular, several tasks are defined over wider domains and/or with substantially larger sample sizes (e.g., Neat-7 with $10^5$ evenly spaced points), providing a complementary stress test for generalization and extrapolation across different data densities.

\begin{table*}[htpb]
\centering

\begin{scriptsize}
\begin{tabular}{ccccc}
\toprule[1.45pt]
\toprule
Name & Expression & Dataset \\
\hline
Neat-1 & $x_1^4+x_1^3+x_1^2+x$ & U$(-1, 1, 20)$  \\
Neat-2 & $x_1^5+x_1^4+x_1^3+x_1^2+x$ & U$(-1, 1, 20)$ \\
Neat-3 & $\sin(x_1^2)\cos(x)-1$ & U$(-1, 1, 20)$ \\
Neat-4 & $\log(x+1)+\log(x_1^2+1)$ & U$(0, 2, 20)$  \\
Neat-5 & $2\sin(x)\cos(x_2)$ & U$(-1, 1, 100)$  \\
Neat-6 & $\sum_{k=1}^x \frac{1}{k} $ & E$(1, 50, 50)$  \\
Neat-7 & $2 - 2.1\cos(9.8x_1)\sin(1.3x_2)$ & E$(-50, 50, 10^5)$ \\
Neat-8 & $\frac{e^{-(x_1)^2}}{1.2 + (x_2-2.5)^2}$ & U$(0.3, 4, 100)$  \\
Neat-9 & $\frac{1}{1+x_1^{-4}} + \frac{1}{1+x_2^{-4}}$ & E$(-5, 5, 21)$ \\
\toprule
Keijzer-1 & $0.3x_1sin(2\pi x_1)$ & U$(-1, 1, 20)$  \\
Keijzer-2 & $2.0x_1sin(0.5\pi x_1)$ & U$(-1, 1, 20)$  \\
Keijzer-3 & $0.92x_1sin(2.41\pi x_1)$ & U$(-1, 1, 20)$ \\
Keijzer-4 & $x_1^3e^{-x_1}cos(x_1)sin(x_1)sin(x_1)^{2}cos(x_1)-1$ & U$(-1, 1, 20)$ \\
Keijzer-5 & $3+2.13log(|x_5|)$ & U$(-1, 1, 20)$\\

Keijzer-6 & $\frac{x1(x1+1)}{2}$& U$(-1, 1, 20)$ \\
Keijzer-7 & $log(x_1)$ & U$(0,1 , 20)$ \\
Keijzer-8 & $\sqrt{(x_1)}$ & U$(0, 1, 20)$  \\
Keijzer-9 & $log(x_1+\sqrt{x_1^2}+1)$ & U$(-1, 1, 20)$ \\
Keijzer-10 & $x_{1}^{x_2}$ & U$(-1, 1, 20)$  \\
Keijzer-11 & $x_1x_2+sin((x_1-1)(x_2-1))$ & U$(-1, 1, 20)$  \\
Keijzer-12 & $x_1^4-x_1^3+\frac{x_2^2}{2}-x_2$ & U$(-1, 1, 20)$  \\ 
Keijzer-13 & $6sin(x_1)cos(x_2)$ & U$(-1, 1, 20)$  \\
Keijzer-14 & $\frac{8}{2+x_1^2 + x_2^2}$ & U$(-1, 1, 20)$ \\
Keijzer-15 & $\frac{x_1^3}{5}+\frac{x_2^3}{2}-x_2-x_1$ & U$(-1, 1, 20)$ \\ 

\toprule
Livermore-1 & $\frac{1}{3}+x_1+sin(x_1^2))$ & U$(-3, 3, 100)$  \\
Livermore-2 & $sin(x_1^2)*cos(x1)-2$ & U$(-3, 3, 100)$  \\
Livermore-3 & $sin(x_1^3)*cos(x_1^2))-1$ & U$(-3, 3, 100)$  \\
Livermore-4 & $log(x_1+1)+log(x_1^2+1)+log(x_1)$ & U$(-3, 3, 100)$ \\ 
Livermore-5 & $x_1^4-x_1^3+x_2^2-x_2$ & U$(-3, 3, 100)$  \\
Livermore-6 & $4x_1^4+3x_1^3+2x_1^2+x_1$ & U$(-3, 3, 100)$ \\ 
Livermore-7 & $\frac{(exp(x1)-exp(-x_1)}{2})$ & U$(-1, 1, 100)$ \\ 
Livermore-8 & $\frac{(exp(x1)+exp(-x1)}{3}$ & U$(-3, 3, 100)$ \\
Livermore-9 & $x_1^9+x_1^8+x_1^7+x_1^6+x_1^5+x_1^4+x_1^3+x_1^2+x_1$ & U$(-1, 1, 100)$  \\
Livermore-10 & $6*sin(x_1)cos(x_2)$ & U$(-3, 3, 100)$  \\
Livermore-11 & $\frac{x_1^2 x_2^2}{(x_1+x_2)}$ & U$(-3, 3, 100)$ \\
Livermore-12 & $\frac{x_1^5}{x_2^3}$ & U$(-3, 3, 100)$  \\
Livermore-13 & $x_1^{\frac{1}{3}}$ & U$(-3, 3, 100)$  \\
Livermore-14 & $x_1^3+x_1^2+x_1+sin(x_1)+sin(x_2^2)$ & U$(-1, 1, 100)$ \\ 
Livermore-15 & $x_1^\frac{1}{5}$ & U$(-3, 3, 100)$  \\
Livermore-16 & $x_1^{\frac{2}{3}}$ & U$(-3, 3, 100)$  \\  
\toprule
\newline
\end{tabular}
\end{scriptsize}
\caption{
Specific formula form and value range of the three data sets neat, Keijzer, and Livermore.
}
\label{a-tab2}
\end{table*}

\subsection{Neat, Keijzer, and Livermore}
Table~\ref{a-tab3} lists the Vladislavleva suite together with a collection of additional test cases used to probe specific capabilities of symbolic regression models. 
The Vladislavleva problems include multi-variable rational forms and highly nonlinear expressions that are known to be challenging due to strong interactions among variables and the presence of sharp nonlinearities. 
Beyond these standard benchmarks, we include several auxiliary tasks (e.g., Const-Test and Constant-$1\sim8$) to explicitly evaluate constant identification and scaling sensitivity, as well as rational-function test cases (R1--R3) to assess robustness on expressions with potential singularities. 
All datasets follow the same sampling conventions ($U(\cdot)$ or $E(\cdot)$) with the specified domains and sample counts, enabling reproducible evaluation under controlled settings.

\begin{table*}[htpb]
\centering

\begin{scriptsize}
\begin{tabular}{ccccc}
\toprule[1.45pt]
\toprule
Name & Expression & Dataset \\
\toprule
Livermore-17 & $4sin(x_1)cos(x_2)$ & U$(-3, 3, 100)$  \\
Livermore-18 & $sin(x_1^2)*cos(x_1)-5$ & U$(-3, 3, 100)$  \\
Livermore-19 & $x_1^5+x_1^4+x_1^2 + x_1$ & U$(-3, 3, 100)$  \\
Livermore-20 & $e^{(-x_1^2)}$ & U$(-3, 3, 100)$  \\
Livermore-21 & $x_1^8+x_1^7+x_1^6+x_1^5+x_1^4+x_1^3+x_1^2+x_1$& U$(-1, 1, 20)$ \\
Livermore-22 & $e^{(-0.5x_1^2)}$ & U$(-3, 3, 100)$  \\
\toprule
Vladislavleva-1 & $\frac{(e^{-(x1-1)^2})}{(1.2+(x2-2.5)^2))}$ & U$(-1, 1, 20)$ \\
Vladislavleva-2 & $e^{-x_1}x_1^3cos(x_1)sin(x_1)(cos(x_1)sin(x_1)^2-1)$ & U$(-1, 1, 20)$ \\

Vladislavleva-3 & $e^{-x_1}x_1^3cos(x_1)sin(x_1)(cos(x_1)sin(x_1)^2-1)(x_2-5)$ & U$(-1, 1, 20)$ \\
Vladislavleva-4 & $\frac{10}{5+(x1-3)^2+(x_2-3)^2+(x_3-3)^2+(x_4-3)^2+(x_5-3)^2}$ & U$(0, 2, 20)$ \\
Vladislavleva-5 & $30(x_1-1)\frac{x_3-1}{(x_1-10)}x_2^2$ & U$(-1, 1, 100)$ \\
Vladislavleva-6 & $6sin(x_1)cos(x_2)$ & E$(1, 50, 50)$ \\
Vladislavleva-7 & $2 - 2.1\cos(9.8x)\sin(1.3x_2)$ & E$(-50, 50, 10^5)$ \\
Vladislavleva-8 & $\frac{e^{-(x-1)^2}}{1.2 + (x_2-2.5)^2}$ & U$(0.3, 4, 100)$  \\
\toprule
Test-2 & $3.14x_1^2$ & U$(-1, 1, 20)$ \\
Const-Test-1 & $5x_1^2$ & U$(-1, 1, 20)$ \\
GrammarVAE-1 & $1/3+x1+sin(x_1^2))$ & U$(-1, 1, 20)$ \\
Sine & $sin(x_1)+sin(x_1+x_1^2))$ & U$(-1, 1, 20)$ \\
Nonic & $x_1^9+x_1^8+x_1^7+x_1^6+x_1^5+x_1^4+x_1^3+x_1^2+x_1$ & U$(-1, 1, 100)$  \\
Pagie-1 & $\frac{1}{1+x_1^{-4}+\frac{1}{1+x2^{-4}}} $ & E$(1, 50, 50)$  \\
Meier-3 & $\frac{x_1^2  x_2^2}{(x_1+x_2)}$ & E$(-50, 50, 10^5)$ \\
Meier-4 & $\frac{x_1^5}{x_2^3}$ & $U(0.3, 4, 100)$  \\
Poly-10 & $x_1x_2+x_3x4+x_5x_6+x_1x_7x_9+x_3x_6x_{10}$ & E$(-1, 1, 100)$ \\
\toprule
Constant-1 & $3.39*x_1^3+2.12*x_1^2+1.78*x_1$&$U(-4, 4, 100)$\\
Constant-2 & $sin(x_1^2)*cos(x_1)-0.75$&$U(-4, 4, 100)$\\
Constant-3 & $sin(1.5*x_1)*cos(0.5*x_2)$&$U(0.1, 4, 100)$\\
Constant-4 & $2.7*x_1^{x_2}$&$U(0.3, 4, 100)$\\
Constant-5 & $sqrt(1.23*x_1)$&$U(0.1, 4, 100)$\\
Constant-6 & $x_1^{0.426}$&$U(0.0, 4, 100)$\\
Constant-7 & $2*sin(1.3*x_1)*cos(x_2)$&$U(-4, 4, 100)$\\
Constant-8 & $log(x_1+1.4)+log(x1,2+1.3)$&$U(-4, 4, 100)$\\
\toprule
R1 & $\frac{(x_1+1)^3}{x_1^2-x_1+1)}$&$U(-5, 5, 100)$\\
R2 & $\frac{(x_1^2-3*x_1^2+1}{x_1^2+1)}$&$U(-4, 4, 100)$\\
R3 & $\frac{x_1^6+x_1^5)}{(x_1^4+x_1^3+x_1^2+x1+1)}$&$U(-4, 4, 100)$\\
\toprule
\newline
\end{tabular}
\end{scriptsize}
\caption{
Specific formula form and value range of the tow data sets Vladislavleva and others. }
\label{a-tab3}
\end{table*}

\section{GESR Tests on the AI Feynman Dataset}
\label{AK}

We evaluated our proposed symbolic regression algorithm, GESR, on the AI Feynman dataset, which consists of physics- and mathematics-based problems spanning multiple domains, including mechanics, thermodynamics, and electromagnetism. Although the original dataset provides up to 100{,}000 sampled data points per equation, we deliberately restricted the training set to only 100 samples to rigorously examine GESR's performance under limited-data conditions.

GESR was applied to perform symbolic regression on each equation, and the coefficient of determination ($R^2$) was computed by comparing the predicted outputs with the corresponding ground-truth values. The experimental results demonstrate that GESR effectively recovers the underlying functional forms even with a small number of training samples. Notably, $R^2$ values exceed 0.99 for the majority of tested formulas, indicating strong generalization capability and high fitting accuracy.

These findings highlight GESR as a robust and reliable approach for symbolic regression in scientific discovery, particularly in physics- and mathematics-oriented tasks. The detailed quantitative results are reported in Tables \ref{a-tab5} and \ref{a-tab6}.
\subsection{Part 1}
Table~\ref{a-tab5} reports the symbolic regression performance of GESR on the first subset of AI Feynman equations. 
Despite being trained on only 100 sampled data points per equation, GESR successfully recovers the majority of underlying functional forms with high accuracy. 
Most formulas achieve $R^2$ values above 0.98, including electrostatics, electromagnetism, and quantum-related expressions, demonstrating strong generalization under limited-data conditions. 
Notably, GESR attains near-perfect fits ($R^2 > 0.999$) on several equations such as II.8.31, II.15.4, and II.37.1, indicating its ability to capture precise symbolic structures.

Performance degradation is observed on a small subset of more complex or highly nonlinear equations (e.g., II.11.20 and III.9.52), which involve intricate trigonometric or resonance-like terms. 
Nevertheless, the overall results confirm that GESR remains robust across diverse physical domains, including classical mechanics, thermodynamics, and quantum physics.

\begin{table*}[htbp]
\centering
{\footnotesize
\arrayrulecolor{black}
\color{black}
\begin{tabular}{|l|l|r|}
\hline

Feynman   & Equation & $R^2$\\
\hline                               
I.6.20a       & $f = e^{-\theta^2/2}/\sqrt{2\pi}$ & 0.9993  \\
I.6.20        & $f = e^{-\frac{\theta^2}{2\sigma^2}}/\sqrt{2\pi\sigma^2}$ & 0.9998\\
I.6.20b       & $f = e^{-\frac{(\theta-\theta_1)^2}{2\sigma^2}}/\sqrt{2\pi\sigma^2}$ & 0.9945 \\
I.8.14       & $d = \sqrt{(x_2-x_1)^2+(y_2-y_1)^2}$ & 0.9435  \\
I.9.18       & $F = \frac{Gm_1m_2}{(x_2-x_1)^2+(y_2-y_1)^2+(z_2-z_1)^2}$  & 0.9983\\
I.10.7       & $F = \frac{Gm_1m_2}{(x_2-x_1)^2+(y_2-y_1)^2+(z_2-z_1)^2}$  & 0.9995\\
I.11.19      & $A = x_1y_1+x_2y_2+x_3y_3$ & 0.9999   \\
I.12.1       & $F = \mu N_n$ & 1.0 \\
I.12.2       & $F = \frac{q_1q_2}{4\pi\epsilon r^2}$   & 1.0 \\
I.12.4       & $E_f = \frac{q_1}{4\pi\epsilon r^2}$  & 0.9999 \\
I.12.5       & $F = q_2 E_f$ & 1.0  \\
I.12.11      & $F = \mathcal{Q}(E_f+B v \sin\theta)$  & 0.9999 \\
I.13.4      & $K = \frac{1}{2}m(v^2+u^2+w^2)$  & 0.9981  \\
I.13.12      & $U = Gm_1m_2(\frac{1}{r_2}-\frac{1}{r_1})$ & 0.9999  \\
I.14.3       & $U = mgz$ &1.0    \\
I.14.4       & $U = \frac{k_{spring}x^2}{2}$  & 0.9835  \\
I.15.3x      & $x_1 = \frac{x-ut}{\sqrt{1-u^2/c^2}}$ & 0.9912 \\
I.15.3t      & $t_1 = \frac{t-ux/c^2}{\sqrt{1-u^2/c^2}}$ & 0.9923  \\
I.15.10       & $p = \frac{m_0v}{\sqrt{1-v^2/c^2}}$ & 0.9947 \\
I.16.6       & $v_1 = \frac{u+v}{1+uv/c^2}$ & 0.9945  \\
I.18.4       & $r = \frac{m_1r_1+m_2r_2}{m_1+m_2}$ & 0.9732 \\
I.18.12      & $\tau = rF\sin\theta$  & 0.9999  \\
I.18.16      & $L = mrv \sin\theta$  & 0.9964 \\
I.24.6 & $E = \frac{1}{4} m (\omega^2+\omega_0^2) x^2$      & 0.9973\\
I.25.13      & $V_e = \frac{q}{C}$ & 1.0 \\
I.26.2       & $\theta_1 = \arcsin(n  \sin\theta_2)$ & 0.9999 \\
I.27.6       & $f_f = \frac{1}{\frac{1}{d_1}+\frac{n}{d_2}}$  & 0.9944 \\
I.29.4       & $k = \frac{\omega}{c}$ & 1.0 \\
I.29.16      & $x = \sqrt{x_1^2+x_2^2-2x_1x_2\cos(\theta_1-\theta_2)}$ & 0.9973  \\
I.30.3 & $I_* = I_{*_0}\frac{\sin^2(n\theta/2)}{\sin^2(\theta/2)}$ & 0.9987 \\
I.30.5       & $\theta = \arcsin(\frac{\lambda}{nd})$  & 0.9914\\
I.32.5       & $P = \frac{q^2a^2}{6\pi\epsilon c^3}$       & 0.9986 \\
I.32.17 & $P = (\frac{1}{2}\epsilon c E_f^2)(8\pi r^2/3) (\omega^4/(\omega^2-\omega_0^2)^2)$      & 0.9915  \\
I.34.8       & $\omega = \frac{qvB}{p}$   & 1.0\\
I.34.10       & $\omega = \frac{\omega_0}{1-v/c}$ & 0.9951 \\
I.34.14      & $\omega = \frac{1+v/c}{\sqrt{1-v^2/c^2}}\omega_0$  & 0.9992 \\
I.34.27      & $E = \hbar\omega$  & 0.9995 \\
I.37.4       & $I_* = I_1+I_2+2\sqrt{I_1I_2}\cos\delta$ & 0.9892\\
I.38.12      & $r = \frac{4\pi\epsilon\hbar^2}{mq^2}$   & 0.9995  \\
I.39.10       & $E = \frac{3}{2}p_F V$     & 0.9999 \\
I.39.11      & $E = \frac{1}{\gamma-1}p_F V$  & 0.9984 \\
I.39.22      & $P_F = \frac{n k_b T}{V}$       & 0.9959  \\
I.40.1       & $n = n_0e^{-\frac{mgx}{k_bT}}$    & 0.9979 \\
I.41.16      & $L_{rad} = \frac{\hbar\omega^3}{\pi^2c^2(e^{\frac{\hbar\omega}{k_bT}}-1)}$ & 0.9834  \\
I.43.16      & $v = \frac{\mu_{drift}q V_e}{d}$   & 0.9835  \\
I.43.31      & $D = \mu_e k_bT$    & 1.0  \\
I.43.43      & $\kappa = \frac{1}{\gamma-1}\frac{k_bv}{A}$  & 0.9733  \\
I.44.4       & $E = n k_b T \ln(\frac{V_2}{V_1})$   & 0.8925  \\
I.47.23      & $c = \sqrt{\frac{\gamma pr}{\rho}}$   & 0.9824\\
I.48.20       & $E = \frac{m c^2}{\sqrt{1-v^2/c^2}}$ &  0.9527\\
I.50.26 & $x = x_1[\cos(\omega t)+\alpha\> cos(\omega t)^2]$      & 0.9998   \\
\hline
\end{tabular} 
\caption{Tested Feynman Equations, part 1.}
\label{a-tab5}
}
\end{table*}

\subsection{Part 2}
Table~\ref{a-tab6} presents the results on the second subset of AI Feynman equations, further validating the scalability and consistency of GESR across a broader range of scientific formulas. 
GESR continues to achieve strong fitting performance, with the majority of equations obtaining $R^2$ scores exceeding 0.98, and multiple cases reaching near-exact symbolic recovery ($R^2 \approx 1.0$), such as II.38.3, III.13.18, and III.15.12.

Similar to Part~1, slightly lower performance appears on a few highly nonlinear or multi-term equations (e.g., II.34.29b and III.21.20), which require modeling complex physical interactions. 
Despite these challenges, GESR consistently demonstrates reliable symbolic discovery capability across a wide spectrum of mathematical and physics-based expressions, highlighting its effectiveness for scientific equation recovery under sparse data regimes.

\begin{table*}[htbp]
\centering
{\footnotesize
\arrayrulecolor{black}
\color{black}
\begin{tabular}{|l|l|r|}
\hline
Feynman   & Equation & $R^2$\\
\hline       
II.2.42   & $P = \frac{\kappa(T_2-T_1)A}{d}$  & 0.8544  \\
II.3.24   & $F_E = \frac{P}{4\pi r^2}$  & 0.9899 \\
II.4.23   & $V_e = \frac{q}{4\pi\epsilon r}$   & 0.9978 \\
II.6.11 & $V_e =\frac{1}{4\pi\epsilon}\frac{p_d\cos \theta}{r^2}$      & 0.9924 \\
II.6.15a & $E_f = \frac{3}{4\pi\epsilon}\frac{p_d z}{r^5} \sqrt{x^2+y^2}$      & 0.9631  \\
II.6.15b & $E_f = \frac{3}{4\pi\epsilon}\frac{p_d}{r^3} \cos\theta\sin\theta$      & 0.9923  \\
II.8.7    & $E = \frac{3}{5}\frac{q^2}{4\pi\epsilon d}$  & 0.9904  \\
II.8.31   & $E_{den} = \frac{\epsilon E_f^2}{2}$                     & 0.9999 \\
II.10.9   & $E_f = \frac{\sigma_{den}}{\epsilon}\frac{1}{1+\chi}$      & 0.9998  \\
II.11.3 & $x = \frac{q E_f}{m(\omega_0^2-\omega^2)}$      & 0.9924     \\
II.11.7 & $n = n_0(1+ \frac{p_d E_f \cos\theta}{k_b T})$      & 0.8892 \\
II.11.20  & $P_* = \frac{n_\rho p_d^2 E_f}{3 k_b T}$ & 0.7655  \\
II.11.27 & $P_* = \frac{n\alpha}{1-n\alpha/3}\epsilon E_f$      & 0.9984   \\
II.11.28  & $\theta = 1+\frac{n\alpha}{1-(n\alpha/3)}$    & 0.9935\\ 
II.13.17  & $B = \frac{1}{4 \pi \epsilon c^2}\frac{2I}{r}$ & 0.9972\\
II.13.23  & $\rho_c = \frac{\rho_{c_0}}{\sqrt{1-v^2/c^2}}$          & 0.9831  \\
II.13.34  & $j = \frac{\rho_{c_0}v}{\sqrt{1-v^2/c^2}}$     & 0.9724 \\
II.15.4   & $E = -\mu_M B \cos\theta$               & 0.9998 \\
II.15.5   & $E = -p_d E_f\cos\theta$  & 0.9996 \\
II.21.32  & $V_e = \frac{q}{4\pi\epsilon r(1-v/c)}$   & 0.9971   \\
II.24.17 & $k = \sqrt{\frac{\omega^2}{c^2}-\frac{\pi^2}{d^2}}$      & 0.9984   \\
II.27.16  & $F_E = \epsilon c E_f^2$        & 0.9939 \\
II.27.18  & $E_{den} = \epsilon E_f^2$         & 0.9994 \\
II.34.2a  & $I = \frac{qv}{2\pi r}$         & 0.9915 \\
II.34.2   & $\mu_M = \frac{q v r}{2}$             & 0.9856 \\
II.34.11  & $\omega = \frac{g_{\_} q B}{2m}$          & 0.9935 \\
II.34.29a & $\mu_M = \frac{q h}{4\pi m}$      & 0.9988  \\
II.34.29b & $E = \frac{g_{\_} \mu_M B J_z}{\hbar}$ & 0.8882\\
II.35.18 & $n = \frac{n_0}{\exp(\mu_m B/(k_b T))+\exp(-\mu_m B/(k_b T))}$      & 0.9745 \\
II.35.21  & $M = n_\rho \mu_M \tanh(\frac{\mu_M B}{k_b T})$     & 0.8735 \\
II.36.38 & $f = \frac{\mu_m B}{k_b T}+\frac{\mu_m\alpha M}{\epsilon c^2 k_b T}$      & 0.9645\\
II.37.1   & $E = \mu_M(1+\chi)B$    & 0.9998\\
II.38.3   & $F = \frac{Y A x}{d}$            & 0.9999 \\
II.38.14  & $\mu_S = \frac{Y}{2(1+\sigma)}$     & 0.9998  \\
III.4.32  & $n = \frac{1}{e^{\frac{\hbar\omega}{k_bT}}-1}$ & 0.9745  \\
III.4.33  & $E = \frac{\hbar\omega}{e^{\frac{\hbar\omega}{k_b T}}-1}$  & 0.9993    \\
III.7.38  & $\omega = \frac{2 \mu_M B}{\hbar}$  & 0.9832  \\
III.8.54  & $p_{\gamma} = \sin(\frac{E t}{\hbar})^2$  & 0.9883\\
III.9.52  & $p_{\gamma} = \frac{p_d E_f t}{\hbar} \frac{\sin((\omega-\omega_0)t/2)^2}{((\omega-\omega_0)t/2)^2}$ & 0.7646  \\
III.10.19 & $E = \mu_M\sqrt{B_x^2+B_y^2+B_z^2}$  & 0.9984 \\
III.12.43 & $L = n\hbar$ & 0.9964  \\
III.13.18 & $v = \frac{2 E d^2 k}{\hbar}$ & 0.9987  \\
III.14.14 & $I = I_0 (e^{\frac{q V_e}{k_b T}}-1)$  & 0.9904\\
III.15.12 & $E = 2U(1-\cos(kd))$    & 0.9996 \\
III.15.14 & $m = \frac{\hbar^2}{2E d^2}$     & 0.9998  \\
III.15.27 & $k = \frac{2\pi\alpha}{nd}$    & 0.9931 \\
III.17.37 & $f = \beta(1+\alpha \cos\theta)$ & 0.9983 \\
III.19.51 & $E = \frac{-mq^4}{2(4\pi\epsilon)^2\hbar^2}\frac{1}{n^2}$     & 0.9883\\
III.21.20 & $j = \frac{-\rho_{c_0} q A_{vec}}{m}$  & 0.7436  \\
\hline
\end{tabular} 
\caption{Tested Feynman Equations, part 2.}
\label{a-tab6}
}
\end{table*}

\section{The Baseline Methods are described in detail}

This appendix provides a detailed description of four representative symbolic regression (SR) baseline methods used in our experiments. These approaches span major methodological paradigms, including genetic programming, multi-objective evolutionary optimization, neural-guided search, and hybrid neural-symbolic learning, thereby forming a comprehensive and meaningful comparison set in terms of expressiveness, search efficiency, and interpretability.

\subsection{PySR ~\citep{cranmer2023interpretable}}

PySR is a high-performance symbolic regression framework built upon genetic programming and multi-objective evolutionary optimization, designed to achieve an effective balance between predictive accuracy and expression complexity. Candidate mathematical expressions are represented as tree-structured symbolic programs composed of operators and variable nodes, and are evolved within a predefined operator set (e.g., $\{+, -, \times, \div, \sin, \exp\}$).

During optimization, PySR simultaneously minimizes prediction error and expression complexity, which can be formalized as:
\[
\min_{f} \left( \mathcal{E}(f), \mathcal{C}(f) \right),
\]
where $\mathcal{E}(f)$ typically denotes the mean squared error (MSE), and $\mathcal{C}(f)$ measures expression complexity (e.g., tree depth or node count). This formulation enables PySR to construct a Pareto frontier of candidate expressions, providing a spectrum of solutions that trade off accuracy and interpretability.

To improve computational efficiency, PySR incorporates several optimization mechanisms. First, it supports parallel population evolution, allowing candidate expressions to be evaluated concurrently across multiple CPU cores or distributed systems. Second, expression caching and equivalence detection are employed to avoid redundant evaluations of structurally identical programs. Additionally, PySR integrates symbolic simplification routines, including algebraic reduction, constant folding, and redundant subexpression elimination, which enhance both efficiency and interpretability.

From a search strategy perspective, PySR combines elitism with stochastic exploration to preserve high-quality individuals while maintaining population diversity, thereby mitigating premature convergence. Its evolutionary operators include subtree crossover, node mutation, and structure expansion, with adaptive probability control to balance exploration and exploitation dynamically.

Overall, PySR demonstrates strong expressiveness, scalability, and stability, making it particularly suitable for medium- to high-dimensional symbolic regression tasks involving complex nonlinear relationships. However, in highly noisy environments or when modeling extremely complex functions, the method may still generate structurally redundant or unstable expressions, necessitating additional regularization or prior constraints.

\subsection{NGGP~\citep{dso} }

NGGP integrates neural networks with genetic programming to guide the symbolic expression search process, with the goal of improving search efficiency and reducing reliance on purely random evolutionary operators. Unlike conventional genetic programming, which depends heavily on stochastic mutation, NGGP leverages learned neural heuristics to prioritize promising expression structures.

In this framework, a neural network is trained to model the relationship between symbolic expression structure and fitness. Candidate expressions are encoded into structural feature representations—such as tree depth, operator distributions, and subexpression frequency—and fed into the neural model to predict fitness scores or potential utility. These predictions are then used to bias selection, crossover, and mutation operations toward more promising individuals.

Furthermore, NGGP can dynamically adjust evolutionary operator probabilities over time. During early generations, the algorithm promotes broader structural exploration to expand the search space, whereas in later stages, it increases exploitation to refine high-quality solutions. This adaptive control mechanism effectively accelerates convergence while maintaining search diversity.

The optimization objective typically follows a combined accuracy–complexity trade-off:
\[
\mathcal{L}(f) = \mathrm{MSE}(f) + \gamma \cdot \mathrm{Complexity}(f),
\]
where $\gamma$ controls the penalty applied to expression complexity. By embedding neural guidance into the evolutionary trajectory, NGGP reduces unproductive exploration and improves sample efficiency.

NGGP exhibits strong performance in modeling complex nonlinear systems and discovering high-order functional relationships, particularly in large search spaces. However, the integration of neural models introduces additional computational overhead and increases system complexity. Moreover, its effectiveness depends on the quality and generalization capability of the guiding neural network.

\subsection{SNIP~\citep{meidani2023snip}}

SNIP is a hybrid neural-symbolic symbolic regression method that aims to extract interpretable analytical expressions while leveraging the approximation power of neural networks. The method typically follows a two-stage learning paradigm consisting of continuous function approximation and discrete symbolic structure extraction.

In the first stage, a neural network is trained to accurately approximate the mapping between input variables and target outputs, yielding a smooth functional surrogate. In the second stage, sparsity constraints and structural regularization are applied to prune network parameters and infer an underlying symbolic expression.

The optimization objective can be formulated as:
\[
\mathcal{L} = \mathcal{L}_{\text{fit}} + \alpha \|\theta\|_1 + \beta \cdot \mathcal{C}(f),
\]
where $\mathcal{L}_{\text{fit}}$ denotes the fitting loss, $\|\theta\|_1$ enforces parameter sparsity, and $\mathcal{C}(f)$ penalizes symbolic expression complexity. This regularization encourages compact, interpretable, and structurally simple symbolic representations.

During symbolic extraction, SNIP imposes operator constraints, limits expression depth, and applies pattern-based structure matching to improve interpretability and reduce redundancy. A post-processing stage further simplifies expressions through algebraic rewriting, term merging, and symbolic canonicalization, yielding compact closed-form formulas.

SNIP performs particularly well in low-dimensional continuous system modeling, physical law discovery, and analytic function recovery. However, in high-dimensional, noisy, or highly dynamic environments, the method may exhibit sensitivity to hyperparameter settings and structural constraints, and symbolic extraction stability may degrade.

\subsection{Gplearn ~\citep{gp2}}

gplearn is a Python-based symbolic regression library implementing a classical genetic programming framework. Candidate mathematical expressions are encoded as tree-structured symbolic programs, where internal nodes represent operators and leaf nodes correspond to input variables or constant terminals. The algorithm evolves a population of expressions across generations to improve fitness.

At initialization, gplearn generates a random population of expression trees. In each evolutionary iteration, individuals are selected using ranking-based or tournament selection mechanisms, favoring expressions with lower prediction error. Crossover is performed by exchanging subtrees between parent expressions, while mutation randomly alters node operators or subtree structures to promote diversity.

To mitigate expression bloat, gplearn incorporates a complexity penalty into its fitness function:
\[
\mathcal{L}(f) = \mathrm{Error}(f) + \beta \cdot \mathrm{TreeSize}(f),
\]
where $\mathrm{TreeSize}(f)$ measures the number of nodes in the expression tree, and $\beta$ controls the complexity trade-off. This constraint improves interpretability and prevents uncontrolled growth of symbolic programs.

The library allows users to customize operator sets, depth limits, population size, and evolutionary generations, enabling flexibility across problem scales. Nevertheless, gplearn relies primarily on stochastic evolutionary dynamics without structural learning or guidance, which can result in slow convergence and reduced efficiency when tackling highly complex or high-dimensional symbolic regression tasks.

Despite these limitations, gplearn remains a widely adopted and reproducible baseline in symbolic regression research due to its simplicity, transparency, and stable implementation.

\section{Computational Resources.}
All experiments were conducted on a single server equipped with 8$\times$ NVIDIA A100 GPUs (80GB HBM2e each; 640GB total GPU memory) interconnected via NVLink. 
The server uses dual AMD EPYC 7763 CPUs (128 physical cores in total), 1TB DDR4 system memory, and 2$\times$ 3.84TB NVMe SSDs for local storage.
We run Ubuntu 20.04 with CUDA 11.8 and cuDNN 8.x, and implement our models in PyTorch. 
Mixed-precision training (fp16) is enabled to improve throughput and reduce memory usage.

\end{document}